\documentclass{article}

\usepackage[nonatbib, preprint]{neurips_2026}
\usepackage[numbers]{natbib}

\usepackage[utf8]{inputenc} %
\usepackage[T1]{fontenc}    %
\usepackage{hyperref}       %
\usepackage{url}            %
\usepackage{booktabs}       %
\usepackage{amsfonts}       %
\usepackage{nicefrac}       %
\usepackage{microtype}      %
\usepackage[svgnames]{xcolor} %
\usepackage{wrapfig}
\usepackage{algorithmic}
\usepackage{amsmath,amssymb,amsthm,mathtools,bm}
\usepackage{titletoc}
\usepackage[most]{tcolorbox}
\usepackage{bbm}

\usepackage{subcaption}
\usepackage{wrapfig}
\usepackage{amsmath}
\usepackage{enumitem}

\usepackage[most]{tcolorbox}
\usepackage{fontawesome5}

\newtcolorbox{badbox}{
    enhanced,
    colback=red!5,
    colframe=red!60!black,
    arc=8pt,
    boxrule=0.6pt,
    left=6pt, right=6pt, top=6pt, bottom=6pt,
    drop shadow={black!30!white},
}

\newtcolorbox{goodbox}{
    enhanced,
    colback=green!5,
    colframe=green!50!black,
    arc=8pt,
    boxrule=0.6pt,
    left=6pt, right=6pt, top=6pt, bottom=6pt,
    drop shadow={black!30!white},
}

\usepackage{algorithm}

\newcommand{\MethodS}{ASTEROID}
\newcommand{\MethodH}{\textcolor{teal}{\textbf{A}}symmetric \textcolor{teal}{\textbf{S}}tudent \textcolor{teal}{\textbf{Te}}acher \textcolor{teal}{\textbf{R}}oll-\textcolor{teal}{\textbf{O}}ut \textcolor{teal}{\textbf{I}}n-context \textcolor{teal}{\textbf{D}}istillation}

\newtheorem{theorem}{Theorem}
\newtheorem{lemma}{Lemma}

\newtheorem{corollary}{Corollary}

\newtheorem{assumption}{Assumption}
\newtheorem{remark}{Remark}

\usepackage{xcolor}
\definecolor{locA}{HTML}{D85A30}
\definecolor{locB}{HTML}{1D9E75}
\definecolor{locC}{HTML}{534AB7}
\newcommand{\search}[2]{\textcolor{#1}{\texttt{search}(\texttt{#2})}}

\newcommand{\sB}{\search{locB}{B}}
\newcommand{\sC}{\search{locC}{C}}

\usepackage[textsize=tiny]{todonotes}
\title{Asymmetric On-Policy Distillation Learns In-Context Exploration}

\author{%
  \textbf{Sriyash Poddar}$^{1}$ \quad
  \textbf{Yanda Bao}$^{1}$ \quad
  \textbf{Jacob Krantz}$^{2}$ \quad
  \textbf{Matthew Chang}$^{2}$ \quad
  \textbf{Xavier Puig}$^{2}$ \\[0.4em]
  \textbf{Roozbeh Mottaghi}$^{2}$ \quad
  \textbf{Natasha Jaques}$^{1*}$ \quad
  \textbf{Abhishek Gupta}$^{1*}$ \\[0.6em]
  {\normalfont\footnotesize Project website: \url{https://weirdlabuw.github.io/asteroid/}}
}

\begin{document}

\maketitle

\begingroup
\renewcommand\thefootnote{}
\footnotetext{%
  \parbox[t]{\textwidth}{%
    $^{1}$Paul G. Allen School of Computer Science and Engineering, University of Washington \\
    $^{2}$Work partially done while at Meta \\
    $^{*}$Equal advising.%
  }%
}
\endgroup

\begin{abstract}
  Sequential decision-making under partial observability is difficult because agents must explore to infer hidden, task-relevant information before acting effectively. 
  Learning such information-gathering behaviors \emph{tabula rasa} via online reinforcement learning is sample-inefficient, whereas offline imitation averages expert actions across hidden states and fails to learn exploration. In this work, we propose \MethodS, a simple and effective student--teacher supervised learning framework that can learn test-time exploration. Although a privileged expert never needs to explore, we demonstrate that a history-conditioned student can learn to do so when trained with on-policy contexts: the student first collects on-policy context under partial observability, and then imitates the privileged expert conditioned on that context. By distilling expert actions directly conditioned on the student’s on-policy history, the student learns to systematically resolve partial observability, i.e., search while uncertain and exploit once the hidden state is revealed. This enables efficient test-time exploration (similar to Bayesian posterior sampling) without incurring the impractical training-data costs of reinforcement learning. Across a diverse range of eleven different environments, including discrete exploration benchmarks, robotic manipulation, navigation, and Procgen environments, \MethodS{} achieves $\sim$2$\times$ higher returns under fixed exploration budgets, uses up to $\sim$100$\times$ fewer interactions than RL baselines, and outperforms frontier models at exploration. Finally, our method transfers policies trained in simulation to vision-denied real-world robot manipulation, exhibiting efficient exploration and increasing success by $+40\%$ over standard distillation. Overall, it provides a simple yet surprisingly effective framework to learn efficient test-time exploration. 
\end{abstract}

\begin{figure}
    \centering
    \includegraphics[width=\linewidth]{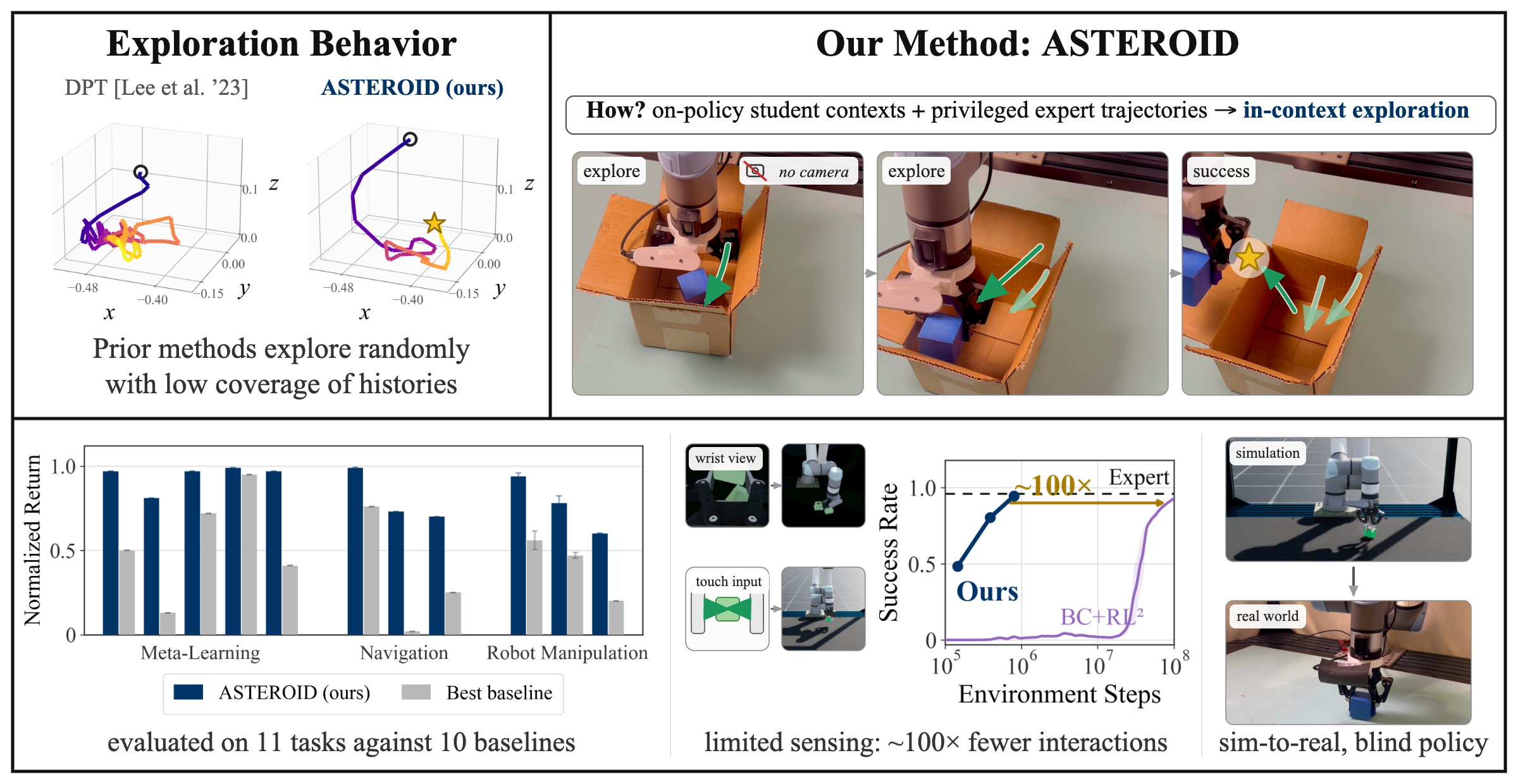}
    \caption{ \MethodS{} provides a simple and reliable method for learning efficient exploration in partially observed domains via on-policy asymmetric student-teacher distillation. 
}
    \label{fig:teaser_fig}
    \vspace{-0.4cm}
\end{figure}

\section{Introduction}\label{sec:intro}

Sequential decision-making under partial observability is a central challenge in reinforcement learning (RL)~\citep{kaelbling1998planning, osband2013psrl}. In the real world, from robotics to embodied AI, agents receive only partial, noisy observations of the environment. To act successfully, agents must first explore to uncover hidden, task-relevant information~\citep{Igl2018DeepVR, zintgraf2021varibad, lee2023supervised, advisor, memmel2024asid, wan2025enhancing}. For example, a household robot searching for an object must use onboard sensing to infer likely object locations, track locations it has already searched, and adapt future exploration based on prior attempts~\citep{Chaplot2020NeuralTS, chaplot2020object, hu2025realworld}. These problems are formalized as partially observable Markov decision processes (POMDPs)~\citep{kaelbling1998planning}.

Solving these tasks requires policies with memory~\citep{hausknecht2015deep, ni2021recurrent}. Learning them from scratch via recurrent RL is sample-inefficient, since the agent must learn to gather information and exploit it at the same time~\citep{bauer2023human, zintgraf2021varibad, duan2016rl}. One way to improve sample efficiency is to assume training-time access to a privileged expert that observes the hidden state and acts accordingly. For example, in household navigation, an oracle planner could use the full map of the training environments, which is unavailable at deployment.  In practice, these teachers can be trained in fully observed tasks using RL, imitation learning, or hybrid approaches~\citep{zhang2026balanced, yin2026emergent,schulman2017ppo,haarnoja2018sacapps,osa2018algorithmic}. This expert can teach the student how to act once the hidden state is known, but not how to find it~\citep{pmlrv139liu21s}. As a result, naively copying the expert does not teach the student to explore~\citep{walsman2023impossibly}. Interactive imitation~\citep{swamy2022sequence} offers a natural fix by relabeling student mistakes with expert actions. However, standard approaches~\citep{ross2011reduction} train the student to predict the average expert action across possible hidden states, so it dithers between hypotheses instead of resolving them~\citep{kim2025distilling, walsman2023impossibly}. Other approaches propose to address the student-teacher mismatch using expensive online RL~\citep{advisor}, differentiable experts~\citep{warrington2021robust}, or offline coverage assumptions~\citep{lee2023supervised}.

In this work, we propose \MethodS{} (\MethodH), a framework that shows that privileged experts can teach exploration through interactive imitation. The idea is simple: the student generates the history online, and the expert provides the behavior to imitate. \MethodS{} rolls out the current student under partial observability, producing a history of observations, actions, and rewards, then trains it to imitate the expert conditioned on that history. The expert never explores. Instead, the student's failed searches and probes become useful context: they specify evidence accumulated so far, while the expert provides next actions given the hidden state. With on-policy history, the policy can identify inconsistent hypotheses and adapt its exploration.

Why does this teach exploration? A partial history represents uncertainty over the hidden state. Returning to household navigation: suppose the object could be in location $A$, $B$, or $C$. If the student searches $A$ and finds nothing, its history rules out $A$ but is still consistent with $B$ and $C$. A privileged expert would then go to \sB{} when the object is at $B$, and \sC{} when it is at $C$. The student therefore learns to act and explore the remaining valid actions: if it fails at $A$, it samples \sB{} or \sC{}. If it then fails at $B$, every expert target is \sC{}, since only $C$ remains. Unlike random sampling of locations, the policy conditions on its history and never retries hypotheses it has ruled out. Across many training tasks, ambiguous histories are labeled with a distribution of potential, consistent expert trajectories. The student samples from this distribution to explore, and each new observation and feedback eliminates additional hypotheses. The key insight: the set of histories that the student needs to learn to adapt from (e.g., failing at $A$, failing at $B$ and then moving to $C$) is entirely covered by on-policy behavior. This is \emph{in-context exploration}, i.e., searching while uncertain, and exploiting once the hidden state is revealed. Additionally, we show that a curriculum over the student rollout horizon improves sample efficiency and enables longer-context adaptation (Section~\ref{sec:method}). \emph{Overall, \MethodS{} incorporates very simple changes into standard online student-teacher distillation to enable test-time exploration behavior}.

We show both empirically (Section~\ref{sec:exp}) and theoretically (Appendix~\ref{app:theory}) that \MethodS{} approximates Bayesian posterior sampling, a provably efficient exploration strategy~\citep{osband2013psrl}. However, posterior sampling has been hard to scale: prior approaches maintain an explicit posterior~\citep{asmuth2012boss} or assume infinite offline data~\citep{lee2023supervised}, which are intractable in realistic settings. \MethodS{} avoids both: it does not maintain an explicit posterior, and on-policy supervised learning avoids impractical data costs. As it is a simple extension of student--teacher distillation, \MethodS{} works with modern architectures and existing sim-to-real frameworks~\cite{yin2026emergent}. In summary, we make the following contributions:
\begin{enumerate}[topsep=1pt, itemsep=-0.5ex, partopsep=1ex, parsep=1ex, leftmargin=0.5cm]
    \item We present \MethodS, a simple student--teacher distillation framework that enables in-context exploration and test-time improvement learned solely from privileged experts.
    \item We evaluate our method across 11 environments, including navigation in realistic household simulators and real-world robot manipulation, where it outperforms RL and supervised baselines under fixed exploration budgets while using two orders of magnitude fewer training interactions.
    \item We show that learned policies exhibit efficient real-world exploration and, under appropriate assumptions, recover posterior sampling with error linear rather than quadratic in context length.
\end{enumerate}

\section{Related Work}\label{sec:related_work}

\textbf{Student--Teacher Distillation.}
Learning from privileged experts arises in settings from self-driving \citep{chen2020learning} to robotic exploration \citep{choudhury2017adaptive}. Interactive imitation addresses the covariate shift that drives the learner into unseen states \citep{ross2011reduction}, and prior work suggests that on-policy imitation is necessary under partial observability \citep{swamy2022sequence}. \MethodS{} is also on-policy, but only in its contexts, not in its action targets. Our work targets this non-realizable regime, where the learner must take information-gathering actions that are not recoverable from a fully observed expert \citep{spencer2021feedback, lee2021bayesian}. Prior approaches handle this via additional assumptions, such as differentiable experts \citep{warrington2021robust}, predicting hidden information~\citep{kumar2021rma}, or incorporating RL with bootstrapping to induce exploration \citep{advisor, walsman2023impossibly,shenfeld2023tgrl,kim2025distilling}. In contrast, we learn information-gathering behavior using iterative distillation of expert actions conditioned on on-policy histories, which is more sample-efficient than RL-based alternatives.

\textbf{Meta Learning.}
Meta-RL learns adaptive exploration across task distributions using recurrent policies~\citep{hausknecht2015deep, duan2016rl, zintgraf2021varibad,pmlrv139liu21s}, but requires extensive online interaction \citep{bauer2023human}, making it difficult to scale to high-dimensional, sparse-reward domains. Our approach replaces this expensive RL-based adaptation with supervised learning, making it efficient to learn long-horizon exploration. Further, unlike prior work limited to Bayes-Adaptive MDPs \citep{lee2021bayesian, zintgraf2021varibad, lee2023supervised}, our method scales to POMDPs, where the policy explores within a single episode to estimate the underlying state and achieve task success.

\textbf{Posterior Sampling.}
Posterior sampling methods \citep{thompson1933, osband2013psrl} drive exploration by maintaining a belief over hidden parameters, acting optimally under samples from this belief, and updating it based on feedback. Despite strong empirical and theoretical results \citep{chapelle2011thompsoneval, osband2013psrl}, such approaches require maintaining an explicit posterior and solving for policies under sampled parameters \citep{osband2013psrl, asmuth2012boss}, limiting scalability to high-dimensional settings. \MethodS{} does not explicitly optimize a posterior-sampling objective. Instead, posterior-sampling-like behavior emerges through iterative on-policy distillation, even in high-dimensional continuous environments (Section~\ref{sec:exp}), and we include theoretical insights to show convergence to posterior sampling on the covered history support (Appendix~\ref{app:theory}). In Section~\ref{sec:exp}, we show that current frontier models are inefficient at performing such exploration under partial observability, motivating the need to explicitly learn such behavior.

\textbf{In-Context Decision-Making.}
Transformers trained autoregressively on sequential data exhibit in-context learning \citep{brown2020language, garg2022can}, which Decision Transformer \citep{chen2021decisiontransformer} and Algorithm Distillation \citep{Laskin2023InContextRL} leverage for decision-making. These works focus on using online rollouts as context to adapt behavior without external prompts.  Closest to our work, Decision-Pretrained Transformer \citep{lee2023supervised} learns an in-context policy that performs posterior sampling by pairing randomly generated contexts with privileged expert supervision. Such methods assume the offline dataset covers all possible interaction histories, which is intractable in high-dimensional domains. \MethodS{}, instead, uses on-policy contexts, removing the coverage assumption and scaling to much harder, higher-dimensional problems (Section~\ref{sec:exp}).

We include an extended discussion of related work in Appendix~\ref{app:related_work}.

\section{Problem Statement}\label{sec:problem}

We study decision-making under \emph{partial observability}, modeled as a partially observable Markov decision process (POMDP) \citep{kaelbling1998planning}. A POMDP is represented by the tuple \((\mathcal{S},\mathcal{A}, \mathcal{O},E,T,R, \rho_0,\gamma)\), where $\mathcal{S}$ is the state space, $\mathcal{A}$ is the action space, $\mathcal{O}$ is the observation space and $E,T,R,\rho_0, \gamma$ are the emission function, transition kernel, reward function, initial state distribution and discount factor, respectively. In a POMDP, the policy does not observe the state $s_t$ but receives partial observations $o_t \sim E(\cdot \mid s_t)$. For simplicity, we also define the Bayes-Adaptive Markov decision process (BAMDP), a special case of POMDPs in which the only hidden variable is a latent parameter, denoted by $z$, i.e., $s_t = (o_t, z) $. For each interaction, $z\in\mathcal{Z}$ is sampled from a prior $p(z)$ and remains fixed for $K$ episodes (unlike a POMDP, where hidden states may vary at each step). Given $z$, BAMDPs are fully observed and can be represented as a collection of MDPs $M_z$ (or tasks).

As a running example, consider a three-door environment: a BAMDP with latent parameter $z$ denoting the target door. At the start of each interaction, $z$ is sampled uniformly from the three doors and remains fixed across multiple episodes, where the objective is to navigate to the correct door. The target is hidden, and the agent must use its interaction history to infer the hidden goal and then exploit that information. We can construct a POMDP by adding noise to agent observations at each time step. 

Since observations $o_t$ are partial, optimal policies depend on the current observation $o_t$ and the history $h_t = (o_0, a_0, r_0, \ldots, o_{t-1}, a_{t-1}, r_{t-1})$. Let $d_\pi(h)$ denote the distribution over histories for a policy $\pi$. Learning information-gathering behavior directly from sparse rewards is highly sample-inefficient, often requiring impractical amounts of experience~\citep{bauer2023human}. To mitigate this, we study an \emph{asymmetric} training--testing setting in which a privileged teacher $\pi^*(a_t\mid s_t)$ can access the true state $s_t$ during training, while the deployed student $\pi_\theta(a_t\mid o_t, h_t)$ acts under partial observability using $o_t$ and $h_t$. During training, the learner interacts with the POMDP to collect transition tuples $(o_t, a_t, r_t, o_{t+1})$ and queries the teacher for actions $a^*_t \sim \pi^*(\cdot \mid s_t)$. Our objective is to use these interactions and teacher supervision to learn a student $\pi_\theta(\cdot\mid o_t, h_t)$ that maximizes the expected return $J(\pi_\theta)=\mathbb{E}_{\pi_\theta}\!\left[\sum_{t=0}^{\infty}\gamma^t r_t\right]$. This asymmetric distillation setting arises naturally in navigation~\citep{jain2021gridtopix}, sim-to-real robotics~\citep{yin2026emergent}, human-in-the-loop learning~\citep{kelly2019hg}, and knowledge distillation~\citep{gu2024minillm}.

\section{Background: Student--Teacher Distillation}\label{sec:prelim}
In the following sections, we focus on the key paradigms for teacher--student distillation under partial observability, discuss why they break down at learning how to explore in asymmetric settings, and motivate \MethodS{} as an iterative on-policy distillation method for learning in-context exploration.

\paragraph{Naive asymmetric student--teacher distillation cannot learn adaptive exploration.} ~\label{sec:tsd}
\begin{wrapfigure}{r}{0.5\linewidth}
    \centering
    \includegraphics[width=\linewidth]{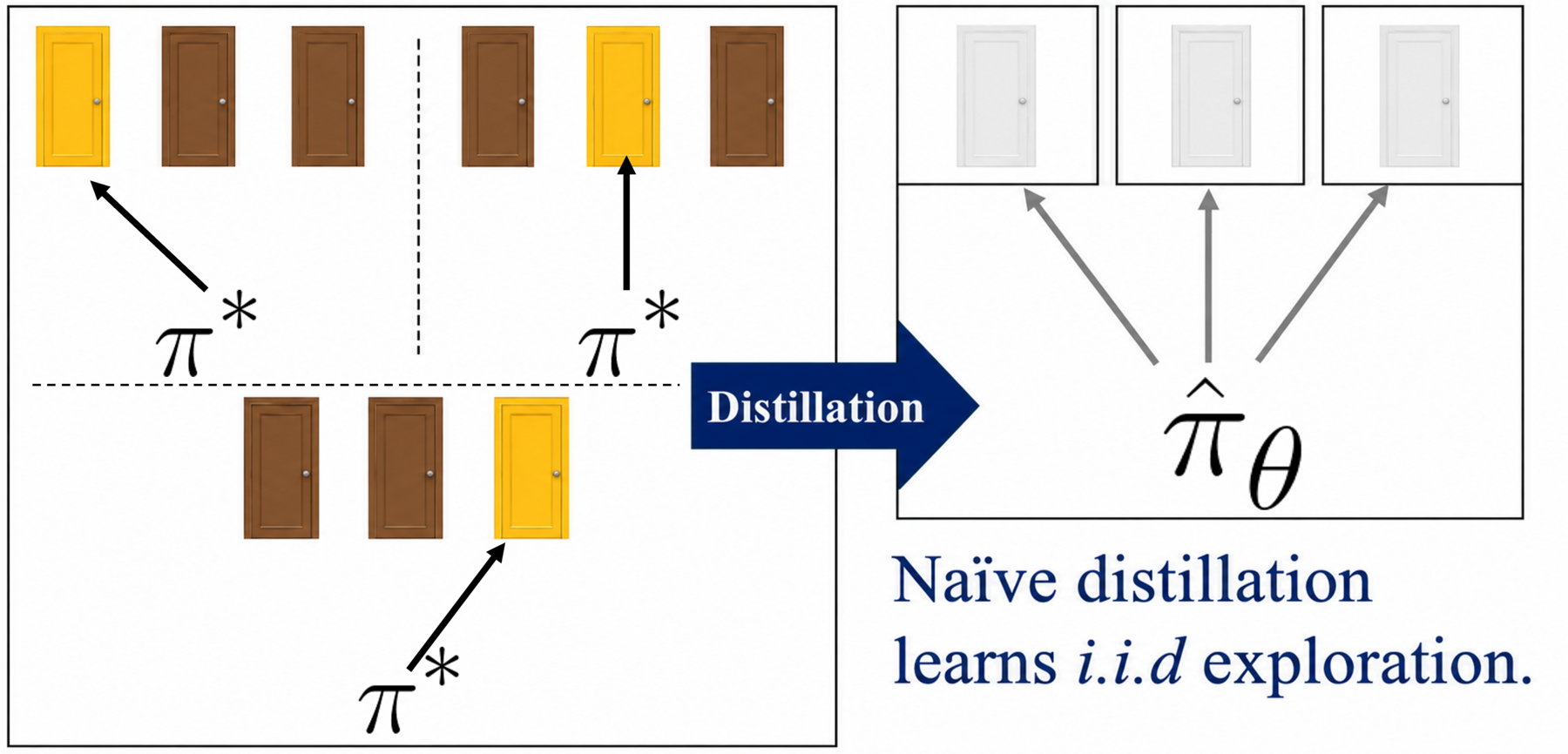}
    \label{fig:tsd}
\end{wrapfigure}
A natural approach to train partially observed student policies is via behavior-cloning trajectories generated solely by a privileged policy \cite{spoc2023ehsani, peng2018sim}. But there is a conflict: the privileged teacher policy $\pi^*(\cdot|s_t)$ solves the task without exploration, whereas the unprivileged student policy $\pi_\theta$ must gather information through exploration before exploiting for task success~\cite{hausknecht2015deep}. Following the three-door example from Section~\ref{sec:problem}, privileged experts observe $z$, while the student only receives partial observations $o_t$. At $t=0$, $o_0$ is identical for all values of $z$. Here, a student policy acting on partial observation $o_0$ cannot distinguish between actions provided for the three goals and learns the marginal expert behavior induced by the training data ${d}_{\pi^*}$, i.e., choosing a door uniformly at random. Thus,

\begin{badbox}
Naive asymmetric distillation from expert trajectories learns \emph{i.i.d.} exploration, but fails to adaptively exploit sequences of environment rewards and feedback.
\end{badbox}

\paragraph{In-context learning via Decision-Pretrained Transformer (DPT) \citep{lee2023supervised} can learn exploration.}

Recent work \citep{lee2023supervised} has shown that asymmetric student--teacher distillation can learn exploration when the training dataset is constructed with specific design choices. 
\citet{lee2023supervised} introduce DPT, an in-context learning framework that constructs a dataset of contexts, observations and target actions: $\mathcal{D}^\text{DPT}=\{(h_t,o_t,a^*_t)^i\}_{i=1}^N$ for supervised learning via random sampling. Here, $d_{\mathcal{D}}$ denotes the empirical distribution of histories and query observations in the dataset $\mathcal{D}$. DPT generates its offline histories with an independent behavior policy.
Further, it labels randomly sampled query observations $o_t \in \mathcal{O}$ with the privileged expert $a^*_t \sim \pi^*(\cdot|o_t,z)$. The history $h_t$ provides context about the underlying hidden parameter, which the policy learns to use at the query observation $o_t$ to predict actions. DPT trains a causal transformer policy maximizing action likelihood over this dataset:

\begingroup
\setlength{\abovedisplayskip}{1pt}
\setlength{\belowdisplayskip}{1pt}
\setlength{\abovedisplayshortskip}{1pt}
\setlength{\belowdisplayshortskip}{1pt}
\begin{equation}
\mathcal{L}(\theta)
\;=\;
\mathop{\mathbb{E}}_{(h_t, o_t, a^*_t) \sim {\mathcal{D}}}
\left[
-\log \pi_{\theta}\!\left(a_t^* \mid o_t, h_t \right)
\right].
\label{eq:loss}
\end{equation}
\endgroup

Under coverage assumptions, DPT provably converges to posterior sampling: it uses $h_t$ to form an implicit posterior over the latent state, and acts according to the expert policy averaged under this posterior, i.e., $\pi(a_t|o_t,h_t) = \int \pi^*(a_t|o_t,z) \;\mathrm{p}(z|o_t,h_t)\; dz$.

Let $\mathcal{H}_t =  (\mathcal{O}\times\mathcal{A}\times\mathcal{R})^{t}\times\mathcal{O}$ denote the space of length-$t$ history and observation tuples $(h_t, o_t)$.

\begin{badbox}
The primary limitation of DPT is the reliance on a full-coverage assumption: the dataset must cover the combinatorial space of arbitrary contexts and query observations, i.e., we require
\[  \mathop{\bigcup}\limits_{t=0}^{H} {\mathcal{H}_t} \subseteq \operatorname{supp}(d_\mathcal{D}),\]
which scales exponentially with the interaction horizon $H$.
\end{badbox}

\begin{wrapfigure}{r}{0.5\linewidth}
    \vspace{-0.5cm}
    \centering
    \includegraphics[width=\linewidth]{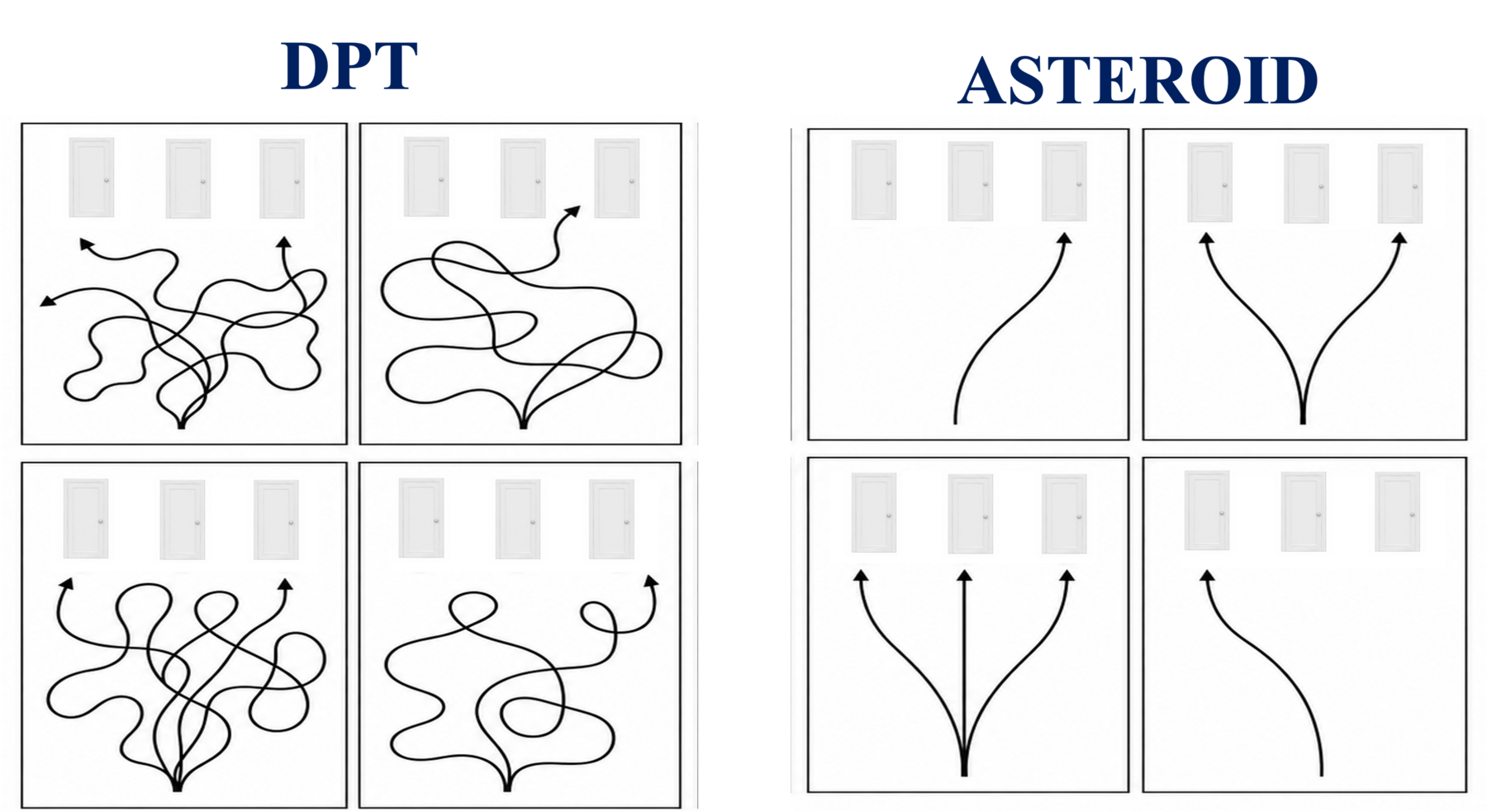}
    \caption{DPT needs intractable full-context coverage for training, while \MethodS{} generates on-policy rollouts to learn scalable exploration. \label{fig:coverage}}
    \vspace{-0.5cm}
\end{wrapfigure}

The above assumption holds in the infinite-data limit, when randomly sampling histories and query observations provides complete coverage (see Figure~\ref{fig:coverage}). Under finite data, this quickly becomes impractical in high-dimensional observation and action spaces, where the number of possible histories grows exponentially, and informative feedback is sparse~\citep{Bellman1966}. 
Consequently, purely offline context datasets do not scale to learning exploration in realistic, large-scale decision-making problems. This has fundamentally limited the applicability of methods like DPT in robotics and other high-dimensional problems.

\section{\MethodH}\label{sec:method}

\begin{figure*}
    \centering
    \includegraphics[width=\linewidth]{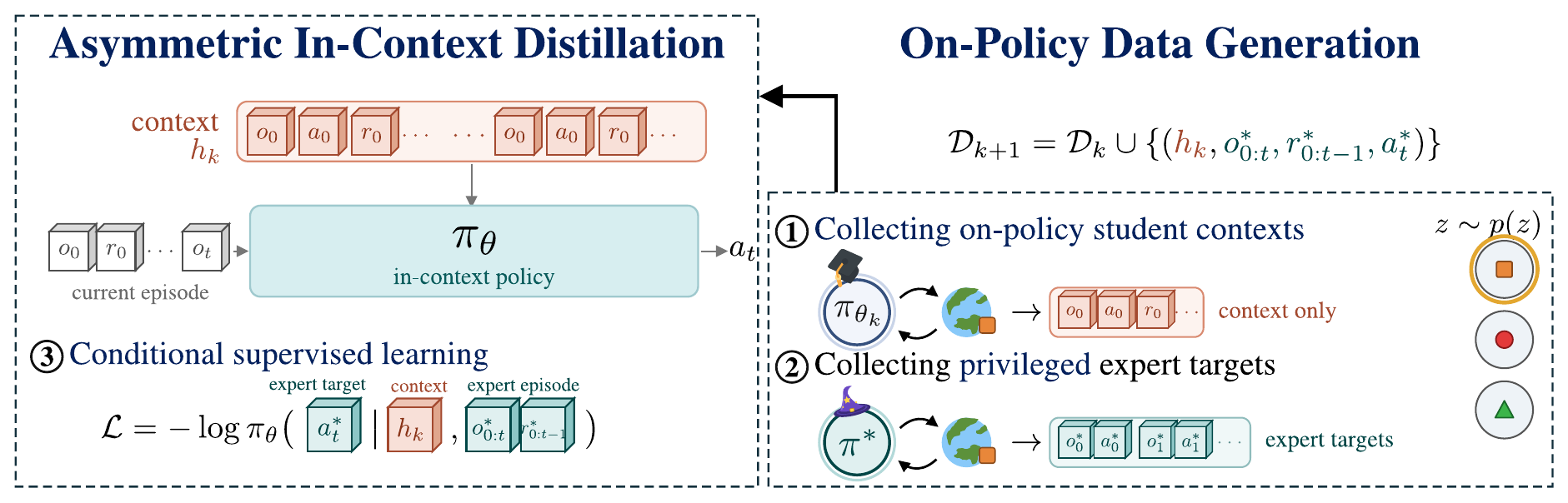}
    \caption{In each iteration, \MethodS{} (1) collects on-policy context from the student $\pi_{\theta_k}$ and then (2) expert trajectories from $\pi^*$, under the same hidden parameters $z$. The student rollouts are used as context, while expert behavior is used as the target. With this dataset, we can (3) train policies that perform in-context exploration.}
    \label{fig:method}
\end{figure*}

In Section~\ref{sec:prelim}, we identified a central bottleneck in supervised learning for in-context exploration: generating data that covers the histories a partially observed agent will experience at deployment.

\begin{goodbox}
Our key insight is that the training dataset $\mathcal{D}$ should not necessarily cover the entire combinatorial history space $\mathcal{H}_t$, but simply the distribution of trajectories generated by the student itself: the ``on-policy'' rollouts $h_t \sim d_{\pi_\theta}$ coupled with expert labels provide sufficient coverage in the dataset to learn an adaptive exploration strategy, i.e. 
\[\operatorname{supp}(d_{\pi_\theta}) \subseteq \operatorname{supp}(d_{\mathcal{D}})\]
\end{goodbox}

Consequently, this relaxes the strong assumption of infinite coverage over histories that renders DPT intractable. Our method \MethodS{} is an iterative asymmetric distillation procedure that alternates between collecting on-policy histories from the student and training a history-conditioned student policy from privileged expert labels. Collecting supervision on student-induced histories reduces the distribution shift between training and deployment. Unlike standard interactive imitation~\citep{ross2011reduction}, which aggregates per-step action labels, \MethodS{} supervises with full expert trajectories and conditions on environment feedback (Appendix~\ref{app:dagger}).

\theoremstyle{definition}
\newtheorem*{informaltheorem}{(Informal) Theorem}
\theoremstyle{plain}

\begin{goodbox}
\begin{informaltheorem}
Under compliance (i.e., independence between the sources of on-policy contexts and supervision labels), on-policy coverage, and bounded episode cost, \MethodS{} approximately
recovers posterior sampling on the histories it visits. Given a supervised learner with
average error $\epsilon$ on its own histories, its gap to posterior sampling over $K$ episodes
is $O(\epsilon K)$. The total regret is 
\[
\mathrm{Reg}_{\mathrm{PS}}(K) + O(\epsilon K)
\]
i.e., posterior sampling's regret plus a linear approximation term.
\end{informaltheorem}
\end{goodbox}

\noindent The formal statements and proofs are given in
Appendix~\ref{app:theory}. The key step is that expert actions paired with 
the student's own histories are exactly posterior-sampling labels
(Lemma~\ref{lem:teacher_ps}). In addition, we show that an offline
in-context learner with the same training error can incur $O(\epsilon K^2)$
error (Lemma~\ref{lem:offline_compounding}), because the test-time histories it
encounters drift away from its fixed training distribution.

\begin{algorithm}[H]
\caption{\MethodS}
\label{alg:method}
\begin{algorithmic}[1]
\setlength{\itemsep}{0pt}
\setlength{\parskip}{0pt}
\REQUIRE iterations $K$, tasks per iteration $N$, horizon $H$, privileged teacher $\pi^*$
\STATE Initialize $\mathcal{D}_0 \leftarrow \text{expert trajectories}$, student $\pi_{\theta_0} \leftarrow \textit{i.i.d}\;\text{policy}$

\FOR{$k=1,\dots,K$}
    \STATE $\mathcal{D}_{k} \leftarrow \mathcal{D}_{k-1}$
    \STATE \texttt{\textcolor{NavyBlue}{// Collect on-policy contexts and expert targets}}
    \FOR{$i=1,\dots,N$}
        \STATE Sample $z \sim p(z)$ and instantiate $M_z$
        \STATE Roll out $\pi_{\theta_{k-1}}$ in $M_z$ for $k$ episodes: $h_k^{(i)} = (\tau^0,\dots,\tau^{k-1})$
        \STATE Roll out $\pi^*$ in $M_z$ (or use offline expert data):
        $\tau^{*,(i)} = (o^*_0,a^*_0,r^*_0,\dots,o^*_{H-1},a^*_{H-1},r^*_{H-1})$
        \STATE Add labels $ \mathcal{D}_{k} \leftarrow \mathcal{D}_{k}\cup\{(h_k^{(i)}, o^*_{0:t}, r^*_{0:t-1}, a^*_t)\}_{t=0}^{H-1}$
    \ENDFOR

    \STATE Train $\pi_{\theta_{k}}$ on $\mathcal{D}_{k}$ using Eq.~\ref{eq:loss} \texttt{\textcolor{NavyBlue}{// Imitate experts conditioned on context}}
\ENDFOR

\STATE \textbf{return} $\pi_{\theta_K}$

\end{algorithmic}
\end{algorithm}

\subsection{How does \MethodS{} learn in-context exploration?}\label{sec:asteroid_practice}

Figure~\ref{fig:method} and Algorithm~\ref{alg:method} summarize \MethodS{}. Let $\pi_{\theta_k}$ denote the student policy and $\mathcal{D}_k$ the dataset at iteration $k$. We first instantiate the method in a BAMDP, where a latent $z \sim p(z)$ stays fixed while the policy interacts for $k$ episodes, and extend it to general POMDPs in Section~\ref{sec:pomdp}. Each iteration has two phases. In the \emph{data generation} phase, we roll out the student for $k$ episodes to produce the on-policy context $h = (\tau^0, \dots, \tau^{k-1})$, then roll out the privileged expert $\pi^*(\cdot \mid o, z)$ under the same $z$ to produce an expert trajectory $\tau^*$. The expert observations $o^*_{0:t}$ and rewards $r^*_{0:t-1}$ from $\tau^*$ serve as the query, and the expert action $a^*_t$ as the label. Unlike DPT, which samples query observations randomly, our queries come from expert rollouts. In the \emph{supervised learning} phase, we train $\pi_{\theta_k}$ to predict $a^*_t$ given $(h, o^*_{0:t}, r^*_{0:t-1})$ with the DPT loss (Eq.~\ref{eq:loss}). In practice, we use a causal transformer~\citep{lee2023supervised} and compute the loss only on expert actions. We repeat this process of collecting on-policy contexts and expert labels, followed by supervised learning iteratively. After iteration $k$, $\pi_{\theta_k}$ adapts using up to $k$ test-time episodes. Revisiting the three-door problem from Section~\ref{sec:problem},

\setlength{\intextsep}{2pt}
\setlength{\columnsep}{10pt}
\begin{wrapfigure}{r}{0.55\linewidth}
    \centering
    \includegraphics[width=\linewidth]{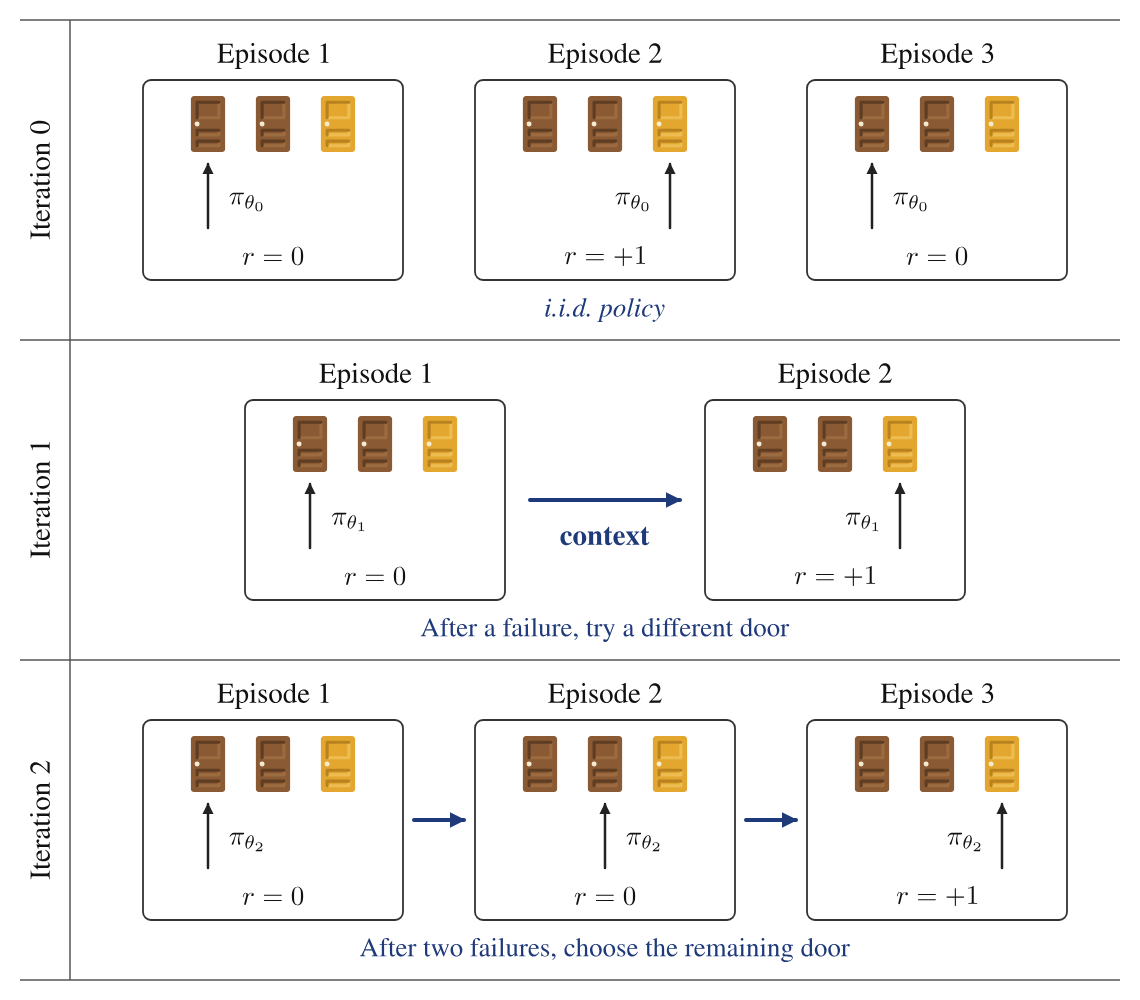}
    \caption{We iteratively learn to adapt exploration using feedback over longer contexts.}
    \label{fig:asteroid-three-door}
    \vspace{-0.25cm}
\end{wrapfigure}


\noindent\textbf{Iteration 0.}
$\pi_{\theta_0}$ learns an \emph{i.i.d.} policy with no in-context learning. It randomly chooses a door every episode, ignoring previous episodes.

\noindent\textbf{Iteration 1.}
$\pi_{\theta_0}$ generates an on-policy trajectory, which visits a random door. The reward reveals whether that door is correct: $+1$ if correct, $0$ otherwise. The expert label then returns to the same door on success or points to a different door after failure. Thus, $\pi_{\theta_1}$ learns to use the first episode as context for choosing the second. However, $\pi_{\theta_1}$ has a context window of one episode, so it cannot reason over more than one attempt, motivating more iterations to expand the window.

\noindent\textbf{Iteration 2.}
$\pi_{\theta_1}$ now generates two context trajectories: it repeats a successful door, or tests a new door after failure. If both tested doors fail, the expert label points to the remaining door. Thus, $\pi_{\theta_2}$ learns in-context exploration that identifies the true goal in this partially observed task.

To further illustrate \MethodS, Appendix~\ref{app:toy} provides another dry run showing that the student behavior approaches posterior sampling. We discuss certain practical details of the method below.

\subsection{Practical Implementation}

\textbf{Initialization.}
We initialize $\mathcal{D}_0$ with privileged expert trajectories. The initial policy $\pi_{\theta_0}$ is therefore trained without contextual information and resembles naive student--teacher distillation; see Section~\ref{sec:tsd}.

\textbf{Deployment.}
At deployment, the policy starts from a null context, $h_0=\emptyset$ (see Algorithm~\ref{alg:deploy}). We do not use any training context; at each step, the policy uses online interaction to build context and infer hidden information. For evaluations longer than the training context length, we use a sliding window.

\begin{algorithm}[H]
\caption{\MethodS{} Deployment}
\label{alg:deploy}
\begin{algorithmic}[1]
\setlength{\itemsep}{0pt}
\setlength{\parskip}{0pt}
\REQUIRE trained student $\pi_{\theta_K}$, horizon $T$
\STATE Initialize context $h \leftarrow \emptyset$
\FOR{$t=0,\dots,T-1$}
    \STATE Observe $o_t$, sample $a_t \sim \pi_{\theta_K}(\cdot \mid h, o_t)$
    \STATE Execute $a_t$, receive $r_t$
    \STATE Update context $h \leftarrow (h, o_t, a_t, r_t)$ \texttt{\textcolor{NavyBlue}{// Build context from online interaction}}
\ENDFOR
\end{algorithmic}
\end{algorithm}

\textbf{Context-length curriculum.}
At each iteration $k$, the policy is rolled out for $k$ episodes, inducing a curriculum on rollout horizons. As the horizon expands, the student can eliminate task instances that are inconsistent with its history over a longer horizon, making exploration efficient. This also makes our method more sample-efficient, and we discuss this further in Appendix \ref{app:curriculum_ablation}.

\subsection{How does \MethodS{} work on arbitrary POMDPs?}\label{sec:pomdp}

We described \MethodS{} in a BAMDP, where a latent variable $z$ remains fixed across multiple episodes. In a general POMDP, the hidden state evolves within a single episode. For example, in a locally observed maze, the agent does not observe its global position, and this hidden position changes after every move. At test time, the policy interacts with the environment for a single episode over a fixed horizon $H$. As the underlying state changes at each step, the agent faces a branching inference problem: each action and observation changes which hidden states remain plausible. Thus, instead of coverage over repeated episodes for a fixed $z$, we need coverage over on-policy histories $h_t$ across timesteps $t\in[H]$. 

In practice (see Algorithm~\ref{alg:method2}), \MethodS{} samples a timestep $t$, rolls out the student to produce an on-policy history $h_t$, and rolls out the expert from the reached state, producing a partial trajectory with observations $o^*_{t:}$ and action labels $a^*_{t:}$ (similar to takeovers in~\cite{kelly2019hg}). The resulting tuples $(h_t, o^*_{t:}, a^*_{t:})$ train the policy to map its history to actions that
are optimal for the hidden state under that history. At the takeover step, it
approximates $\pi(a_t \mid o_t, h_t) = \int \pi^*(a_t \mid s_t)\, p(s_t \mid o_t, h_t)\, ds_t$, and over
the future steps, the expert trajectory teaches it to follow one hypothesis
rather than dithering.
For example, in the maze, after each move, the local observations and rewards prune some possible positions while keeping others plausible. Further, we induce a curriculum in the context horizon by iteratively biasing the sampled timestep $t$ toward longer horizons. Thus, we adapt the method to a POMDP: collect on-policy histories, then use privileged labels to teach it how to guide future behavior.

\section{Experiments}\label{sec:exp}

In our experiments, we aim to answer the following questions: 
\begin{enumerate}[topsep=1pt, itemsep=-0.5ex, partopsep=1ex, parsep=1ex, leftmargin=*]
    \item Does \MethodS{} learn to solve partially observed problems efficiently on deployment?
    \item Is \MethodS{} sample-efficient compared to baselines?
    \item What kind of exploratory behavior do policies trained with \MethodS{} display at test time?
    \item Does \MethodS{} learn exploration policies that can be deployed on real-robot hardware?
    \item Do frontier models perform in-context exploration?
    \item Is learning from on-policy contexts necessary to learn exploration efficiently?
\end{enumerate}

\subsection{Training and Evaluation Details}

We evaluate all methods across 11 environments and three seeds\footnote{We run Habitat on a single seed and the wrist-camera RL baseline on two seeds due to high wall-clock time.}. We include error bars denoting the standard error across the seeds. For model architecture and tokenization, we use the causal transformer implementation from~\cite{lee2023supervised}. We include the hyperparameters, training, and environment details in Appendix~\ref{app:experiments}.  We also include ablations over \MethodS{} in Appendix~\ref{app:ablation}, and qualitative videos on the website\footnote{Project
webpage: \href{https://weirdlabuw.github.io/asteroid/}{weirdlabuw.github.io/asteroid}}. 
We evaluate the policy's performance and exploration efficiency across multiple test-time episodes on held-out goals. We compare policy performance across training iterations, with a fixed test-time budget.

\begin{figure}
    \centering
    \includegraphics[width=\linewidth]{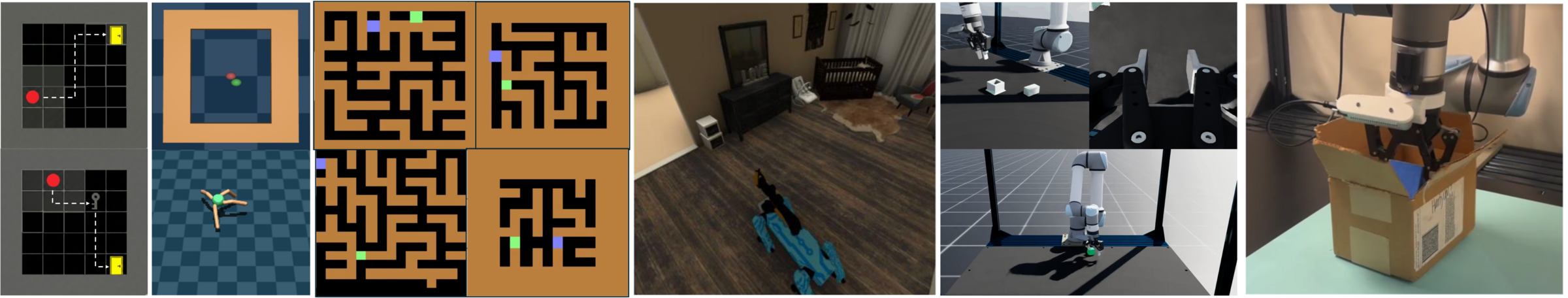}
    \caption{Tasks: (left to right) meta-learning gridworld benchmark, partially observed navigation, Procgen maze navigation, embodied object search, simulated robot and real robot manipulation.
    \label{fig:envs}}
\end{figure}

\subsection{Tasks}
We evaluate our approach on a diverse suite of partially observable decision-making problems across robotic manipulation, discrete exploration, and continuous navigation (Figure~\ref{fig:envs}), requiring strong elements of exploration, memory, and adaptation under limited sensing.

\begin{itemize}[topsep=1pt, itemsep=-0.5ex, partopsep=1ex, parsep=1ex, leftmargin=*]
    \item  \textbf{Partially observed gridworlds~\cite{MinigridMiniworld23}} are a classic benchmark consisting of discrete search problems, where an agent acts over multiple episodes in a gridworld to reach one or more goals. The agent is in an $n\times n$ grid and must perform targeted exploration to solve the task, such as finding the hidden goal or sequentially finding the hidden key and the door across multiple episodes.

    \item \textbf{Meta-Learning navigation~\cite{Rakelly2019EfficientOM}} is a meta-learning task where an agent must navigate in an environment to a hidden goal. We instantiate (i) Ant Navigation, where the agent is an 8-DoF robot, and (ii) 2D Navigation, which reduces the dynamics to a simple point mass.

    \item \textbf{Procgen Maze Navigation \citep{cobbe2019procgen}} involves finding the goal in procedurally generated mazes, where the agent receives only a local $3\times3$ observation and has to infer the maze layout and goal location using the interaction history.
    
    \item \textbf{Embodied Object Search~\cite{puig2024habitat}} involves searching for a household object in a fixed environment in Habitat-Sim. This presents a visually complex domain for a realistic partially observed problem.
    
    \item \textbf{Simulated Robotic Manipulation~\cite{yin2026emergent}} tasks are high-dimensional, continuous-control manipulation tasks under partial observability: Wrist-Camera Peg-Insertion and Tactile-Pick. Here, the robot has to use limited sensing to manipulate objects, i.e., explore over the table using only the wrist view, or use touch sensing to locate the object and complete the given task.

    \item \textbf{Real Robot Manipulation~\cite{zhang2026balanced,bronars2026blindpreciseassemblyproprioceptive}}: We reconstruct one of the simulated tasks in the real world, where the robot has to pick up a cube using proprioception from the arm and gripper only. We train policies in simulation with domain randomization~\cite{tobin2017domain, yin2026emergent} and then transfer them to the real world.

\end{itemize}

\subsection{Baselines}
We compare against ten strong baselines from several frameworks under partial observability:

\begin{itemize}[topsep=1pt, itemsep=-0.5ex, partopsep=1ex, parsep=1ex, leftmargin=*]
    \item \textbf{Asymmetric student--teacher:} \textbf{BC}: imitates privileged expert trajectories, \textbf{DAgger}~\cite{ross2011reduction}: relabels the student's per-step actions with the expert. \textbf{DPT}~\cite{lee2023supervised} trains a transformer on an in-context dataset of randomly collected interactions to predict oracle actions;
    \textbf{SPOC}~\cite{spoc2023ehsani} imitates shortest-path trajectories with data augmentation, yielding exploration-like behavior.
    \textbf{AAWR}~\cite{hu2025realworld} extends Advantage Weighted Regression~\cite{AWRPeng19} to POMDPs, using privileged state during training to learn a stronger critic.
    \item \textbf{Meta-RL:}
    \textbf{RL$^2$}~\citep{duan2016rl} trains a recurrent policy to adapt to the hidden parameters purely through interaction, treating the policy's hidden state as an implicit learning algorithm that updates from histories across episodes.
    \textbf{Distilled Planner}~\cite{puig2024habitat,osband2013psrl}: For embodied navigation in Habitat, we train a policy that performs posterior-sampling RL (PSRL)~\cite{osband2013psrl} (i.e., upper bound performance).
    \item \textbf{Imitation-informed RL:}
    \textbf{BC+PPO}~\cite{schulman2017ppo, duan2016rl} accelerates PPO~\cite{schulman2017ppo} by adding a behavior-cloning loss over privileged actions; for meta-learning tasks, we use an adapted variant (\textbf{BC+RL$^2$}) that leverages privileged supervision to bootstrap learning.
    \textbf{ADVISOR}~\cite{advisor} dynamically combines imitation learning (IL) and online RL objectives during training, resolving the imitation gap adaptively. \textbf{RMA}~\cite{kumar2021rma} learns to predict the privileged latent from history, and a frozen expert policy acts on that predicted latent.
\end{itemize}

\subsection{Does \MethodS{} learn to solve partially observed problems efficiently on deployment?}\label{sec:better_explore}

\begin{figure*}
    \centering
    \includegraphics[width=\linewidth]{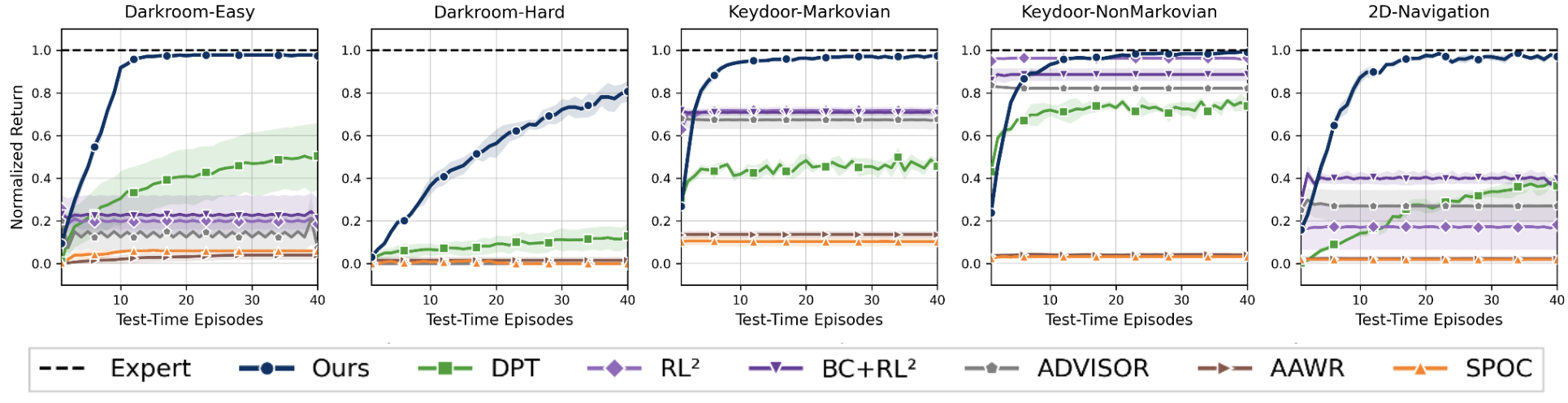}
    \caption{Performance on discrete exploration tasks requiring memory and goal inference. We observe that \MethodS{} outperforms all baselines in terms of test-time efficiency. It performs efficient exploration and almost converges to the expert after inferring the hidden task variables.\label{fig:incontext_pomdp_tasks}}
\end{figure*}

In Figure~\ref{fig:incontext_pomdp_tasks}, we compare the test-time performance of in-context policies learned via \MethodS{} against baselines. Given the same test-time budget, our method achieves substantially higher final performance. Unlike the baselines, \MethodS{}'s on-policy data generation and expert supervision provide coverage over the context the policy sees and the subsequent actions it should take at deployment, i.e., it can effectively adapt its exploration using these contexts to infer task-relevant information. Figures~\ref{fig:pomdp_manipulation},~\ref{fig:ant_behavior}, and~\ref{fig:extra_envs} show the same trend in partially observed navigation and manipulation. In Procgen Maze (Figure~\ref{fig:procgen}), \MethodS{} solves procedurally generated mazes by using local observation history to maintain an implicit belief over the state, guiding exploration toward unexplored regions~\citep{karkus2017qmdp}. This avoids redundant search and enables \MethodS{} to achieve higher success while being $2\times$ faster.
In contrast, BC achieves non-trivial success through random exploration but does not use context to avoid redundant search, making it highly inefficient. DPT also underperforms because offline datasets provide incomplete coverage over test-time histories. For RMA, latent uncertainty generates out-of-distribution (OOD) inputs for the base policy, which is trained only under full observability. As a result, it never learns to explore. \MethodS{} instead directly supervises the action distribution conditioned on the student's own context to perform directed exploration. In \MethodS{}, exploration is the learned behavior, not a byproduct. Finally, as we discuss in Appendix~\ref{app:dagger}, we see that DAgger under partial observability leads to performance degradation (below BC), while history conditioning and multi-step supervision in \MethodS{}  learn an effective search policy. We also show robustness to noisy experts in Appendix~\ref{app:noisy_ablation}.  Videos of the learned behavior are on the website\footnote{Project
webpage: \href{https://weirdlabuw.github.io/asteroid/}{weirdlabuw.github.io/asteroid}}.

\subsection{Is \MethodS{} sample-efficient?}\label{sec:efficiency}

\begin{figure*}
\centering
\begin{minipage}{0.42\textwidth}
    \centering
    \includegraphics[width=\linewidth]{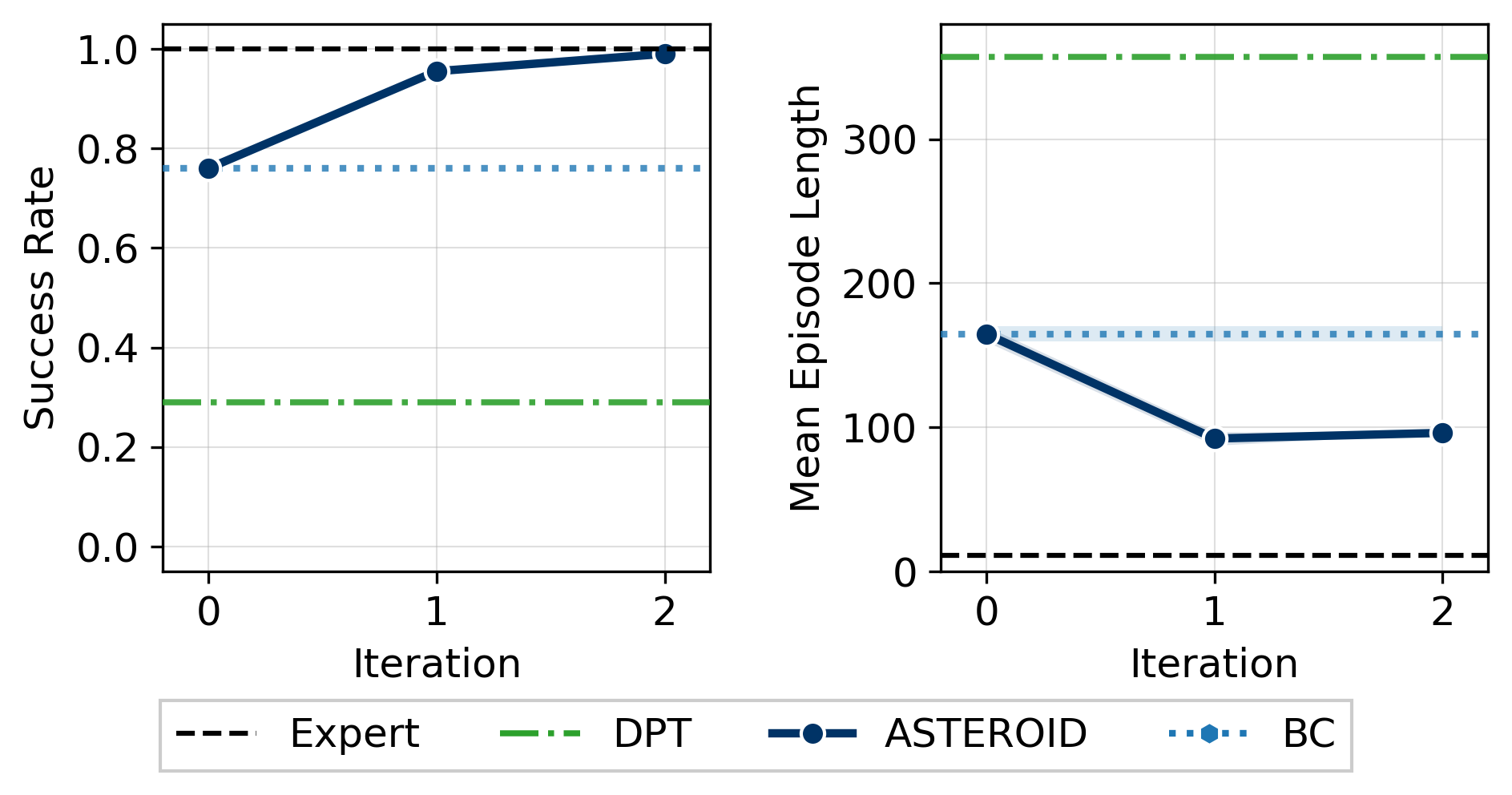}
    \caption{\MethodS{} improves success while reducing episode length on Procgen Maze; BC succeeds non-trivially via random exploration but is much less efficient.\label{fig:procgen}}
\end{minipage}\hfill
\begin{minipage}{0.53\textwidth}
    \centering
    \includegraphics[width=\linewidth]{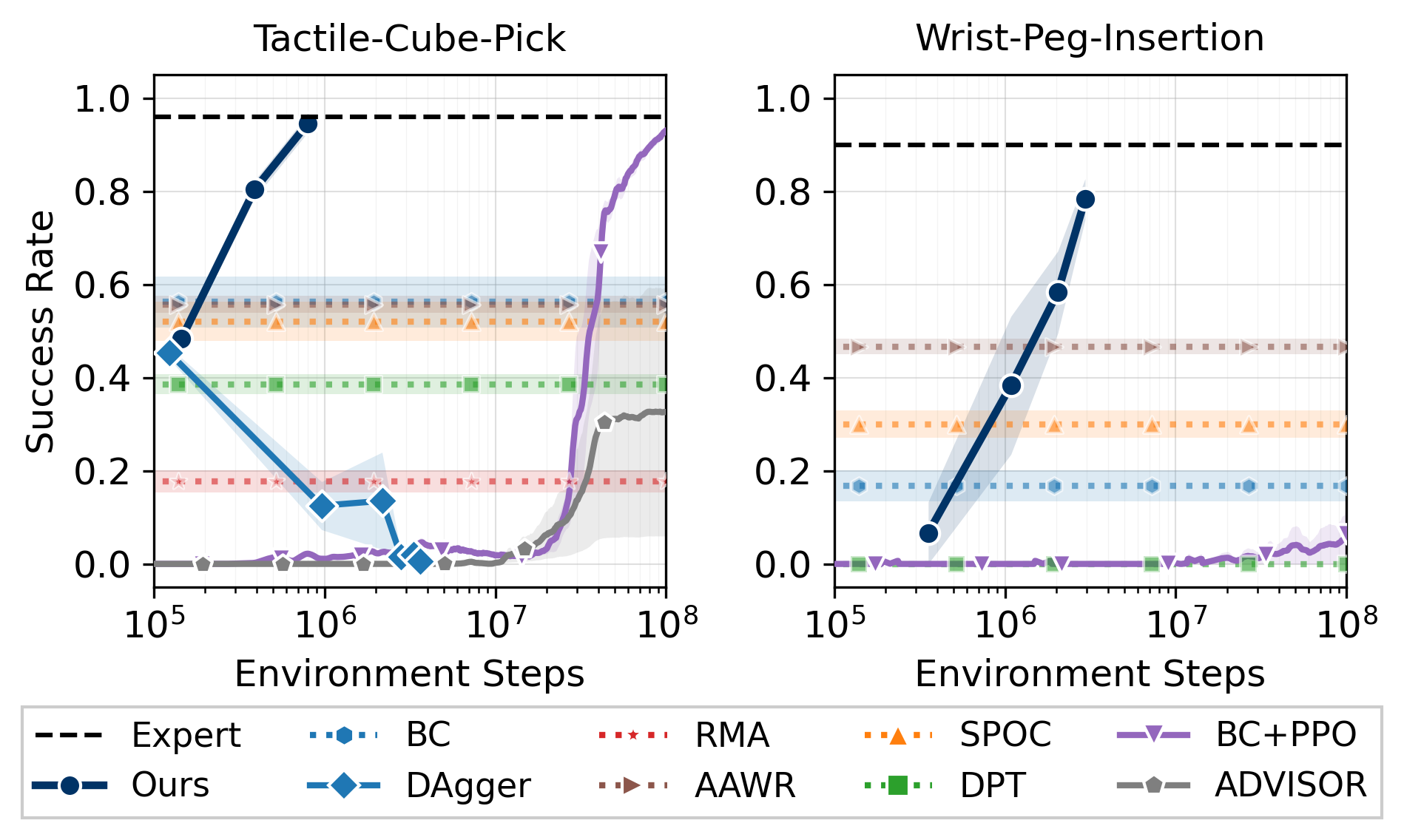}
    \captionof{figure}{We observe extreme sample efficiency ($\sim100\times$ speedup) when using \MethodS{} to train policies for robot manipulation under limited sensing.\label{fig:pomdp_manipulation}}
\end{minipage}
\end{figure*}

In this section, we compare the training-time sample efficiency of \MethodS{} against baselines that combine RL with privileged supervision to learn policies under partial observability. Using GPU-parallelized simulation~\cite{mittal2023orbit}, we evaluate manipulation tasks in which the policy has access to limited sensory modalities. In Figure~\ref{fig:pomdp_manipulation}, \MethodS{} is $\sim100\times$ more efficient in environment steps than RL baselines, learning a policy that explores and successfully picks the cube or inserts the peg. In Figure~\ref{fig:no_proprioception}, the student observes only the history of past actions and tactile sensing (no proprioception), and \MethodS{} efficiently learns policies with $\sim92\%$ success rate under this partial observability.
The supervised updates in \MethodS{} (compared to expensive RL) make it highly efficient during training, even for high-dimensional tasks. While RL baselines can learn in low-dimensional settings (Appendix~\ref{app:rl_ablation}), they are significantly sample-inefficient in domains such as ant navigation and robotics (Figures~\ref{fig:pomdp_manipulation} and~\ref{fig:ant_behavior}).

\subsection{What kind of in-context exploratory behavior do policies learned via \MethodS{} display?}

\begin{figure*}
    \centering
    \includegraphics[width=0.98\linewidth]{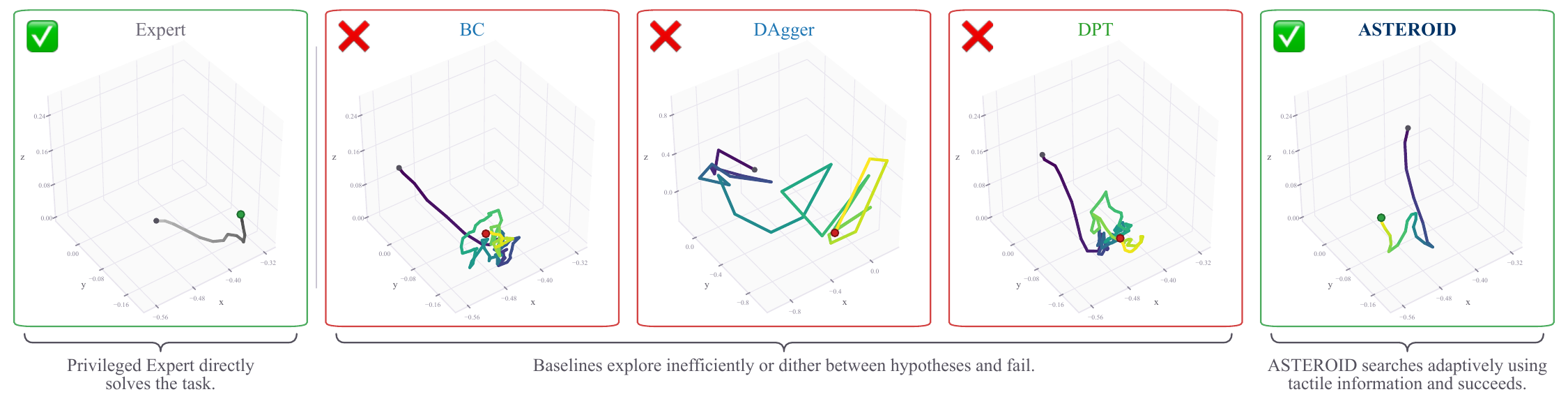}
    \caption{End-effector trajectories on the tactile cube-pick task, where the policy must locate and lift a cube using proprioception and touch sensing. The privileged state expert (gray) observes the cube pose and grasps it directly. Under partial observability, the baselines fail to adapt their actions, whereas \MethodS{} performs in-context exploration, exploring the table until contact, then grasping and lifting the cube. \label{fig:robot_exp_3ed}}
    \vspace{-0.3cm}
\end{figure*}

In Tactile-Pick (Figure~\ref{fig:robot_exp_3ed}), the policy probes multiple grasp points and uses tactile feedback from failed attempts to inform its next actions, unlike BC, which treats trials independently and relies on i.i.d.\ exploration. Additionally, other baselines fail to learn directed exploration either because of incompatible supervision or inefficient context generation during training. Similarly, in Figure~\ref{fig:explore_wrist}, the robot policy scans the tabletop using its wrist camera to localize both the object and target. In contrast, naive distillation goes out of distribution and degenerates.

\begin{figure}
    \centering
    \begin{minipage}[t]{0.48\linewidth}
        \centering
        \vbox to 4cm{%
            \includegraphics[width=\linewidth, height=4cm,
                keepaspectratio]{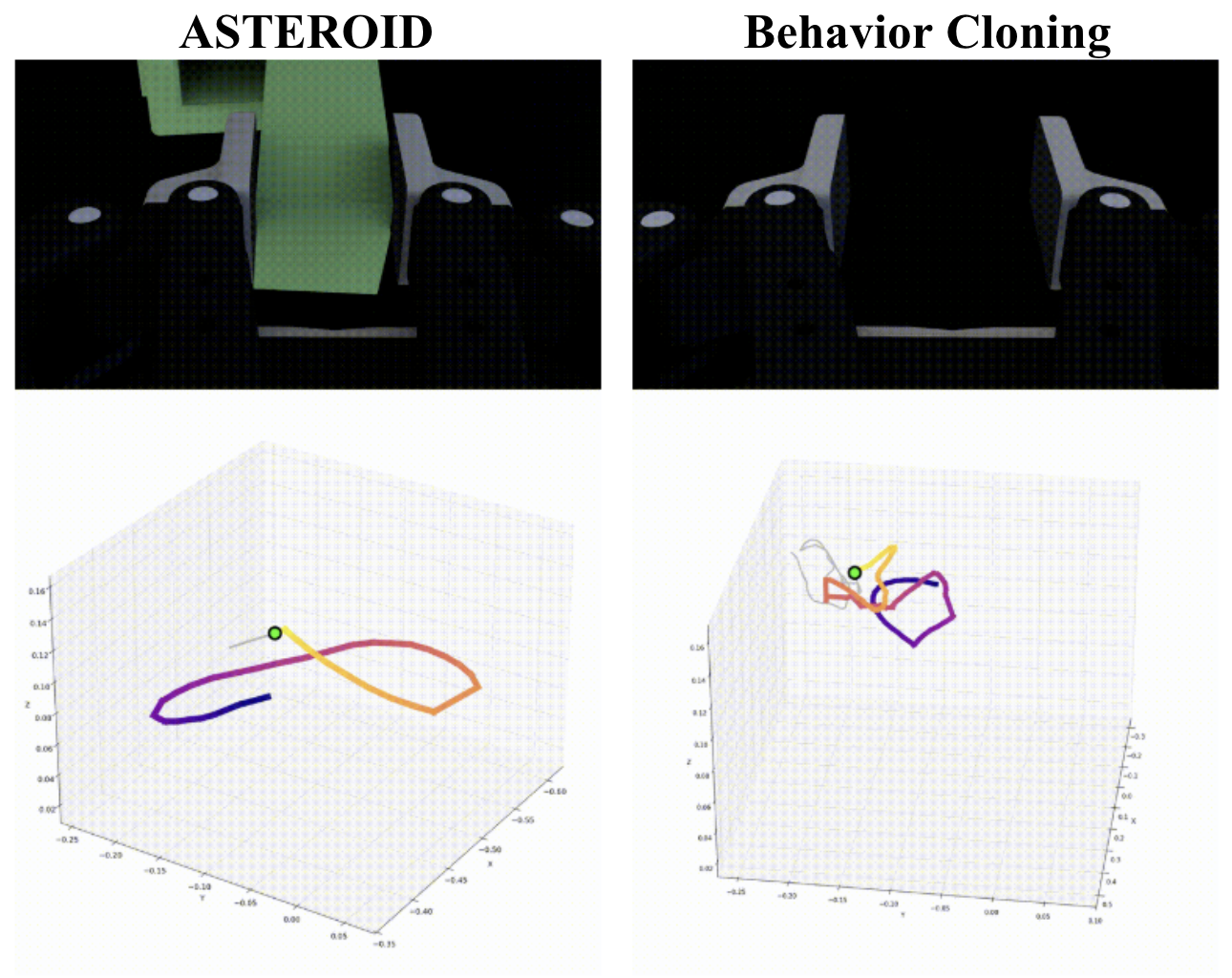}%
            \vfil
        }
        \caption{Robot positions (bottom row) during exploration for Peg-Insertion from wrist images (top row). \MethodS{} learns policies that sweep the table to locate the peg and successfully place it in the hole, whereas direct imitation policies collapse under limited observability.}
        \label{fig:explore_wrist}
    \end{minipage}\hfill%
    \begin{minipage}[t]{0.48\linewidth}
        \centering
        \vbox to 4cm{%
            \includegraphics[width=\linewidth, height=4cm,
                keepaspectratio]{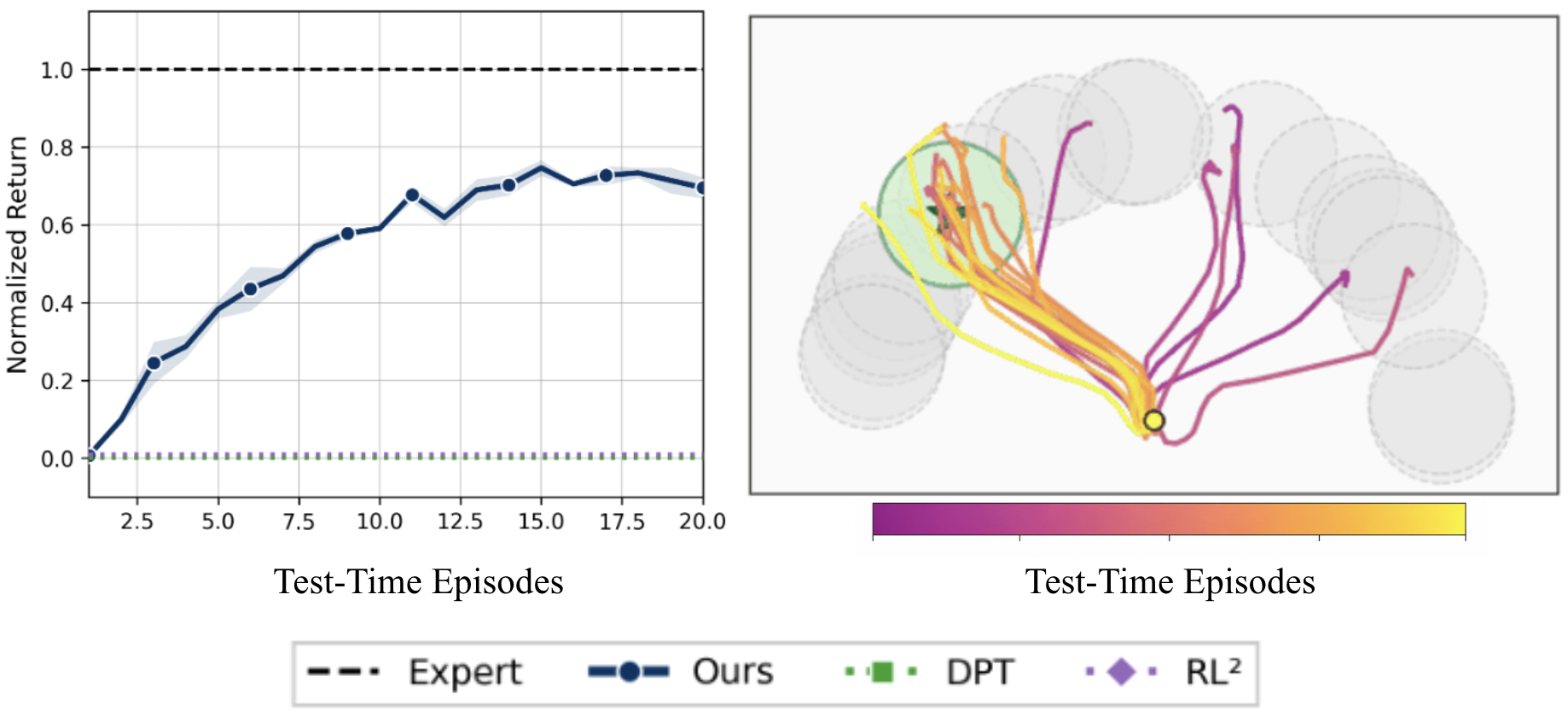}%
            \vfil
        }
        \caption{Ant Navigation. (Right)~\MethodS{} performs posterior sampling over goals while controlling the 8-DoF robot and achieves high success (Left), while RL$^2$ and DPT (success rate $\sim0$) fail to learn control and exploration for higher-dimensional action spaces. \label{fig:ant_behavior}}
    \end{minipage}
\end{figure}

We also see goal-directed, memory-based exploration (Appendix~\ref{sec:mem}) in Habitat and navigation tasks: the policy explores the environment, remembers failed locations, and avoids redundant search. See Appendix~\ref{app:qualitative} and the website for videos. In Figure~\ref{fig:ant_behavior}, we observe in the Ant Navigation task that the learned exploration resembles strategies found in prior meta-learning methods~\citep{zintgraf2021varibad}. The ant efficiently explores different goals and eliminates regions where it does not receive a high reward (see trajectories in Figure~\ref{fig:ant_behavior}). This enables the agent to adaptively find the right goal and exploit it. Although RL can learn exploration, its sample inefficiency makes it impractical. DPT performs poorly in continuous control, as it is nearly impossible to construct offline datasets via random exploration that cover test-time histories.

\subsection{Does \MethodS{} learn exploration policies that can be deployed on real robot hardware?}

In Figure~\ref{fig:real_combined}, we evaluate \MethodS{} on a real robot manipulation task that requires exploration under partial observability. We first train clairvoyant experts in Isaac Lab using large-scale reinforcement learning with domain randomization~\citep{yin2026emergent,zhang2026balanced}. These experts
use privileged state information to pick up a cube. At deployment, the real robot observes only proprioception (joint angles and gripper positions, no touch). The cameras are off, so it cannot see the cube, and vision-based transfer is not applicable. We train student policies to explore the world and infer this hidden information by trying different grasp locations and using gripper jaw position.

\begin{wrapfigure}{r}{0.55\columnwidth}
    \centering
    \includegraphics[width=\linewidth]{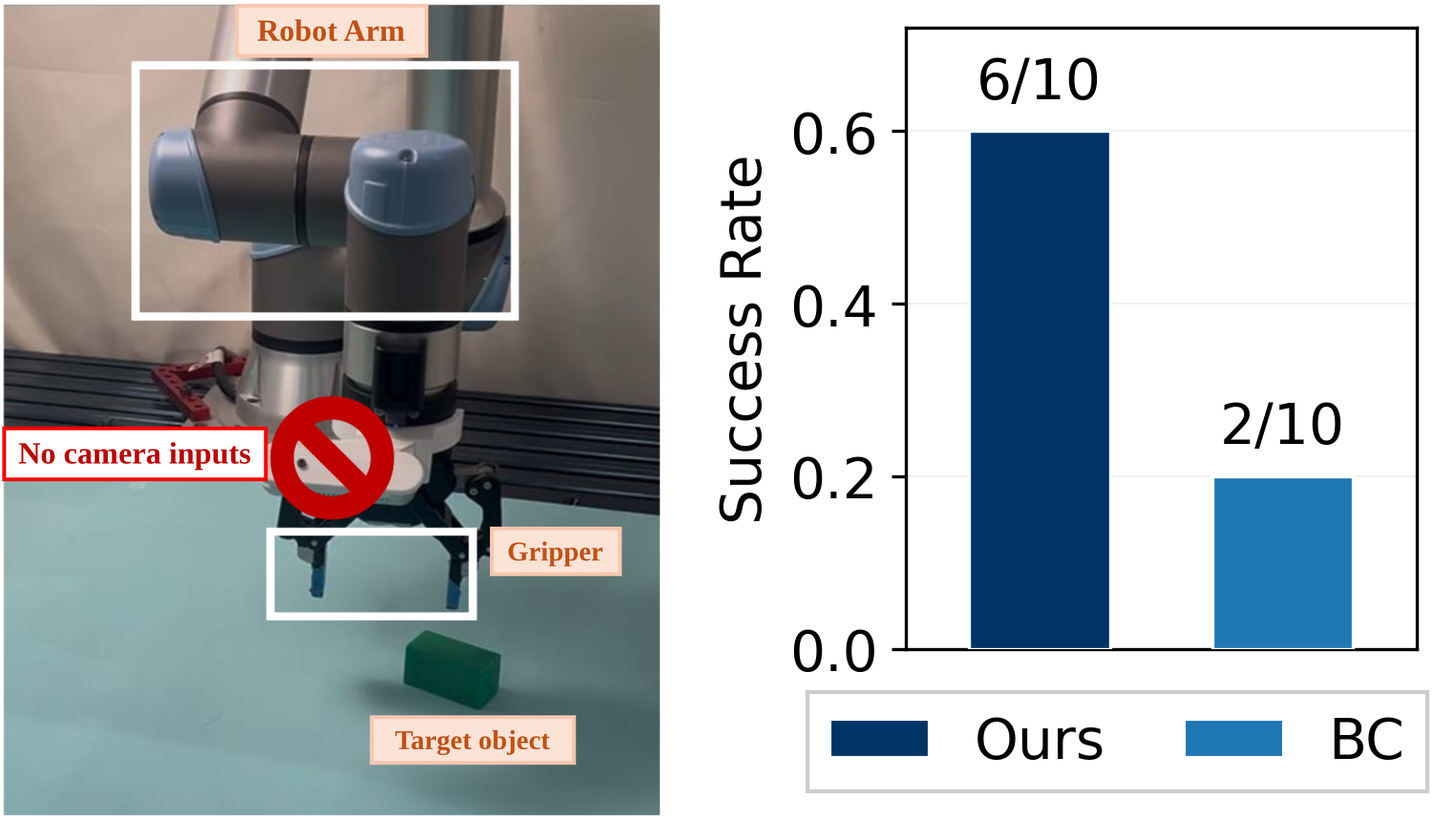}
    \vspace{-5pt}
    \caption{(Left) Real robot setup, where it must rely solely on proprioception (arm joint and gripper positions) to infer the object's location and pick it up. (Right) Comparing the performance of sim-to-real distillation via naive BC and our method on the real-robot task. \MethodS{} learns an effective exploration policy that achieves a higher success rate across 10 trials with randomized cube positions.}
    \label{fig:real_combined}
\end{wrapfigure}

The standard transfer baseline is to distill the simulation expert into the student policy via behavior cloning. This student policy imitates expert trajectories, but does not learn how to recover from failed attempts. In contrast, \MethodS{} trains the policy to perform in-context exploration, so the robot learns to adapt its next grasp attempt based on the history of previous attempts.
As shown in Figure~\ref{fig:real_combined} (right), \MethodS{} achieves a higher success rate than behavior cloning. The qualitative rollouts in Figure~\ref{fig:real_demo} show this behavior: after a failed attempt, the baseline goes out of distribution, while \MethodS{} continues searching over potential cube locations using its in-context abilities, completes the grasp, and successfully picks up the cube without cameras. We further demonstrate qualitatively that as the policy does not use vision, it can successfully solve the task in situations where the partial observability is natural, e.g., rummaging inside a box or under visual distractors (see robot videos on our website\footnote{Project
webpage: \href{https://weirdlabuw.github.io/asteroid/}{weirdlabuw.github.io/asteroid}}).

\subsection{Do frontier models perform in-context exploration?}

\begin{wrapfigure}{l}{0.55\linewidth}
    \centering
    \includegraphics[width=\linewidth]{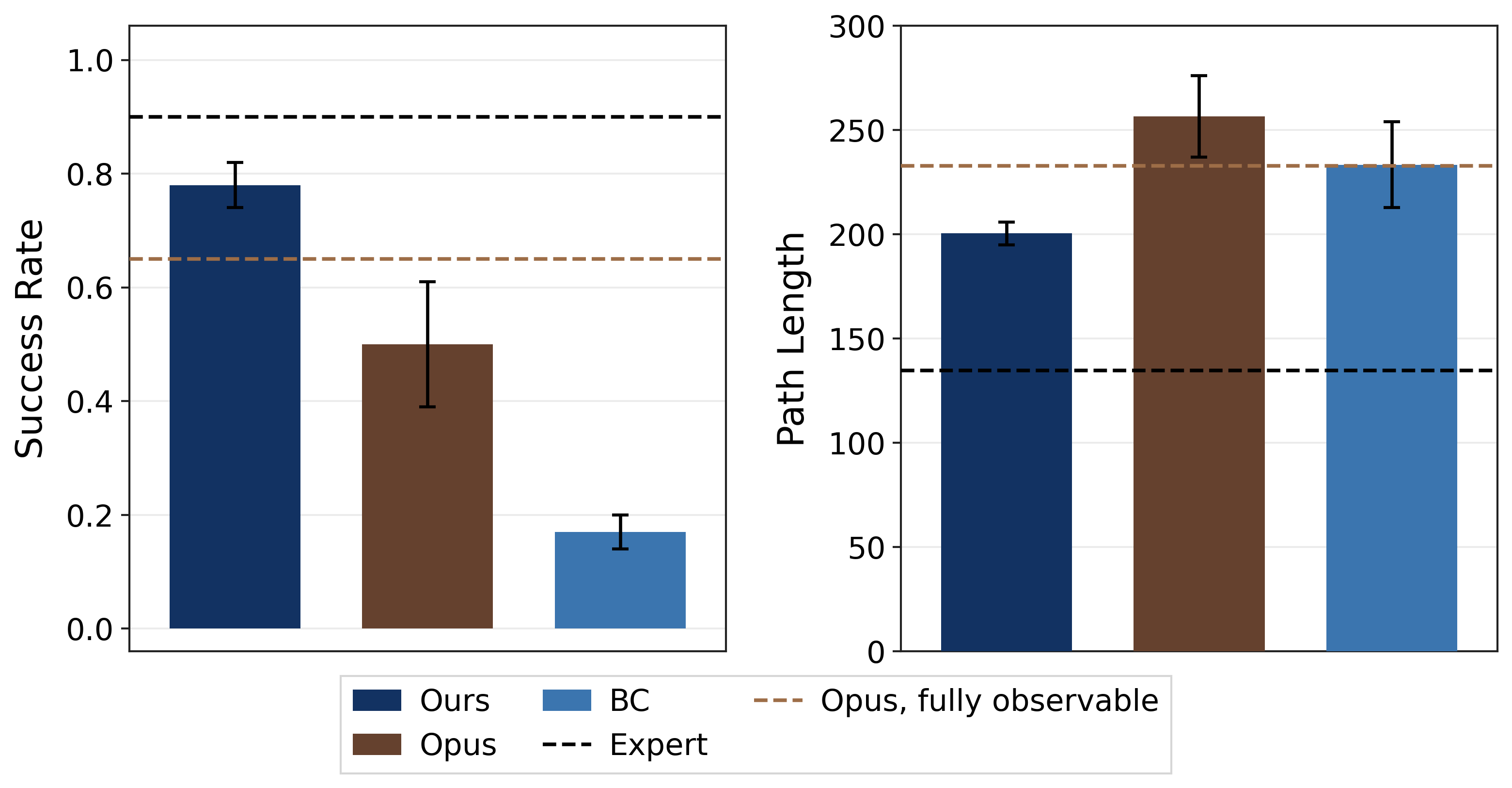}
    \caption{ In the Peg-Insertion task, \MethodS{} outperforms a frontier model (Claude Opus 5.5) at exploring the table to pick up the peg and insert it, while also being faster.}
        \label{fig:opus}
\end{wrapfigure}

Given recent progress in highly capable foundation models~\cite{opus55}, a natural question is: \emph{if we use frontier models for exploration in partially observed environments, how well do they perform?} To investigate this, we conduct an experiment on the Peg-Insertion task where the robot has to use its wrist camera to explore the table, locate and pick up the peg, and insert it into the peg hole.

Following work on using foundation models to write code for controlling robots, \textit{e.g.}, ``code-as-policies''~\cite{codeaspolicies2022,chen2026wordslongwaylanguage}, we first let a Claude Opus 5.5 coding agent (inside Claude Code) build a harness around the environment. Like \MethodS{}, it has access to the full Isaac Lab environment and the privileged expert policy, which it uses to collect 3 demonstrations, train a perception model, and tune the prompt. We then use Claude Opus 5.5 as the policy: during each episode, we show it the camera images and robot state, and it writes a small program that controls the robot until the current sub-task (\textit{e.g.}, grasping the peg) is done, then looks again and writes the next subtask policy. Notably, directly using the model to output actions for this led to almost zero successes, as the model failed to precisely insert the peg.  
Under full observability, the model achieves a $\sim65\%$ success rate, below the expert policy ($\sim90\%$), while also being $\sim2\times$ less efficient (see Figure~\ref{fig:opus}). To adapt the policy to partial observability, we prompt the agent to tune the code policy to a wrist-camera-only setting (preserving the base controller it has generated). In the Wrist-Camera Peg-Insertion task, we see that our learned policy outperforms the frontier model ($\sim78\%$ vs.\ $\sim50\%$ success rate). This implies that while frontier models are capable of completing tasks under full observability, the ability to explore via context is still missing from the repertoire of skills they have.

\subsection{Is learning from on-policy contexts necessary?}

\begin{wrapfigure}{r}{0.45\linewidth}
    \centering
    \vspace{-0.8cm}
    \includegraphics[width=0.75\linewidth]{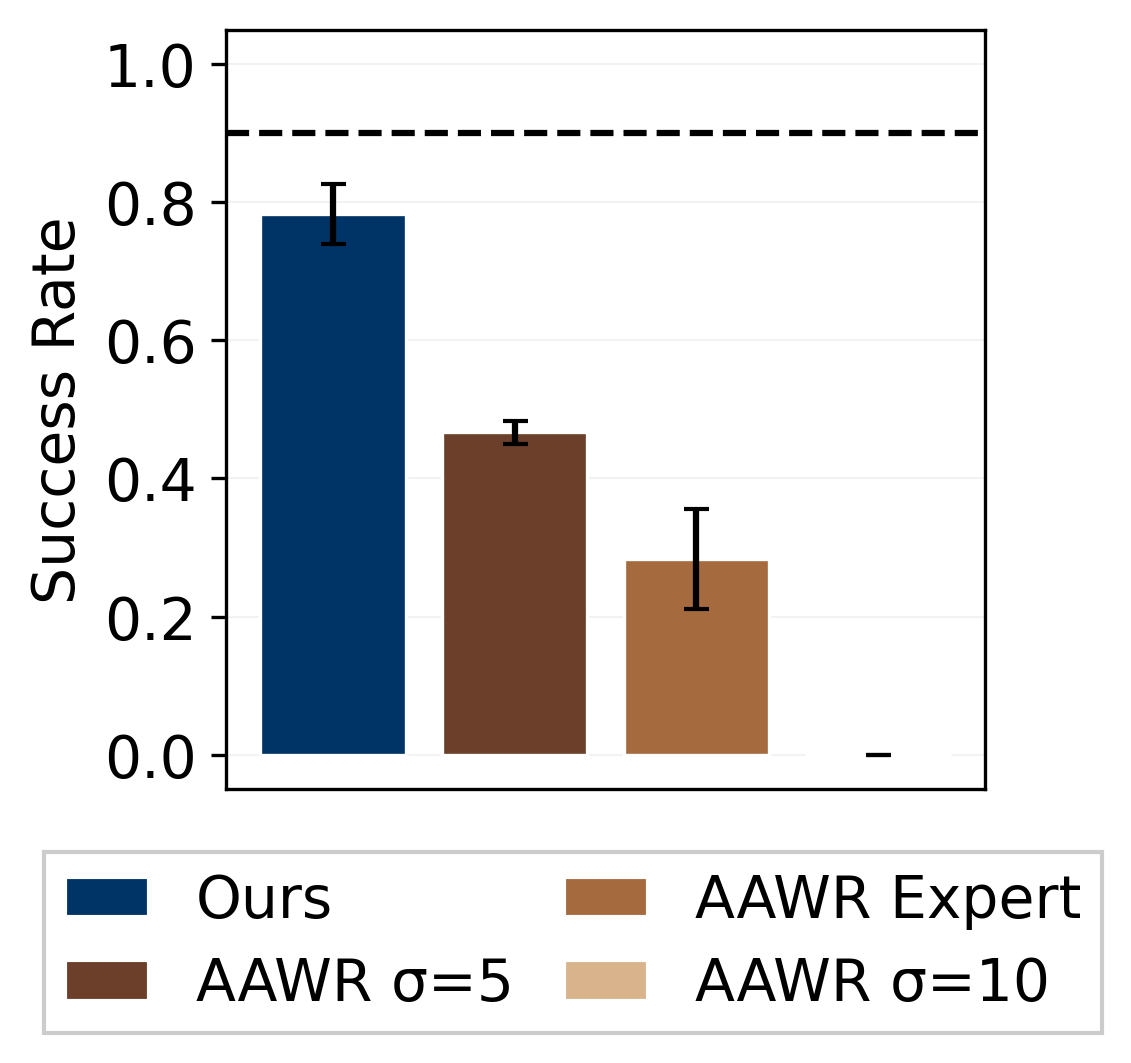}
    \caption{\MethodS{} generates data online, while AAWR is sensitive to the quality of offline datasets.}
    \label{fig:aawr_ablation}
\end{wrapfigure}

Across multiple domains, DPT performs poorly because offline random contexts fail to capture the deployment histories. In Figure~\ref{fig:aawr_ablation}, we compare \MethodS{} against AAWR~\cite{hu2025realworld} trained with offline datasets of varying diversity in the peg-insertion task, to highlight the importance of on-policy context collection.  AAWR~\citep{hu2025realworld}, which learns from an offline dataset, is highly sensitive to the dataset quality and diversity. \MethodS{} avoids this offline dependence by generating on-policy contexts from the student. This focuses supervision on deployment-relevant contexts, and is robust to imperfect expert labels (see Appendix~\ref{app:noisy_ablation}).

\begin{figure*}
    \centering
    \includegraphics[width=\linewidth]{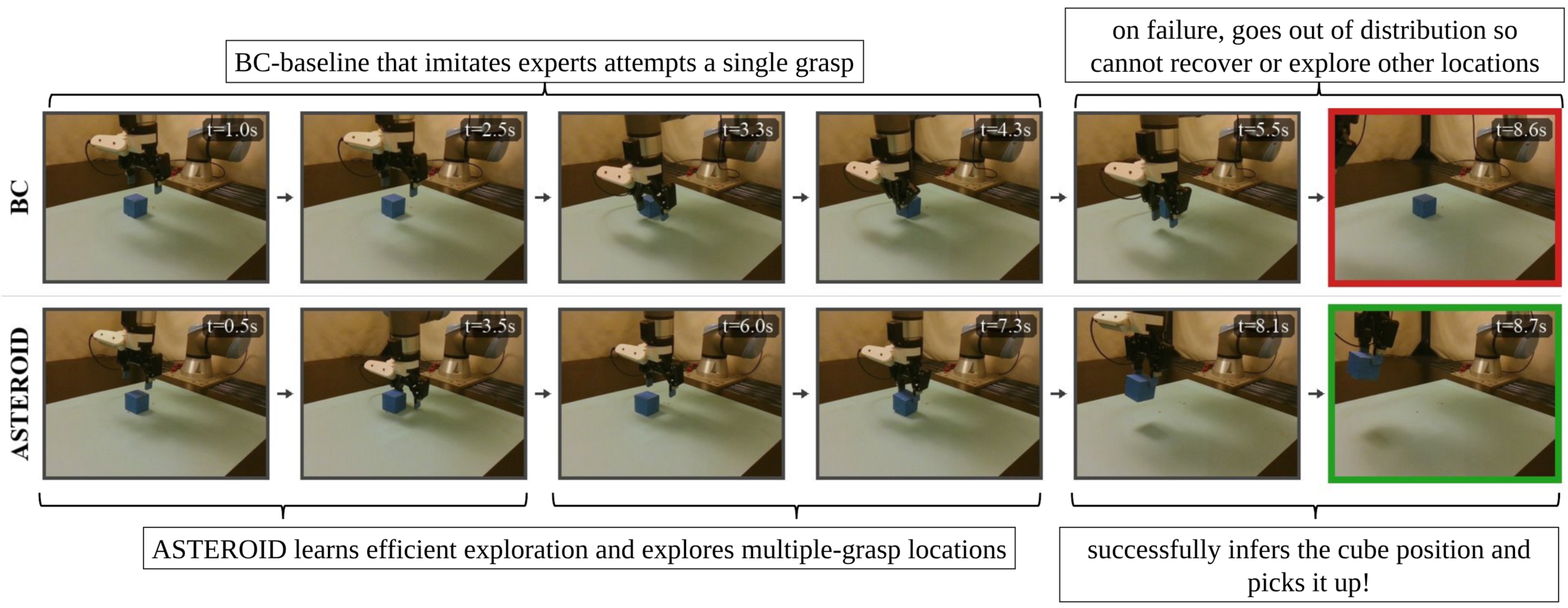}
    \caption{\footnotesize We evaluate sim-to-real transfer of policies learned via \MethodS{} and naive behavior cloning on a real-robot task, where the robot cannot use the camera to localize the object and pick it up. Instead, it has to use the proprioception of the arm and the gripper and explore different locations by attempting grasps at different positions. We compare the behavior learned with BC vs \MethodS{}. The baseline (first row) is trained to imitate expert data only and goes out of distribution. Meanwhile, \MethodS{} (second row) learns a policy that uses its in-context exploration to explore multiple grasp locations and pick up the object.}
    \label{fig:real_demo}
    \vspace{-0.5cm}
\end{figure*}

\section{Discussion}

In this work, we propose \MethodS, an iterative \emph{asymmetric} distillation framework for learning in-context exploration. The key idea is to train a history-conditioned policy using supervised targets from a privileged expert, while collecting the histories \emph{on-policy} from the student. This directly addresses two bottlenecks for exploration: (i) privileged teachers do not demonstrate information-gathering behavior, and (ii) offline or randomly collected histories may not match the distribution the student will encounter at deployment. While the need for history-conditioned exploration policies is well established, we show that imitation of privileged teachers with on-policy contexts, without RL or latent regression, is sufficient to learn exploration scalably across diverse simulated and real-world tasks. We believe that as robots and AI agents are deployed in the wild, acting as web assistants or robot butlers, in-context exploration learned via simple and scalable online training is key to enabling truly autonomous behavior.

 A limitation is that the approach relies on a strong privileged expert, though it is agnostic to how the expert is generated. An interesting direction is to remove this reliance and train experts from offline demonstrations. We note that posterior sampling exploration is not Bayes-optimal, i.e., it trades off some exploration efficiency for tractability and strong regret guarantees. This motivates future extensions that combine our approach with optimistic exploration. Another limitation is that this exploration relies on inferring hidden states from the overall expert support. When the hidden information can only be gathered outside this support, the student cannot reduce its uncertainty. A promising future direction is to learn a residual policy on top of this student to learn off-support information-gathering behavior. Additionally, the student policy is trained to imitate expert actions, but may produce unsafe exploration. Therefore, an important direction is to enforce safety via explicit constraints and safety filters. 

\section*{Acknowledgments}

We thank the members of the WEIRD Lab for their helpful feedback and
insightful discussions. We thank the authors of OmniReset for guiding us in setting up the simulation experts and sim-to-real transfer of robot policies. The authors acknowledge the UW Hyak Team for providing HPC resources that have contributed to the research results reported in this paper. This work was partially funded by NSF Award 2535892.

\bibliography{references}

\begin{thebibliography}{79}
\providecommand{\natexlab}[1]{#1}
\providecommand{\url}[1]{\texttt{#1}}
\expandafter\ifx\csname urlstyle\endcsname\relax
  \providecommand{\doi}[1]{doi: #1}\else
  \providecommand{\doi}{doi: \begingroup \urlstyle{rm}\Url}\fi

\bibitem[Agarwal et~al.(2020)Agarwal, Kakade, Krishnamurthy, and
  Sun]{agarwal2020flambestructuralcomplexityrepresentation}
A.~Agarwal, S.~Kakade, A.~Krishnamurthy, and W.~Sun.
\newblock Flambe: Structural complexity and representation learning of low rank
  mdps, 2020.
\newblock URL \url{https://arxiv.org/abs/2006.10814}.

\bibitem[{Anthropic}(2026)]{opus55}
{Anthropic}.
\newblock Claude opus 5.5.
\newblock Technical report, Anthropic, September 2026.

\bibitem[Asmuth et~al.(2009)Asmuth, Li, Littman, Nouri, and
  Wingate]{asmuth2012boss}
J.~Asmuth, L.~Li, M.~L. Littman, A.~Nouri, and D.~Wingate.
\newblock A bayesian sampling approach to exploration in reinforcement
  learning.
\newblock UAI '09, page 19–26, Arlington, Virginia, USA, 2009. AUAI Press.
\newblock ISBN 9780974903958.

\bibitem[Bansal et~al.(2018)Bansal, Krizhevsky, and
  Ogale]{bansal2018chauffeurnet}
M.~Bansal, A.~Krizhevsky, and A.~Ogale.
\newblock Chauffeurnet: Learning to drive by imitating the best and
  synthesizing the worst.
\newblock \emph{arXiv preprint arXiv:1812.03079}, 2018.

\bibitem[Bauer et~al.(2023)Bauer, Baumli, Behbahani, Bhoopchand,
  Bradley-Schmieg, Chang, Clay, Collister, Dasagi, Gonzalez, Gregor, Hughes,
  Kashem, Loks-Thompson, Openshaw, Parker-Holder, Pathak, Perez-Nieves,
  Rakicevic, Rockt\"{a}schel, Schroecker, Singh, Sygnowski, Tuyls, York,
  Zacherl, and Zhang]{bauer2023human}
J.~Bauer, K.~Baumli, F.~Behbahani, A.~Bhoopchand, N.~Bradley-Schmieg, M.~Chang,
  N.~Clay, A.~Collister, V.~Dasagi, L.~Gonzalez, K.~Gregor, E.~Hughes,
  S.~Kashem, M.~Loks-Thompson, H.~Openshaw, J.~Parker-Holder, S.~Pathak,
  N.~Perez-Nieves, N.~Rakicevic, T.~Rockt\"{a}schel, Y.~Schroecker, S.~Singh,
  J.~Sygnowski, K.~Tuyls, S.~York, A.~Zacherl, and L.~Zhang.
\newblock Human-timescale adaptation in an open-ended task space.
\newblock In \emph{Proceedings of the 40th International Conference on Machine
  Learning}, ICML'23. JMLR.org, 2023.

\bibitem[Bellman(1966)]{Bellman1966}
R.~Bellman.
\newblock Dynamic programming.
\newblock \emph{Science}, 153\penalty0 (3731):\penalty0 34--37, 1966.
\newblock \doi{10.1126/science.153.3731.34}.

\bibitem[Bhardwaj et~al.(2017)Bhardwaj, Choudhury, and
  Scherer]{bhardwaj2017learning}
M.~Bhardwaj, S.~Choudhury, and S.~Scherer.
\newblock Learning heuristic search via imitation.
\newblock In \emph{Conference on Robot Learning}, pages 271--280. PMLR, 2017.

\bibitem[Bronars et~al.(2026)Bronars, Park, Kedia, and
  Agrawal]{bronars2026blindpreciseassemblyproprioceptive}
A.~Bronars, Y.~Park, K.~Kedia, and P.~Agrawal.
\newblock Blind precise assembly: Proprioceptive sim-to-real peg-in-hole for
  diverse parts, 2026.
\newblock URL \url{https://arxiv.org/abs/2610.23849}.

\bibitem[Brown et~al.(2020)Brown, Mann, Ryder, Subbiah, Kaplan, Dhariwal,
  Neelakantan, Shyam, Sastry, Askell, et~al.]{brown2020language}
T.~Brown, B.~Mann, N.~Ryder, M.~Subbiah, J.~D. Kaplan, P.~Dhariwal,
  A.~Neelakantan, P.~Shyam, G.~Sastry, A.~Askell, et~al.
\newblock Language models are few-shot learners.
\newblock \emph{Advances in neural information processing systems},
  33:\penalty0 1877--1901, 2020.

\bibitem[Chapelle and Li(2011)]{chapelle2011thompsoneval}
O.~Chapelle and L.~Li.
\newblock An empirical evaluation of thompson sampling.
\newblock In J.~Shawe-Taylor, R.~Zemel, P.~Bartlett, F.~Pereira, and
  K.~Weinberger, editors, \emph{Advances in Neural Information Processing
  Systems}, volume~24. Curran Associates, Inc., 2011.
\newblock URL
  \url{https://proceedings.neurips.cc/paper_files/paper/2011/file/e53a0a2978c28872a4505bdb51db06dc-Paper.pdf}.

\bibitem[Chaplot et~al.(2020{\natexlab{a}})Chaplot, Gandhi, Gupta, and
  Salakhutdinov]{chaplot2020object}
D.~S. Chaplot, D.~Gandhi, A.~Gupta, and R.~Salakhutdinov.
\newblock Object goal navigation using goal-oriented semantic exploration.
\newblock In \emph{In Neural Information Processing Systems (NeurIPS)},
  2020{\natexlab{a}}.

\bibitem[Chaplot et~al.(2020{\natexlab{b}})Chaplot, Salakhutdinov, Gupta, and
  Gupta]{Chaplot2020NeuralTS}
D.~S. Chaplot, R.~Salakhutdinov, A.~K. Gupta, and S.~Gupta.
\newblock Neural topological slam for visual navigation.
\newblock \emph{2020 IEEE/CVF Conference on Computer Vision and Pattern
  Recognition (CVPR)}, pages 12872--12881, 2020{\natexlab{b}}.
\newblock URL \url{https://api.semanticscholar.org/CorpusID:214754592}.

\bibitem[Chen et~al.(2020)Chen, Zhou, Koltun, and
  Kr{\"a}henb{\"u}hl]{chen2020learning}
D.~Chen, B.~Zhou, V.~Koltun, and P.~Kr{\"a}henb{\"u}hl.
\newblock Learning by cheating.
\newblock In \emph{Conference on robot learning}, pages 66--75. PMLR, 2020.

\bibitem[Chen et~al.(2026)Chen, Jain, Goossen, Romig, Murray, Walker, and
  Cakmak]{chen2026wordslongwaylanguage}
D.~Chen, A.~R. Jain, E.~Goossen, E.~Romig, M.~Murray, N.~Walker, and M.~Cakmak.
\newblock A few words go a long way: Language guided robot policy synthesis,
  2026.
\newblock URL \url{https://arxiv.org/abs/2607.23784}.

\bibitem[Chen et~al.(2021)Chen, Lu, Rajeswaran, Lee, Grover, Laskin, Abbeel,
  Srinivas, and Mordatch]{chen2021decisiontransformer}
L.~Chen, K.~Lu, A.~Rajeswaran, K.~Lee, A.~Grover, M.~Laskin, P.~Abbeel,
  A.~Srinivas, and I.~Mordatch.
\newblock Decision transformer: Reinforcement learning via sequence modeling.
\newblock \emph{arXiv preprint arXiv:2106.01345}, 2021.

\bibitem[Chevalier{-}Boisvert et~al.(2023)Chevalier{-}Boisvert, Dai, Towers,
  Perez{-}Vicente, Willems, Lahlou, Pal, Castro, and
  Terry]{MinigridMiniworld23}
M.~Chevalier{-}Boisvert, B.~Dai, M.~Towers, R.~Perez{-}Vicente, L.~Willems,
  S.~Lahlou, S.~Pal, P.~S. Castro, and J.~Terry.
\newblock Minigrid {\&} miniworld: Modular {\&} customizable reinforcement
  learning environments for goal-oriented tasks.
\newblock In \emph{Advances in Neural Information Processing Systems 36, New
  Orleans, LA, USA}, December 2023.

\bibitem[Choudhury et~al.(2017)Choudhury, Kapoor, Ranade, Scherer, and
  Dey]{choudhury2017adaptive}
S.~Choudhury, A.~Kapoor, G.~Ranade, S.~Scherer, and D.~Dey.
\newblock Adaptive information gathering via imitation learning.
\newblock \emph{arXiv preprint arXiv:1705.07834}, 2017.

\bibitem[Cobbe et~al.(2019)Cobbe, Hesse, Hilton, and
  Schulman]{cobbe2019procgen}
K.~Cobbe, C.~Hesse, J.~Hilton, and J.~Schulman.
\newblock Leveraging procedural generation to benchmark reinforcement learning.
\newblock \emph{arXiv preprint arXiv:1912.01588}, 2019.

\bibitem[Daum{\'e}~Iii et~al.(2009)Daum{\'e}~Iii, Langford, and
  Marcu]{daume2009search}
H.~Daum{\'e}~Iii, J.~Langford, and D.~Marcu.
\newblock Search-based structured prediction.
\newblock \emph{Machine learning}, 75\penalty0 (3):\penalty0 297--325, 2009.

\bibitem[De~Haan et~al.(2019)De~Haan, Jayaraman, and Levine]{de2019causal}
P.~De~Haan, D.~Jayaraman, and S.~Levine.
\newblock Causal confusion in imitation learning.
\newblock \emph{Advances in neural information processing systems}, 32, 2019.

\bibitem[Duan et~al.(2016)Duan, Schulman, Chen, Bartlett, Sutskever, and
  Abbeel]{duan2016rl}
Y.~Duan, J.~Schulman, X.~Chen, P.~L. Bartlett, I.~Sutskever, and P.~Abbeel.
\newblock Rl$^2$: Fast reinforcement learning via slow reinforcement learning.
\newblock \emph{arXiv preprint arXiv:1611.02779}, 2016.

\bibitem[Ehsani et~al.(2023)Ehsani, Gupta, Hendrix, Salvador, Weihs, Zeng,
  Singh, Kim, Han, Herrasti, Krishna, Schwenk, VanderBilt, and
  Kembhavi]{spoc2023ehsani}
K.~Ehsani, T.~Gupta, R.~Hendrix, J.~Salvador, L.~Weihs, K.-H. Zeng, K.~P.
  Singh, Y.~Kim, W.~Han, A.~Herrasti, R.~Krishna, D.~Schwenk, E.~VanderBilt,
  and A.~Kembhavi.
\newblock Imitating shortest paths in simulation enables effective navigation
  and manipulation in the real world.
\newblock \emph{arXiv}, 2023.

\bibitem[Finn et~al.(2017)Finn, Abbeel, and Levine]{finn2017maml}
C.~Finn, P.~Abbeel, and S.~Levine.
\newblock Model-agnostic meta-learning for fast adaptation of deep networks.
\newblock In \emph{Proceedings of the 34th International Conference on Machine
  Learning - Volume 70}, ICML'17, page 1126–1135. JMLR.org, 2017.

\bibitem[Fu et~al.(2024)Fu, Huang, Datta, Chen, Panitch, Liu, Li, and
  Goldberg]{fu2024icrt}
L.~Fu, H.~Huang, G.~Datta, L.~Y. Chen, W.~C.-H. Panitch, F.~Liu, H.~Li, and
  K.~Goldberg.
\newblock In-context imitation learning via next-token prediction.
\newblock \emph{arXiv preprint arXiv:2408.15980}, 2024.

\bibitem[Garg et~al.(2022)Garg, Tsipras, Liang, and Valiant]{garg2022can}
S.~Garg, D.~Tsipras, P.~S. Liang, and G.~Valiant.
\newblock What can transformers learn in-context? a case study of simple
  function classes.
\newblock \emph{Advances in neural information processing systems},
  35:\penalty0 30583--30598, 2022.

\bibitem[Gu et~al.(2024)Gu, Dong, Wei, and Huang]{gu2024minillm}
Y.~Gu, L.~Dong, F.~Wei, and M.~Huang.
\newblock Minillm: Knowledge distillation of large language models.
\newblock In \emph{The twelfth international conference on learning
  representations}, 2024.

\bibitem[Haarnoja et~al.(2018)Haarnoja, Zhou, Hartikainen, Tucker, Ha, Tan,
  Kumar, Zhu, Gupta, Abbeel, and Levine]{haarnoja2018sacapps}
T.~Haarnoja, A.~Zhou, K.~Hartikainen, G.~Tucker, S.~Ha, J.~Tan, V.~Kumar,
  H.~Zhu, A.~Gupta, P.~Abbeel, and S.~Levine.
\newblock Soft actor-critic algorithms and applications.
\newblock Technical report, 2018.

\bibitem[Hausknecht and Stone(2015)]{hausknecht2015deep}
M.~J. Hausknecht and P.~Stone.
\newblock Deep recurrent q-learning for partially observable mdps.
\newblock In \emph{AAAI fall symposia}, volume~45, page 141, 2015.

\bibitem[Higuera et~al.(2023)Higuera, Ortiz, Qi, Pineda, Boots, and
  Mukadam]{higuera2023perceiving}
C.~Higuera, J.~Ortiz, H.~Qi, L.~Pineda, B.~Boots, and M.~Mukadam.
\newblock Perceiving extrinsic contacts from touch improves learning insertion
  policies, 2023.

\bibitem[Hu et~al.(2025)Hu, Wang, Yuan, Luo, Li, Lambrechts, Rybkin, and
  Jayaraman]{hu2025realworld}
E.~S. Hu, J.~Wang, X.~Yuan, F.~Luo, M.~Li, G.~Lambrechts, O.~Rybkin, and
  D.~Jayaraman.
\newblock Real-world reinforcement learning of active perception behaviors.
\newblock In \emph{The Thirty-ninth Annual Conference on Neural Information
  Processing Systems}, 2025.
\newblock URL \url{https://openreview.net/forum?id=RkdTtznSAL}.

\bibitem[Igl et~al.(2018)Igl, Zintgraf, Le, Wood, and Whiteson]{Igl2018DeepVR}
M.~Igl, L.~M. Zintgraf, T.~A. Le, F.~Wood, and S.~Whiteson.
\newblock Deep variational reinforcement learning for pomdps.
\newblock In \emph{International Conference on Machine Learning}, 2018.
\newblock URL \url{https://api.semanticscholar.org/CorpusID:46955236}.

\bibitem[Intelligence et~al.(2025)Intelligence, Black, Brown, Darpinian,
  Dhabalia, Driess, Esmail, Equi, Finn, Fusai, et~al.]{intelligence2025pi_}
P.~Intelligence, K.~Black, N.~Brown, J.~Darpinian, K.~Dhabalia, D.~Driess,
  A.~Esmail, M.~Equi, C.~Finn, N.~Fusai, et~al.
\newblock $\pi_{0.5}$: a vision-language-action model with open-world
  generalization.
\newblock \emph{arXiv preprint arXiv:2504.16054}, 2025.

\bibitem[Jain et~al.(2021)Jain, Liu, Lazebnik, Kembhavi, Weihs, and
  Schwing]{jain2021gridtopix}
U.~Jain, I.-J. Liu, S.~Lazebnik, A.~Kembhavi, L.~Weihs, and A.~G. Schwing.
\newblock Gridtopix: Training embodied agents with minimal supervision.
\newblock In \emph{Proceedings of the IEEE/CVF International Conference on
  Computer Vision}, pages 15141--15151, 2021.

\bibitem[Kaelbling et~al.(1998)Kaelbling, Littman, and
  Cassandra]{kaelbling1998planning}
L.~P. Kaelbling, M.~L. Littman, and A.~R. Cassandra.
\newblock Planning and acting in partially observable stochastic domains.
\newblock \emph{Artificial Intelligence}, 101\penalty0 (1):\penalty0 99--134,
  1998.
\newblock ISSN 0004-3702.
\newblock \doi{https://doi.org/10.1016/S0004-3702(98)00023-X}.
\newblock URL
  \url{https://www.sciencedirect.com/science/article/pii/S000437029800023X}.

\bibitem[Karkus et~al.(2017)Karkus, Hsu, and Lee]{karkus2017qmdp}
P.~Karkus, D.~Hsu, and W.~S. Lee.
\newblock Qmdp-net: Deep learning for planning under partial observability.
\newblock \emph{Advances in neural information processing systems}, 30, 2017.

\bibitem[Kelly et~al.(2019)Kelly, Sidrane, Driggs-Campbell, and
  Kochenderfer]{kelly2019hg}
M.~Kelly, C.~Sidrane, K.~Driggs-Campbell, and M.~J. Kochenderfer.
\newblock Hg-dagger: Interactive imitation learning with human experts.
\newblock In \emph{2019 International Conference on Robotics and Automation
  (ICRA)}, pages 8077--8083. IEEE, 2019.

\bibitem[Kim et~al.(2025)Kim, Chin, Vasudev, and Choudhury]{kim2025distilling}
Y.~Kim, N.~Chin, A.~Vasudev, and S.~Choudhury.
\newblock Distilling realizable students from unrealizable teachers.
\newblock \emph{arXiv preprint arXiv:2505.09546}, 2025.

\bibitem[Kumar et~al.(2021)Kumar, Fu, Pathak, and Malik]{kumar2021rma}
A.~Kumar, Z.~Fu, D.~Pathak, and J.~Malik.
\newblock Rma: Rapid motor adaptation for legged robots.
\newblock \emph{arXiv preprint arXiv:2107.04034}, 2021.

\bibitem[Laskin et~al.(2022)Laskin, Wang, Oh, Parisotto, Spencer, Steigerwald,
  Strouse, Hansen, Filos, Brooks, Gazeau, Sahni, Singh, and
  Mnih]{Laskin2023InContextRL}
M.~Laskin, L.~Wang, J.~Oh, E.~Parisotto, S.~Spencer, R.~Steigerwald,
  D.~Strouse, S.~Hansen, A.~Filos, E.~Brooks, M.~Gazeau, H.~Sahni, S.~Singh,
  and V.~Mnih.
\newblock In-context reinforcement learning with algorithm distillation, 2022.
\newblock URL \url{https://arxiv.org/abs/2210.14215}.

\bibitem[Lee et~al.(2021)Lee, Hou, Choudhury, and Srinivasa]{lee2021bayesian}
G.~Lee, B.~Hou, S.~Choudhury, and S.~S. Srinivasa.
\newblock Bayesian residual policy optimization:: Scalable bayesian
  reinforcement learning with clairvoyant experts.
\newblock In \emph{2021 IEEE/RSJ International Conference on Intelligent Robots
  and Systems (IROS)}, pages 5611--5618. IEEE, 2021.

\bibitem[Lee et~al.(2020)Lee, Hwangbo, Wellhausen, Koltun, and
  Hutter]{lee2020learning}
J.~Lee, J.~Hwangbo, L.~Wellhausen, V.~Koltun, and M.~Hutter.
\newblock Learning quadrupedal locomotion over challenging terrain.
\newblock \emph{Science robotics}, 5\penalty0 (47):\penalty0 eabc5986, 2020.

\bibitem[Lee et~al.(2023)Lee, Xie, Pacchiano, Chandak, Finn, Nachum, and
  Brunskill]{lee2023supervised}
J.~N. Lee, A.~Xie, A.~Pacchiano, Y.~Chandak, C.~Finn, O.~Nachum, and
  E.~Brunskill.
\newblock Supervised pretraining can learn in-context reinforcement learning.
\newblock \emph{arXiv preprint arXiv:2306.14892}, 2023.

\bibitem[Liang et~al.(2022)Liang, Huang, Xia, Xu, Hausman, Ichter, Florence,
  and Zeng]{codeaspolicies2022}
J.~Liang, W.~Huang, F.~Xia, P.~Xu, K.~Hausman, B.~Ichter, P.~Florence, and
  A.~Zeng.
\newblock Code as policies: Language model programs for embodied control.
\newblock In \emph{arXiv preprint arXiv:2209.07753}, 2022.

\bibitem[Liu et~al.(2021)Liu, Raghunathan, Liang, and Finn]{pmlrv139liu21s}
E.~Z. Liu, A.~Raghunathan, P.~Liang, and C.~Finn.
\newblock Decoupling exploration and exploitation for meta-reinforcement
  learning without sacrifices.
\newblock In M.~Meila and T.~Zhang, editors, \emph{Proceedings of the 38th
  International Conference on Machine Learning}, volume 139 of
  \emph{Proceedings of Machine Learning Research}, pages 6925--6935. PMLR,
  18--24 Jul 2021.
\newblock URL \url{https://proceedings.mlr.press/v139/liu21s.html}.

\bibitem[Liu and Abbeel(2023)]{liu2023emergent}
H.~Liu and P.~Abbeel.
\newblock Emergent agentic transformer from chain of hindsight experience.
\newblock In \emph{International Conference on Machine Learning}, pages
  21362--21374. PMLR, 2023.

\bibitem[Memmel et~al.(2024)Memmel, Wagenmaker, Zhu, Yin, Fox, and
  Gupta]{memmel2024asid}
M.~Memmel, A.~Wagenmaker, C.~Zhu, P.~Yin, D.~Fox, and A.~Gupta.
\newblock Asid: Active exploration for system identification in robotic
  manipulation.
\newblock \emph{arXiv preprint arXiv:2404.12308}, 2024.

\bibitem[Min et~al.(2022)Min, Lyu, Holtzman, Artetxe, Lewis, Hajishirzi, and
  Zettlemoyer]{min2022rethinking}
S.~Min, X.~Lyu, A.~Holtzman, M.~Artetxe, M.~Lewis, H.~Hajishirzi, and
  L.~Zettlemoyer.
\newblock Rethinking the role of demonstrations: What makes in-context learning
  work?
\newblock In \emph{Proceedings of the 2022 conference on empirical methods in
  natural language processing}, pages 11048--11064, 2022.

\bibitem[Mittal et~al.(2023)Mittal, Yu, Yu, Liu, Rudin, Hoeller, Yuan, Singh,
  Guo, Mazhar, et~al.]{mittal2023orbit}
M.~Mittal, C.~Yu, Q.~Yu, J.~Liu, N.~Rudin, D.~Hoeller, J.~L. Yuan, R.~Singh,
  Y.~Guo, H.~Mazhar, et~al.
\newblock Orbit: A unified simulation framework for interactive robot learning
  environments.
\newblock \emph{IEEE Robotics and Automation Letters}, 8\penalty0 (6):\penalty0
  3740--3747, 2023.

\bibitem[Nguyen et~al.(2022)Nguyen, Baisero, Wang, Amato, and
  Platt]{nguyen2022leveraging}
H.~Nguyen, A.~Baisero, D.~Wang, C.~Amato, and R.~Platt.
\newblock Leveraging fully observable policies for learning under partial
  observability.
\newblock \emph{arXiv preprint arXiv:2211.01991}, 2022.

\bibitem[Ni et~al.(2021)Ni, Eysenbach, and Salakhutdinov]{ni2021recurrent}
T.~Ni, B.~Eysenbach, and R.~Salakhutdinov.
\newblock Recurrent model-free rl can be a strong baseline for many pomdps.
\newblock \emph{arXiv preprint arXiv:2110.05038}, 2021.

\bibitem[Nikulin et~al.(2025)Nikulin, Zisman, Zemtsov, and
  Kurenkov]{nikulin2025xland100blargescalemultitaskdataset}
A.~Nikulin, I.~Zisman, A.~Zemtsov, and V.~Kurenkov.
\newblock Xland-100b: A large-scale multi-task dataset for in-context
  reinforcement learning, 2025.
\newblock URL \url{https://arxiv.org/abs/2406.08973}.

\bibitem[Osa et~al.(2018)Osa, Pajarinen, Neumann, Bagnell, Abbeel, and
  Peters]{osa2018algorithmic}
T.~Osa, J.~Pajarinen, G.~Neumann, J.~A. Bagnell, P.~Abbeel, and J.~Peters.
\newblock An algorithmic perspective on imitation learning.
\newblock \emph{Foundations and Trends{\textregistered} in Robotics},
  7\penalty0 (1-2):\penalty0 1--179, 2018.

\bibitem[Osband et~al.(2013)Osband, Russo, and Van~Roy]{osband2013psrl}
I.~Osband, D.~Russo, and B.~Van~Roy.
\newblock {(More)} efficient reinforcement learning via posterior sampling.
\newblock In \emph{Proceedings of the 27th International Conference on Neural
  Information Processing Systems - Volume 2}, NIPS'13, page 3003–3011, Red
  Hook, NY, USA, 2013. Curran Associates Inc.

\bibitem[Peng et~al.(2018)Peng, Andrychowicz, Zaremba, and Abbeel]{peng2018sim}
X.~B. Peng, M.~Andrychowicz, W.~Zaremba, and P.~Abbeel.
\newblock Sim-to-real transfer of robotic control with dynamics randomization.
\newblock In \emph{2018 IEEE international conference on robotics and
  automation (ICRA)}, pages 3803--3810. IEEE, 2018.

\bibitem[Peng et~al.(2019)Peng, Kumar, Zhang, and Levine]{AWRPeng19}
X.~B. Peng, A.~Kumar, G.~Zhang, and S.~Levine.
\newblock Advantage-weighted regression: Simple and scalable off-policy
  reinforcement learning.
\newblock \emph{CoRR}, abs/1910.00177, 2019.
\newblock URL \url{https://arxiv.org/abs/1910.00177}.

\bibitem[Puig et~al.(2024)Puig, Undersander, Szot, Cote, Yang, Partsey, Desai,
  Clegg, Hlavac, Min, Vondrus, Gervet, Berges, Turner, Maksymets, Kira,
  Kalakrishnan, Malik, Chaplot, Jain, Batra, Rai, and
  Mottaghi]{puig2024habitat}
X.~Puig, E.~Undersander, A.~Szot, M.~D. Cote, T.-Y. Yang, R.~Partsey, R.~Desai,
  A.~Clegg, M.~Hlavac, S.~Y. Min, V.~Vondrus, T.~Gervet, V.-P. Berges, J.~M.
  Turner, O.~Maksymets, Z.~Kira, M.~Kalakrishnan, J.~Malik, D.~S. Chaplot,
  U.~Jain, D.~Batra, A.~Rai, and R.~Mottaghi.
\newblock Habitat 3.0: A co-habitat for humans, avatars, and robots.
\newblock In \emph{ICLR}. OpenReview.net, 2024.
\newblock URL
  \url{http://dblp.uni-trier.de/db/conf/iclr/iclr2024.html#PuigUSCYPDCHMVG24}.

\bibitem[Rakelly et~al.(2019)Rakelly, Zhou, Quillen, Finn, and
  Levine]{Rakelly2019EfficientOM}
K.~Rakelly, A.~Zhou, D.~Quillen, C.~Finn, and S.~Levine.
\newblock Efficient off-policy meta-reinforcement learning via probabilistic
  context variables.
\newblock In \emph{International Conference on Machine Learning}, 2019.
\newblock URL \url{https://api.semanticscholar.org/CorpusID:84187276}.

\bibitem[Ross and Bagnell(2014)]{ross2014reinforcement}
S.~Ross and J.~A. Bagnell.
\newblock Reinforcement and imitation learning via interactive no-regret
  learning.
\newblock \emph{arXiv preprint arXiv:1406.5979}, 2014.

\bibitem[Ross et~al.(2011)Ross, Gordon, and Bagnell]{ross2011reduction}
S.~Ross, G.~Gordon, and D.~Bagnell.
\newblock A reduction of imitation learning and structured prediction to
  no-regret online learning.
\newblock In \emph{Proceedings of the fourteenth international conference on
  artificial intelligence and statistics}, pages 627--635. JMLR Workshop and
  Conference Proceedings, 2011.

\bibitem[Schulman et~al.(2017)Schulman, Wolski, Dhariwal, Radford, and
  Klimov]{schulman2017ppo}
J.~Schulman, F.~Wolski, P.~Dhariwal, A.~Radford, and O.~Klimov.
\newblock Proximal policy optimization algorithms.
\newblock \emph{CoRR}, abs/1707.06347, 2017.
\newblock URL
  \url{http://dblp.uni-trier.de/db/journals/corr/corr1707.html#SchulmanWDRK17}.

\bibitem[Shenfeld et~al.(2023)Shenfeld, Hong, Tamar, and
  Agrawal]{shenfeld2023tgrl}
I.~Shenfeld, Z.-W. Hong, A.~Tamar, and P.~Agrawal.
\newblock Tgrl: An algorithm for teacher guided reinforcement learning.
\newblock In \emph{International Conference on Machine Learning}, pages
  31077--31093. PMLR, 2023.

\bibitem[Somani et~al.(2013)Somani, Ye, Hsu, and Lee]{somani2013despot}
A.~Somani, N.~Ye, D.~Hsu, and W.~S. Lee.
\newblock Despot: Online pomdp planning with regularization.
\newblock In C.~Burges, L.~Bottou, M.~Welling, Z.~Ghahramani, and
  K.~Weinberger, editors, \emph{Advances in Neural Information Processing
  Systems}, volume~26. Curran Associates, Inc., 2013.
\newblock URL
  \url{https://proceedings.neurips.cc/paper_files/paper/2013/file/c2aee86157b4a40b78132f1e71a9e6f1-Paper.pdf}.

\bibitem[Spencer et~al.(2021)Spencer, Choudhury, Venkatraman, Ziebart, and
  Bagnell]{spencer2021feedback}
J.~Spencer, S.~Choudhury, A.~Venkatraman, B.~Ziebart, and J.~A. Bagnell.
\newblock Feedback in imitation learning: The three regimes of covariate shift.
\newblock \emph{arXiv preprint arXiv:2102.02872}, 2021.

\bibitem[Sun et~al.(2017)Sun, Venkatraman, Gordon, Boots, and
  Bagnell]{sun2017deeply}
W.~Sun, A.~Venkatraman, G.~J. Gordon, B.~Boots, and J.~A. Bagnell.
\newblock Deeply aggrevated: Differentiable imitation learning for sequential
  prediction.
\newblock In \emph{International conference on machine learning}, pages
  3309--3318. PMLR, 2017.

\bibitem[Swamy et~al.(2022)Swamy, Choudhury, Bagnell, and
  Wu]{swamy2022sequence}
G.~Swamy, S.~Choudhury, J.~Bagnell, and S.~Z. Wu.
\newblock Sequence model imitation learning with unobserved contexts.
\newblock \emph{Advances in Neural Information Processing Systems},
  35:\penalty0 17665--17676, 2022.

\bibitem[Tennenholtz et~al.(2021)Tennenholtz, Hallak, Dalal, Mannor, Chechik,
  and Shalit]{tennenholtz2021covariate}
G.~Tennenholtz, A.~Hallak, G.~Dalal, S.~Mannor, G.~Chechik, and U.~Shalit.
\newblock On covariate shift of latent confounders in imitation and
  reinforcement learning.
\newblock \emph{arXiv preprint arXiv:2110.06539}, 2021.

\bibitem[Thompson(1933)]{thompson1933}
W.~R. Thompson.
\newblock On the likelihood that one unknown probability exceeds another in
  view of the evidence of two samples.
\newblock \emph{Biometrika}, 25:\penalty0 285--294, 1933.
\newblock \doi{10.1093/biomet/25.3-4.285}.
\newblock URL \url{https://doi.org/10.1093/biomet/25.3-4.285}.

\bibitem[Tobin et~al.(2017)Tobin, Fong, Ray, Schneider, Zaremba, and
  Abbeel]{tobin2017domain}
J.~Tobin, R.~Fong, A.~Ray, J.~Schneider, W.~Zaremba, and P.~Abbeel.
\newblock Domain randomization for transferring deep neural networks from
  simulation to the real world.
\newblock In \emph{2017 IEEE/RSJ international conference on intelligent robots
  and systems (IROS)}, pages 23--30. IEEE, 2017.

\bibitem[von Werra et~al.(2020)von Werra, Belkada, Tunstall, Beeching, Thrush,
  Lambert, Huang, Rasul, and Gallouédec]{vonwerra2020trl}
L.~von Werra, Y.~Belkada, L.~Tunstall, E.~Beeching, T.~Thrush, N.~Lambert,
  S.~Huang, K.~Rasul, and Q.~Gallouédec.
\newblock {TRL: Transformers Reinforcement Learning}, 2020.
\newblock URL \url{https://github.com/huggingface/trl}.

\bibitem[Walsman et~al.(2023)Walsman, Zhang, Choudhury, Fox, and
  Farhadi]{walsman2023impossibly}
A.~Walsman, M.~Zhang, S.~Choudhury, D.~Fox, and A.~Farhadi.
\newblock Impossibly good experts and how to follow them.
\newblock In \emph{The Eleventh International Conference on Learning
  Representations}, 2023.
\newblock URL \url{https://openreview.net/forum?id=sciA_xgYofB}.

\bibitem[Wan et~al.(2025)Wan, Wu, Abdulhai, Shani, and
  Jaques]{wan2025enhancing}
Y.~Wan, J.~Wu, M.~Abdulhai, L.~Shani, and N.~Jaques.
\newblock Enhancing personalized multi-turn dialogue with curiosity reward.
\newblock \emph{arXiv preprint arXiv:2504.03206}, 2025.

\bibitem[Wang et~al.(2016)Wang, Kurth-Nelson, Soyer, Leibo, Tirumala, Munos,
  Blundell, Kumaran, and Botvinick]{wang2016learning}
J.~X. Wang, Z.~Kurth-Nelson, H.~Soyer, J.~Z. Leibo, D.~Tirumala, R.~Munos,
  C.~Blundell, D.~Kumaran, and M.~M. Botvinick.
\newblock Learning to reinforcement learn.
\newblock \emph{ArXiv}, abs/1611.05763, 2016.
\newblock URL \url{https://api.semanticscholar.org/CorpusID:13623631}.

\bibitem[Warrington et~al.(2021)Warrington, Lavington, Scibior, Schmidt, and
  Wood]{warrington2021robust}
A.~Warrington, J.~W. Lavington, A.~Scibior, M.~Schmidt, and F.~Wood.
\newblock Robust asymmetric learning in pomdps.
\newblock In \emph{International Conference on Machine Learning}, pages
  11013--11023. PMLR, 2021.

\bibitem[Weihs et~al.(2021)Weihs, Jain, Liu, Salvador, Lazebnik, Kembhavi, and
  Schwing]{advisor}
L.~Weihs, U.~Jain, I.-J. Liu, J.~Salvador, S.~Lazebnik, A.~Kembhavi, and
  A.~Schwing.
\newblock Bridging the imitation gap by adaptive insubordination.
\newblock In \emph{NeurIPS}, 2021.

\bibitem[Yin et~al.(2026)Yin, Westenbroek, Zhang, Tran, Dagnino, Shilamkar,
  Mbiziwo-Tiapo, Bagaria, Liu, Mullins, Kolobov, and Gupta]{yin2026emergent}
P.~Yin, T.~Westenbroek, Z.~Zhang, J.~Tran, I.~Dagnino, E.~Shilamkar,
  N.~Mbiziwo-Tiapo, S.~Bagaria, X.~Liu, G.~Mullins, A.~Kolobov, and A.~Gupta.
\newblock Emergent dexterity via diverse resets and large-scale reinforcement
  learning.
\newblock In \emph{Proceedings of the International Conference on Learning
  Representations (ICLR)}, 2026.

\bibitem[You et~al.(2025)You, Thomas, Schutz, Skilton, Hawes, and
  Buffet]{You_Thomas_Schutz_Skilton_Hawes_Buffet_2025}
Y.~You, V.~Thomas, A.~Schutz, R.~Skilton, N.~Hawes, and O.~Buffet.
\newblock Partially observable monte-carlo graph search.
\newblock \emph{Proceedings of the International Conference on Automated
  Planning and Scheduling}, 35\penalty0 (1):\penalty0 279--287, Sep. 2025.
\newblock \doi{10.1609/icaps.v35i1.36129}.
\newblock URL \url{https://ojs.aaai.org/index.php/ICAPS/article/view/36129}.

\bibitem[Yu et~al.(2020)Yu, Quillen, He, Julian, Hausman, Finn, and
  Levine]{yu2020meta}
T.~Yu, D.~Quillen, Z.~He, R.~Julian, K.~Hausman, C.~Finn, and S.~Levine.
\newblock Meta-world: A benchmark and evaluation for multi-task and meta
  reinforcement learning.
\newblock In \emph{Conference on robot learning}, pages 1094--1100. PMLR, 2020.

\bibitem[Zhang et~al.(2026)Zhang, Guaman~Castro, Yin, Dagnino, Gupta, Scalise,
  and Boots]{zhang2026balanced}
O.~Zhang, M.~Guaman~Castro, P.~Yin, I.~Dagnino, A.~Gupta, R.~Scalise, and
  B.~Boots.
\newblock A balanced data diet: Addressing the exploration bottleneck in
  mega-scale rl for robot control.
\newblock In \emph{Conference on Robot Learning (CoRL)}, 2026.

\bibitem[Zintgraf et~al.(2021)Zintgraf, Schulze, Lu, Feng, Igl, Shiarlis, Gal,
  Hofmann, and Whiteson]{zintgraf2021varibad}
L.~Zintgraf, S.~Schulze, C.~Lu, L.~Feng, M.~Igl, K.~Shiarlis, Y.~Gal,
  K.~Hofmann, and S.~Whiteson.
\newblock Varibad: variational bayes-adaptive deep rl via meta-learning.
\newblock \emph{J. Mach. Learn. Res.}, 22\penalty0 (1), Jan. 2021.
\newblock ISSN 1532-4435.

\end{thebibliography}
\bibliographystyle{abbrvnat}

\newpage
\appendix

\startcontents[appendix]

\section*{Appendix Contents}
\printcontents[appendix]{}{1}{\setcounter{tocdepth}{2}}

\clearpage

\section{Related Work}\label{app:related_work}

\textbf{Student--Teacher Distillation.}
Learning from privileged experts arises in settings from self-driving \citep{chen2020learning, de2019causal, bansal2018chauffeurnet} to robotic exploration \citep{choudhury2017adaptive, jain2021gridtopix, bhardwaj2017learning}. In partially observed domains, interactive imitation learning addresses covariate shift and the resulting, compounding errors that drive the learner into unseen states \citep{chen2020learning, lee2020learning} with classic iterative schemes including \citep{daume2009search, ross2011reduction, ross2014reinforcement, sun2017deeply}. Recent results suggest that on-policy imitation is necessary under partial observability \citep{swamy2022sequence}, and that off-policy cloning \citep{liu2023emergent} can fail in non-realizable settings where information-gathering behavior is required \citep{spencer2021feedback, advisor}. Our work targets this non-realizable regime: the learner must take information-gathering actions that are not directly recoverable from a fully observed expert \citep{lee2021bayesian}. Prior approaches handle this via additional assumptions (e.g., access to confounders at test time \citep{tennenholtz2021covariate} or differentiable experts \citep{warrington2021robust}) or by incorporating reinforcement learning with bootstrapping via demonstrations or resets to induce exploration \citep{nguyen2022leveraging, walsman2023impossibly, hu2025realworld, advisor, kim2025distilling}. In contrast, we learn information-gathering behavior using iterative distillation conditioned on on-policy histories, which is substantially more stable and sample-efficient than RL-based alternatives.

\textbf{Meta Learning.}
A common approach to learning in POMDPs is to use recurrent policies that compress the interaction history into a latent belief representation \citep{hausknecht2015deep, Igl2018DeepVR}. Meta-RL provides a general framework for learning adaptive exploration across task distributions~\citep{duan2016rl,zintgraf2021varibad,Rakelly2019EfficientOM,pmlrv139liu21s}, but often requires extensive online training interaction~\citep{wang2016learning,bauer2023human}, making it difficult to scale to high-dimensional, sparse-reward domains with long-horizon credit assignment. Our approach replaces this expensive RL-based adaptation with iterative supervised learning, still enabling long-horizon, history-dependent exploration through on-policy context collection. Additionally, unlike prior work~\citep{lee2023supervised, zintgraf2021varibad, duan2016rl, lee2021bayesian}, our method is not limited to Bayes-Adaptive MDPs \citep{lee2021bayesian}, where the agent interacts with an environment with a fixed underlying parameter across multiple episodes. We demonstrate that our method scales to POMDPs, in which the policy explores actions within a single episode to estimate the underlying state distribution and achieve task success.

\textbf{Posterior Sampling.}
Posterior sampling methods \citep{thompson1933, osband2013psrl} drive exploration by implicitly or explicitly maintaining a belief over the environment's underlying hidden parameters, acting optimally under samples from this distribution, and updating the belief based on feedback. Despite strong empirical and theoretical results \citep{chapelle2011thompsoneval, osband2013psrl}, such approaches typically require maintaining and sampling from an explicit posterior, and solving for policies under the sampled parameters (e.g., PSRL, BOSS) \citep{osband2013psrl, asmuth2012boss}, limiting practicality in rich-observation settings. Recent approaches \cite{somani2013despot, You_Thomas_Schutz_Skilton_Hawes_Buffet_2025} learn Bayes-optimal policies for POMDPs but are applied in discrete settings, whereas posterior sampling yields a simple, scalable solution with provable regret guarantees. \MethodS{} does not explicitly optimize a posterior-sampling objective, but in Section~\ref{sec:exp} we show that posterior-sampling-like behavior emerges naturally from our method through iterative on-policy distillation, even in high-dimensional continuous environments. Further, in Appendix~\ref{app:theory}, we show that theoretically, \MethodS{} converges to posterior sampling with linear regret akin to standard online learning methods \citep{ross2011reduction}. We show that under stated assumptions, \MethodS{} converges to posterior sampling on the covered history support. We also observe this behavior empirically in high-dimensional manipulation, and it transfers to a real robot (see Section~\ref{sec:exp}).

\textbf{In-Context Decision-Making.}
Recent works~\citep{intelligence2025pi_, brown2020language, vonwerra2020trl} extensively use transformers as decision-making policies. When trained on a large corpus of sequential data using autoregressive supervised prediction, transformers can exhibit in-context learning \citep{fu2024icrt, min2022rethinking, garg2022can}, i.e., adapting their behavior to new inputs based on interaction histories or input-output examples provided in context. In online decision-making, prior works have used these emergent capabilities to steer policies conditioned on rewards or prior context. For example, Decision Transformer (DT)~\cite{chen2021decisiontransformer} learns a transformer to autoregressively predict actions conditioned on past interactions and the desired return. Similarly, Algorithm Distillation \cite{Laskin2023InContextRL} uses experience traces from a learning method (such as PPO \citep{schulman2017ppo}) to adapt and improve policy behavior based on past contexts. More recently, Decision-Pretrained Transformer~\citep{lee2023supervised} shows that we can learn an in-context policy that performs posterior sampling in certain POMDPs. DPT generates the context using random interactions in the environment and action supervision from oracle experts with privileged information. Such methods assume that these offline datasets provide coverage over all possible histories of interaction, which is intractable and unrealistic in high-dimensional domains. \MethodS, instead, uses a dataset of on-policy contexts, removing the coverage assumption and scaling to much harder and higher-dimensional problems (see Section~\ref{sec:exp}).

\section{Ablation Experiments}\label{app:ablation}

\subsection{Performance of \MethodS{} on additional environments}

In Figure~\ref{fig:extra_envs}, we show the performance of our proposed method across additional partially observed environments. \MethodS{} leverages on-policy context accumulation and expert supervision to solve these tasks efficiently. In the Tactile-Pick  w/o Proprioception task, the objective is to pick a randomly placed cube on the table. The arm is randomly initialized, and the only observation available to the student is tactile forces and the last action. We observe that ASTEROID effectively solves the high-dimensional POMDP task with $\sim92\%$ success by implicitly maintaining a belief over the underlying state. We plot the performance over iterations and the resulting behavior in Figure~\ref{fig:no_proprioception}. Note that in the tactile-pick (and some other) tasks, BC while history-conditioned is not efficient and achieves non-trivial performance (here, $\sim48\%$) because certain evaluation reset conditions do not require significant or directed exploration to locate the object.

\begin{figure*}[h]
    \centering
    \begin{subfigure}[t]{0.32\linewidth}
        \centering
        \includegraphics[width=\linewidth]{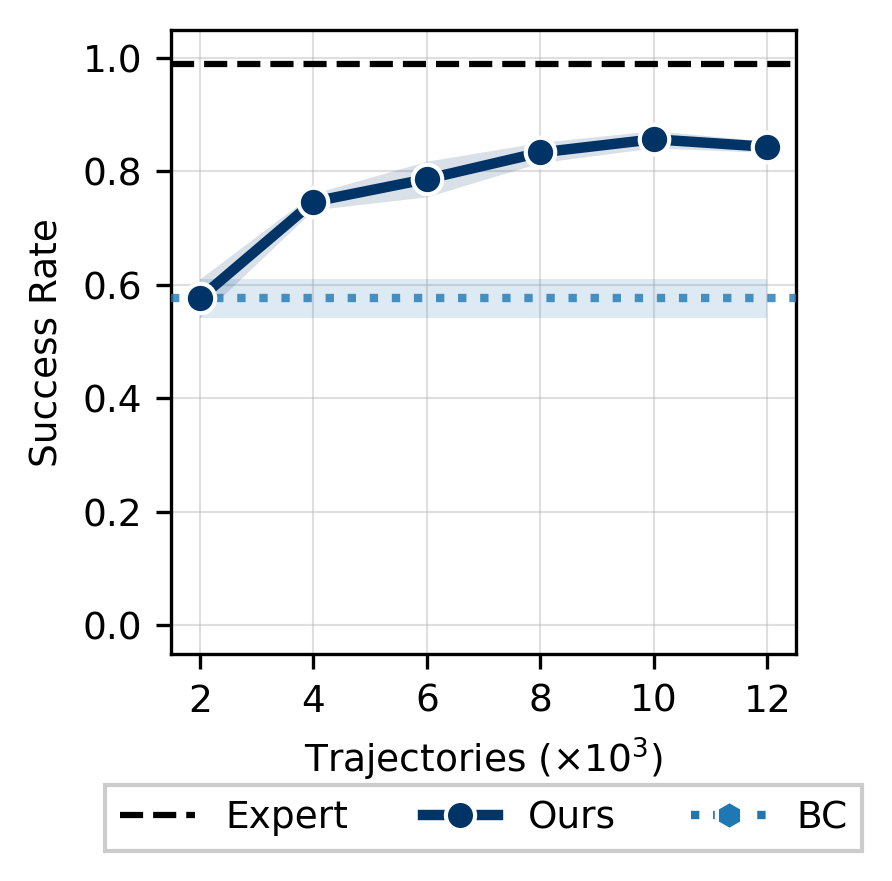}
        \caption{Fixed Maze Navigation}
        \label{fig:maze_nav}
    \end{subfigure}
    \hfill
    \begin{subfigure}[t]{0.32\linewidth}
        \centering
        \includegraphics[width=\linewidth]{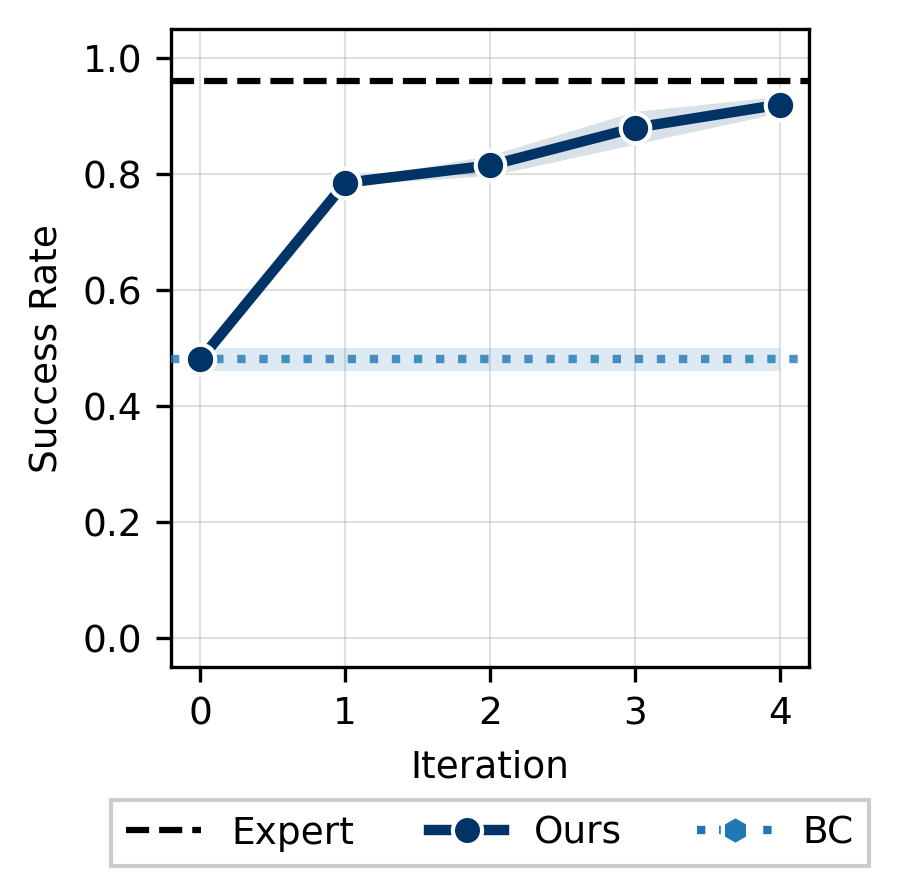}
        \caption{Tactile-Pick  w/o Proprioception}
        \label{fig:no_proprioception}
    \end{subfigure}
    \hfill
    \begin{subfigure}[t]{0.32\linewidth}
        \centering
        \includegraphics[width=\linewidth]{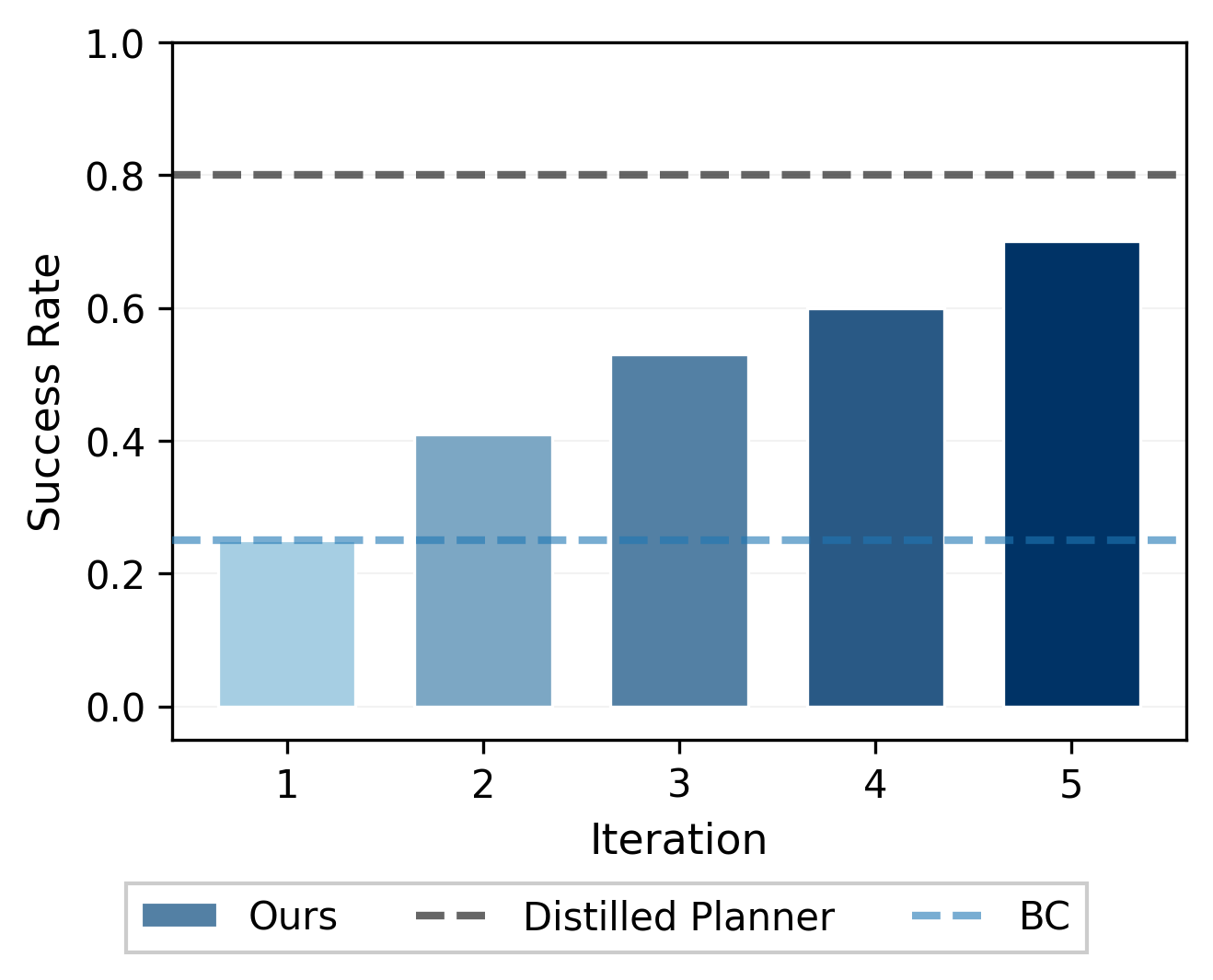}
        \caption{Habitat Object Search}
        \label{fig:habitat}
    \end{subfigure}
    \caption{\MethodS{} learns to solve POMDPs such as maze navigation, robot manipulation, and Habitat object search, where agents must act under limited fields of view, random resets, and partial sensing.}
    \label{fig:extra_envs}
\end{figure*}

\begin{figure*}[h]
    \centering
    \begin{subfigure}[t]{0.31\linewidth}
        \centering
        \includegraphics[width=\linewidth]{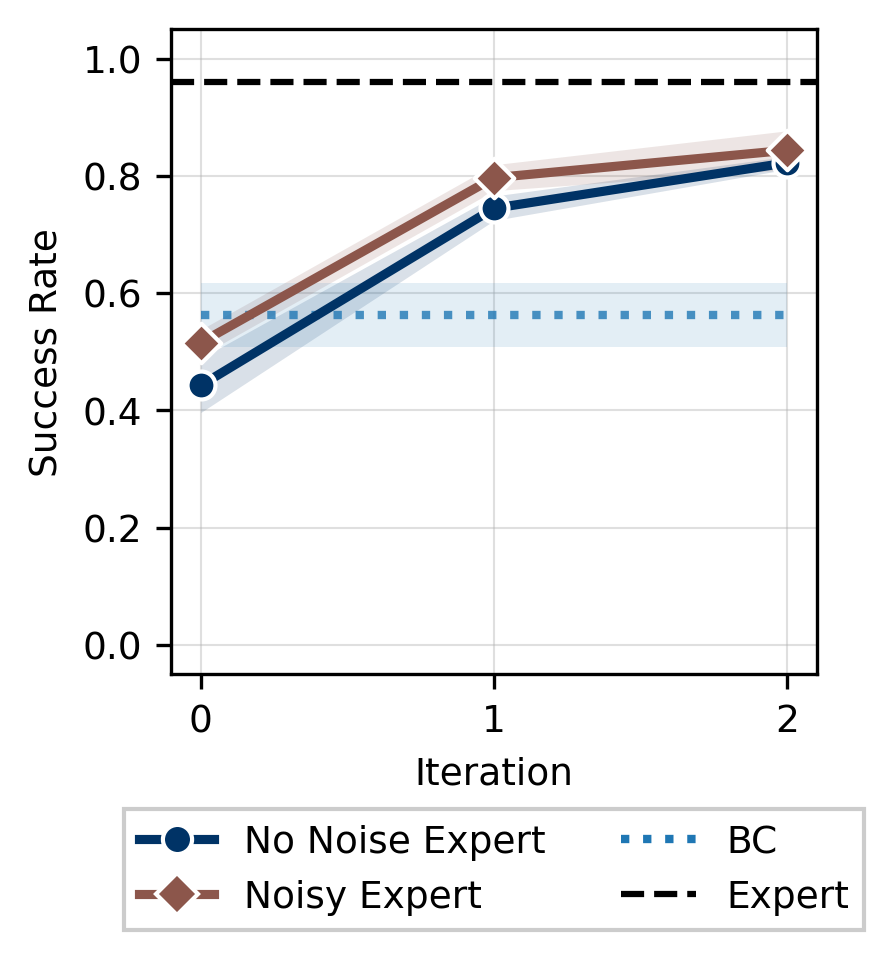}
    \end{subfigure}
    \hfill
    \begin{subfigure}[t]{0.34\linewidth}
        \centering
        \includegraphics[width=\linewidth]{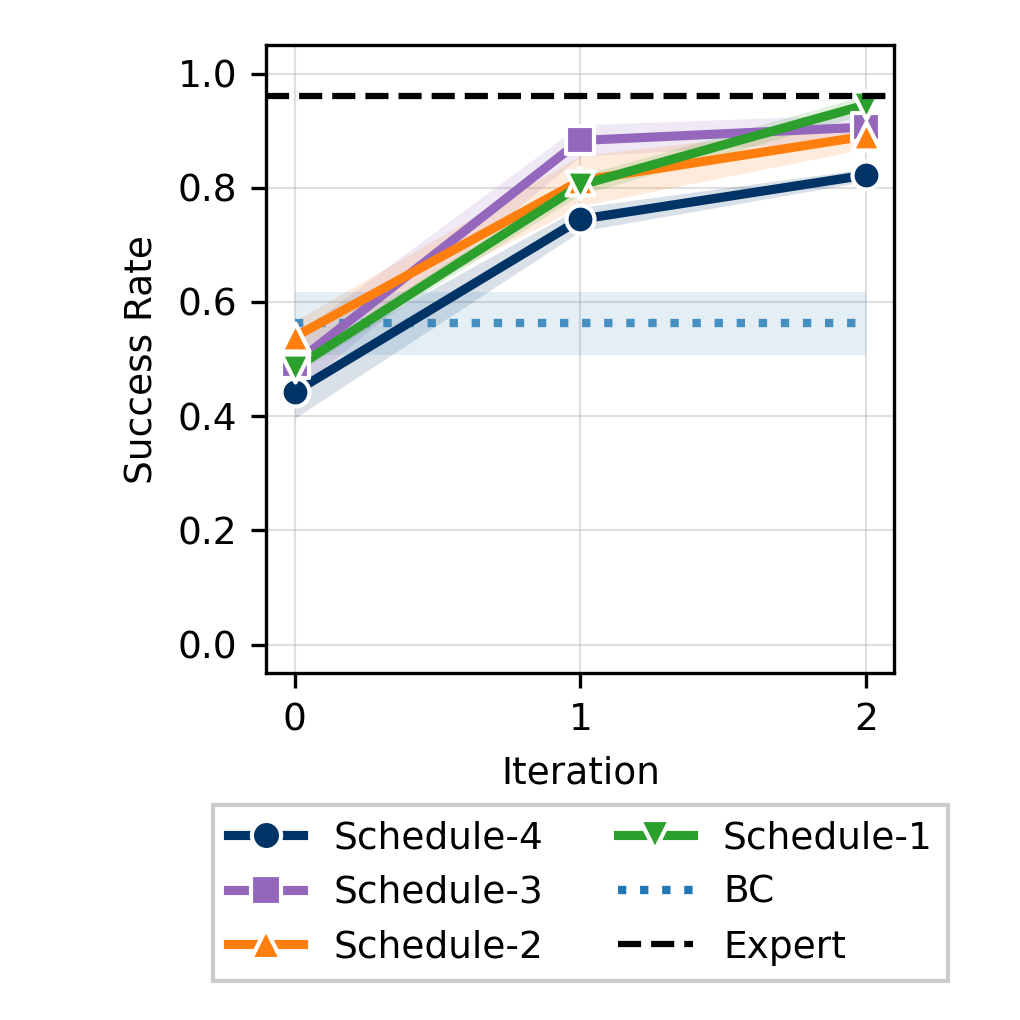}
    \end{subfigure}
    \hfill
    \begin{subfigure}[t]{0.32\linewidth}
        \centering
        \includegraphics[width=\linewidth]{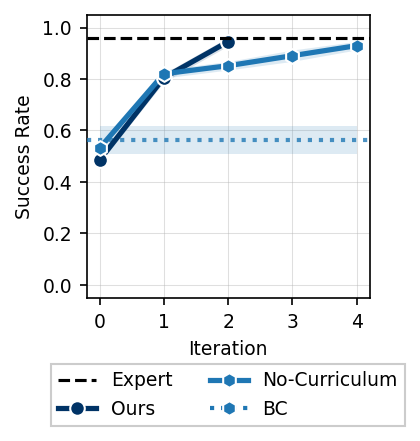}
    \end{subfigure}
    \caption{We run ablations on different parts of \MethodS{} in the Tactile-Pick task. Left: we run \MethodS{} with a noisy expert and observe no performance drop. Middle: we sweep over different curricula with increasing exploration horizons and find minimal sensitivity. Right: we show that adding curriculum improves the efficiency of \MethodS{} by ensuring that as the horizon increases over iterations, the learned policy does not suffer from OOD histories or generating random actions.}
    \label{fig:ablation_combined}
\end{figure*}

\subsection{Robustness of \MethodS{} to noise in privileged expert supervision}\label{app:noisy_ablation}

In Figure~\ref{fig:ablation_combined} (left), we evaluate the robustness of \MethodS{} through an ablation study in which a noisy expert supervises the student. We inject Gaussian noise into the expert actions during training ($\sigma=0.2$) and find that ASTEROID remains performant under noisy supervision. The slight performance improvement may stem from the regularizing effect of learning from noisy targets.

\subsection{How does the curriculum of exploration affect \MethodS{} during training?}\label{app:curriculum_ablation}
At each iteration $k < K$, the policy interacts with the environment for $k$ episodes, thereby inducing an increasing curriculum on on-policy rollout horizons. As the context window gradually expands, it allows the student to use longer histories to rule out hidden states or task instances inconsistent with its observations and feedback (Appendix~\ref{sec:context_length}).
The curriculum also avoids a subtle distribution-shift issue. Suppose the maximum context length is $K$, and we roll out an intermediate policy $\pi_{\theta_k}$ for more than $k$ episodes. Since $\pi_{\theta_k}$ was only trained on contexts of length at most $k$, maintaining an in-distribution context would require a sliding window over the most recent $k$ episodes. However, this window can discard early exploration outcomes. If a failed goal is no longer in context, the policy may re-explore it, producing redundant exploration histories that a fully converged policy would not generate. Informally, this reduces the effective number of length-$H$ exploration contexts from $|\mathcal{Z}|^H$ arbitrary failure sequences to at most $\frac{|\mathcal{Z}|!}{(|\mathcal{Z}|-H)!}$ ordered non-redundant histories. This further reduces the effective context space, thereby making \MethodS{} sample-efficient.

In Figure~\ref{fig:ablation_combined} (right), we compare the performance of \MethodS{} across two different schedules: (i) No-Curriculum: the exploration horizon is sampled from a fixed range across all iterations, and (ii) Curriculum: the exploration horizon is sampled from a range that increases every iteration. We observe that curriculum-based exploration is more efficient than the non-curriculum variant, as it keeps student policy rollouts in distribution and makes on-policy data generation more sample-efficient. Furthermore, in Figure~\ref{fig:ablation_combined} (middle), we compare the robustness of the method across different values of increasing curriculum and observe that it is relatively insensitive to this parameter.

\subsection{\MethodS{} is significantly more scalable than RL in complex domains}\label{app:rl_ablation}

\begin{wrapfigure}{l}{0.3\linewidth}
    \vspace{-0.3cm}
    \centering
    \includegraphics[width=0.8\linewidth]{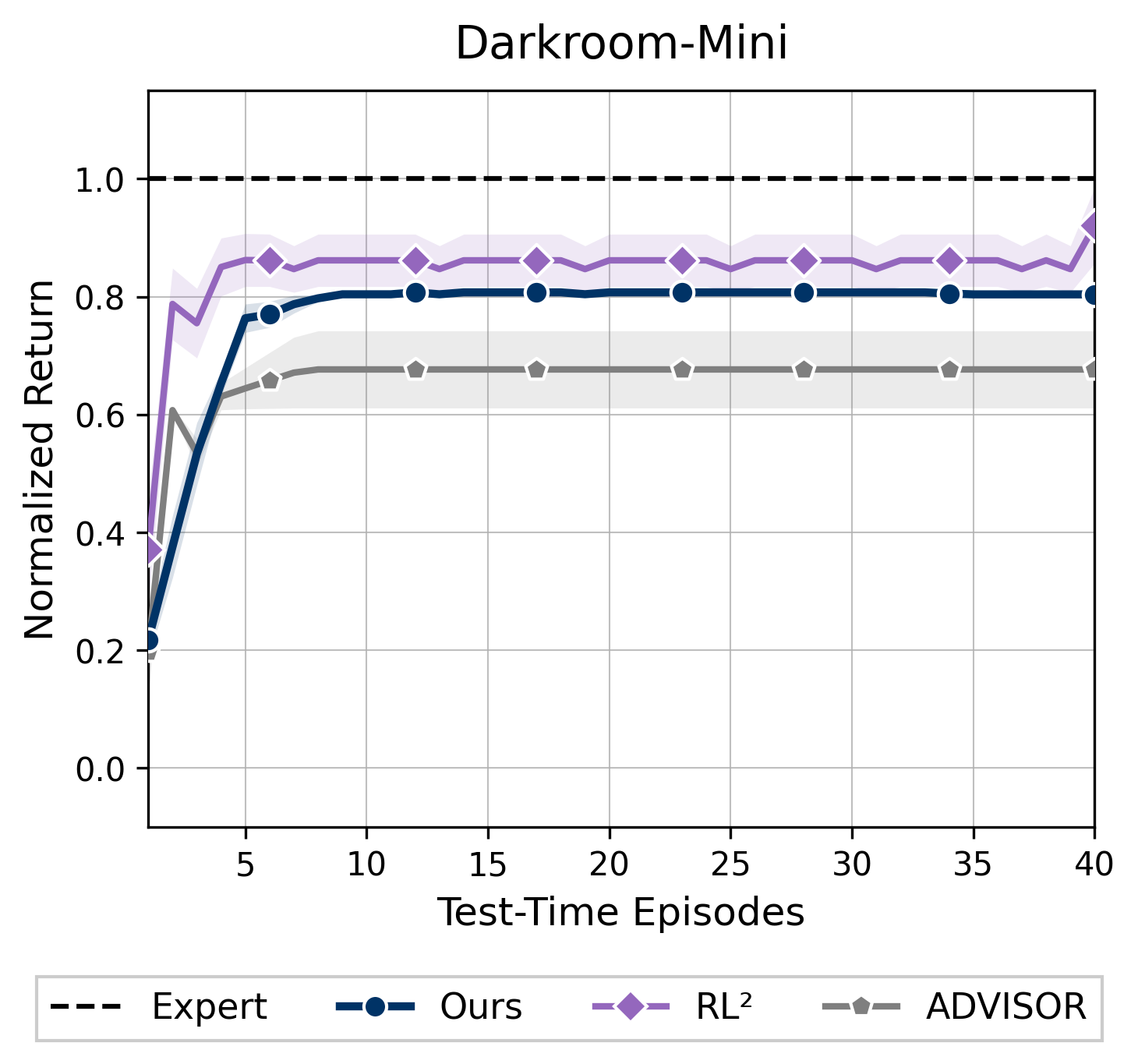}
    \caption{\footnotesize RL is effective over short horizons, but \MethodS{} uses supervised updates to be effective across long-horizon tasks as well.}
    \label{fig:darkroom_mini}
\end{wrapfigure}

In Figures~\ref{fig:incontext_pomdp_tasks} and~\ref{fig:pomdp_manipulation}, we observe that RL-based baselines (RL$^2$, BC+RL$^2$, and ADVISOR) severely underperform in some long-horizon sparse reward tasks. To further ablate this, we run the RL baselines and our method in a smaller version of the Darkroom environment, where the grid size is $5\times5$. We observe that RL baselines learn competitive search policies in these small-scale environments. \MethodS{} uses supervised learning to model expert actions while building exploration context via student rollouts, rather than relying on expert supervision and dynamic programming over random contexts to learn how to explore and exploit information. As a result, \MethodS{} performs better than RL baselines in high-dimensional, long-horizon settings.

\subsection{Does iterative context accumulation actually help \MethodS?}\label{app:iteration_ablation}\label{sec:context_length}

\begin{figure}[h!]
    \centering
    \includegraphics[width=\linewidth]{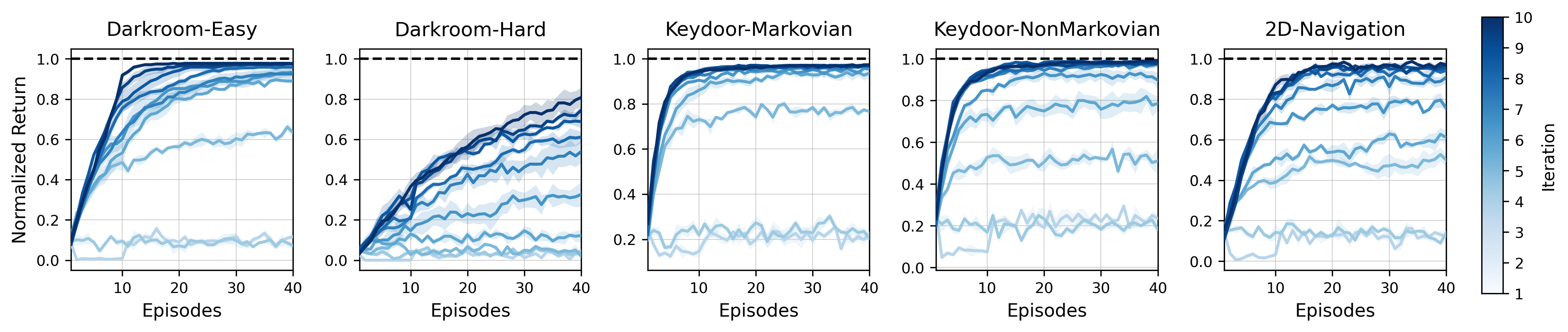}
    \caption{The test-time efficiency of \MethodS{} increases significantly with more iterations of asymmetric on-policy distillation, demonstrating that as the context dataset grows, the policy can adapt over longer horizons, making exploration more efficient.
    \label{fig:iteration_ablation}}
\end{figure}

Figure \ref{fig:habitat} shows exploration performance as a function of the number of in-context exploration iterations on a high-dimensional navigation task. In Habitat Object Search, performance improves steadily as additional exploration steps allow the policy to gather informative observations and refine its behavior, despite the complexity of the observations, almost converging to the posterior-sampling upper bound. Further, in Figure~\ref{fig:iteration_ablation}, we observe that the test-time efficiency of \MethodS{} improves significantly with more iterations. These results indicate that \MethodS{} effectively uses in-context interaction to accumulate task-relevant information, translating additional exploration into improved downstream task performance in realistic, partially observable environments, such as higher success rates in embodied object search in Habitat and in robot manipulation tasks. In Appendix~\ref{app:curriculum_ablation}, we further include ablations showing the benefits and robustness of \MethodS{} to the curriculum over exploration horizon across multiple iterations.

\subsection{How well does \MethodS{} generalize under shifted latent priors?}\label{app:shifted_priors}

\begin{wrapfigure}{r}{0.6\textwidth}
    \centering
    \includegraphics[width=\linewidth]{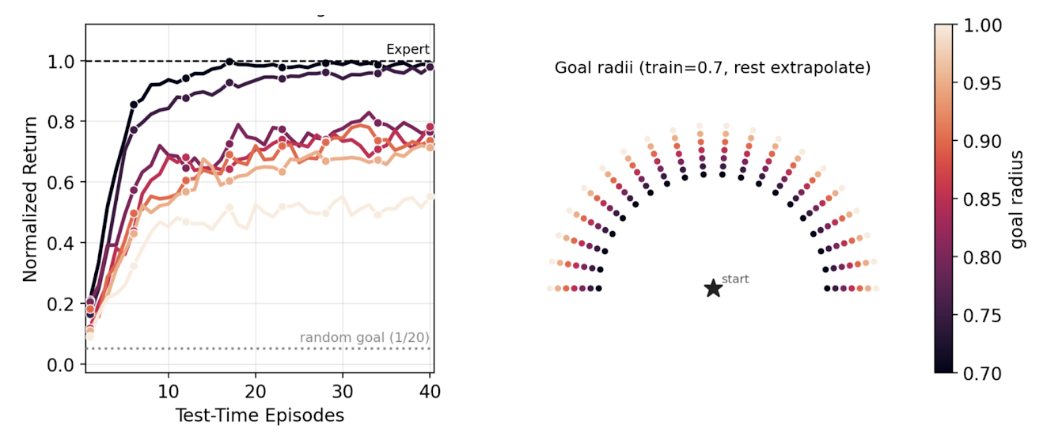}
    \caption{Generalization under shifted latent priors in 2D Navigation. Training goals lie at radius 0.7; evaluation goals are sampled at increasing radii up to 1.0. Performance degrades gradually with shift magnitude.}
    \label{fig:shifted_priors}
\end{wrapfigure}

In Figure~\ref{fig:shifted_priors}, we additionally measure how far generalization extends. In the 2D Navigation task, training goals are sampled on a semicircle of radius 0.7 from the origin, and at evaluation we sample goals at increasing radii from 0.7 (in-distribution) to 1.0 (far out-of-distribution), keeping the dynamics and observations unchanged. This isolates the shift to the latent (goal) distribution alone. As shown in Figure~\ref{fig:shifted_priors}, \MethodS{} retains near-optimal performance near the training distribution (normalized return of 0.99 at radius 0.7 and 0.98 at 0.75) and degrades gradually as the goals move farther away (0.41 at radius 1.0). Generally, amortized POMDP solvers~\cite{lee2023supervised, osband2013psrl, finn2017maml} that are learned via interaction or offline datasets suffer from a decline in performance as the train and test distributions diverge, and this observation is consistent with past results~\cite{nikulin2025xland100blargescalemultitaskdataset,yu2020meta}.

\section{Qualitative Results}\label{app:qualitative}

\subsection{\MethodS{} learns efficient exploration in robot manipulation under limited sensing and navigation}

\begin{wrapfigure}{l}{0.55\linewidth}
    \centering
    \includegraphics[width=\linewidth]{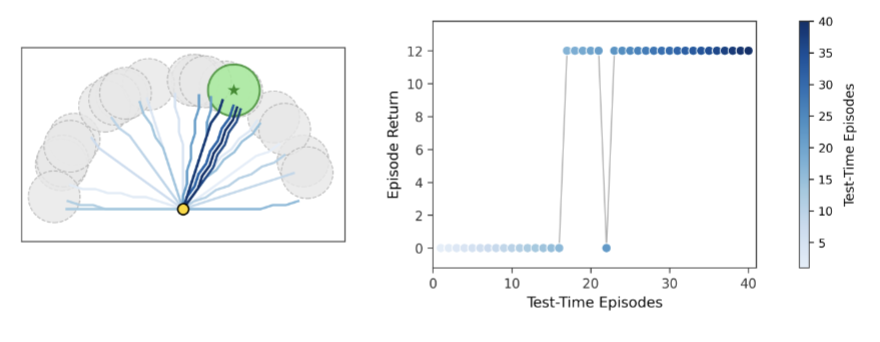}
    \caption{ In 2D navigation, \MethodS{} adapts its exploration based on context, converging to the true goal.}
        \label{fig:explore_2d}
\end{wrapfigure}

In Figures~\ref{fig:explore_wrist} and~\ref{fig:explore_2d}, we further compare the qualitative exploration behavior across the Peg-Insertion and 2D Navigation tasks. We observe that \MethodS{} enables active exploration-like behavior during manipulation, in which the robot arm moves efficiently across the table to locate the peg and the target. Similarly, in the 2D navigation maze task, \MethodS{} learns a policy that adaptively searches for the goal in unexplored regions and converges to high reward in subsequent episodes once it can identify the goal region from past trials. 
We include full videos of the policies on our website\footnote{Project webpage: \href{https://weirdlabuw.github.io/asteroid/}{weirdlabuw.github.io/asteroid}}, which also includes goal-directed exploration across different rooms in Habitat.

\subsection{\MethodS{} presents emergent in-context properties such as memory as a result of on-policy in-context learning}\label{sec:mem}

\begin{figure}[h!]
    \centering
    \begin{subfigure}{0.8\linewidth}
        \centering
        \includegraphics[width=\linewidth]{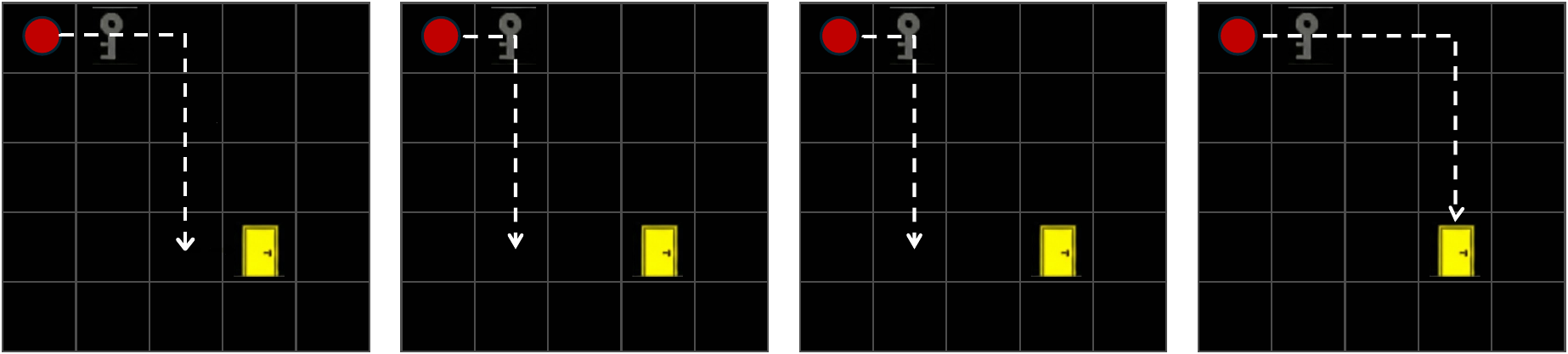}
        \caption{The agent finds the key in earlier episodes, can reference it back in its context, and always picks up the key before searching for the door.}
    \end{subfigure}\hfill
    \begin{subfigure}{0.8\linewidth}
        \centering
        \includegraphics[width=\linewidth]{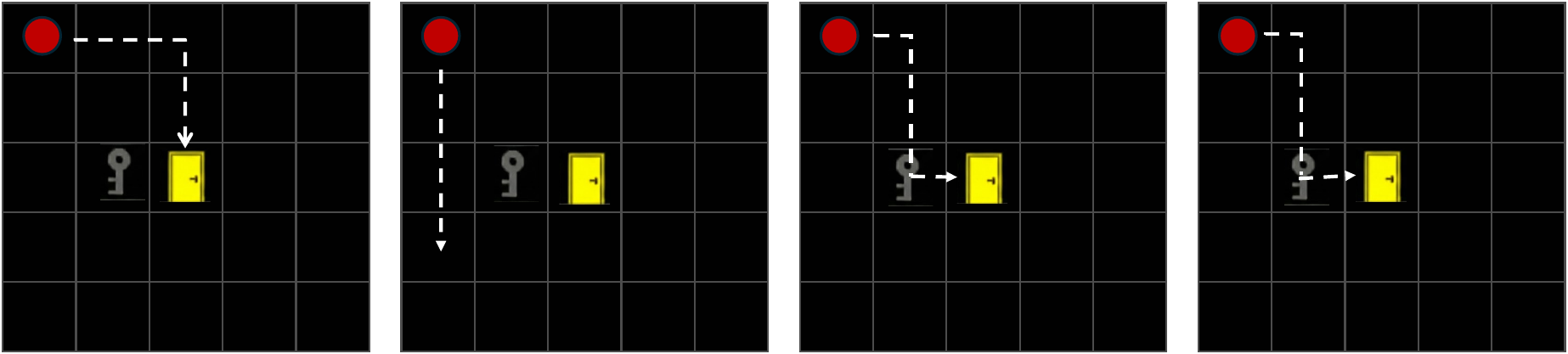}
        \caption{The agent finds the door in the first episode but gets no reward. So, in the next episodes, it looks for the key and memorizes the door location. As a result, in episode 3, it finds the key and goes directly to the door (without having to search for it, since it has memory).}
    \end{subfigure}\hfill
    \caption{In the Key-Door task, the agent explores to find both the key and the door to get +1 reward for each task. \MethodS{} learns an effective strategy: if it reaches the key or door first, it uses its context to remember the location, and upon finding the other goal, it can immediately complete the task without additional exploration.\label{fig:memory}}
\end{figure}

To investigate whether the policy learned with \MethodS{} actually attends to context, we construct a variant of the Key-Door environment in which the agent observes two boolean flags in its state whenever it lands on the key or the door during exploration. Policies trained with \MethodS{} exhibit emergent in-context memory. If the agent discovered the key or door in earlier attempts, it can later retrieve that information from context and use it efficiently.

\section{Per-step interactive imitation learning fails in history-conditioned asymmetric distillation}\label{app:dagger}

\begin{wrapfigure}{r}{0.5\linewidth}
    \centering
    \includegraphics[width=\linewidth]{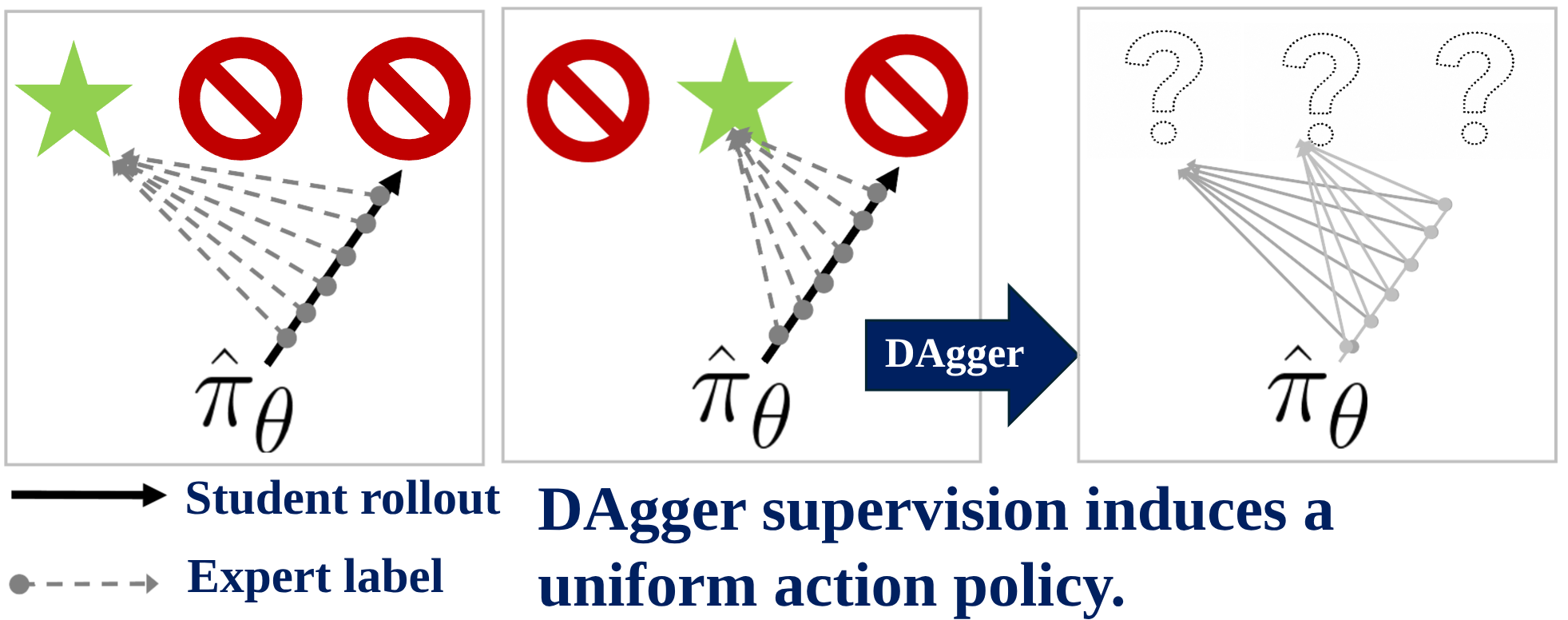}
    \caption{Under DAgger, the same student rollout receives per-step expert labels toward different goals depending on the hidden goal (left, middle). Aggregating these labels yields a policy that chooses a goal uniformly at each step (right) instead of committing to one hypothesis.}
    \label{fig:dagger}
\end{wrapfigure}

An extension to improve naive teacher--student distillation is to apply standard interactive learning methods, such as DAgger~\cite{ross2011reduction}, on a history-conditioned student policy $\pi_\theta(\cdot | o_t, h_t)$. However, these methods query the expert at every state the student visits, i.e., they train on $(h_t, o_t) \rightarrow a^*_t$, while \MethodS{} uses the student rollouts only as context and imitates the expert at the states it visits, i.e., $(h_t, o^*_t) \rightarrow a^*_t$. Further, in POMDPs, DAgger surprisingly converges to an exploration strategy that is generally worse than naive distillation. Consider that a policy has to take $H$ steps towards a potential goal to receive a reward and disambiguate between them. Here, $\pi_\theta(\cdot | o_t, h_t)$ cloned on expert behaviour would randomly pick a door and generate an $H$ step trajectory. Applying DAgger, an expert provides supervision at each step of the student rollout. In Figure~\ref{fig:dagger}, we see that under different latent settings, the same student trajectory receives multi-modal action targets at each step. Supervised learning on this aggregated dataset would instead learn a policy that chooses a random goal at each step, which does not correspond to meaningful, effective exploration. Over repeated interactions, the updated policy uniformly explores the environment to find the goal, leading to extremely high exploration budgets. Overall, DAgger underperforms as a baseline in partially observed domains, even below naive behaviour cloning (see Figure~\ref{fig:pomdp_manipulation}). Instead, \MethodS{} first collects student rollouts as context and then uses trajectory/context-level supervision. In BAMDPs, this expert supervision can also come from offline data, whereas DAgger queries an online expert. This encourages the policy to execute goal-directed actions over a longer horizon (rather than uniform dithering), resulting in much more efficient exploration, even in high-dimensional continuous action spaces (see Figure~\ref{fig:ant_behavior}).

\section{Didactic Example: \MethodS{} learns efficient posterior sampling\label{app:toy}}

\begin{figure}[!h]
    \centering
    \includegraphics[width=0.6\linewidth]{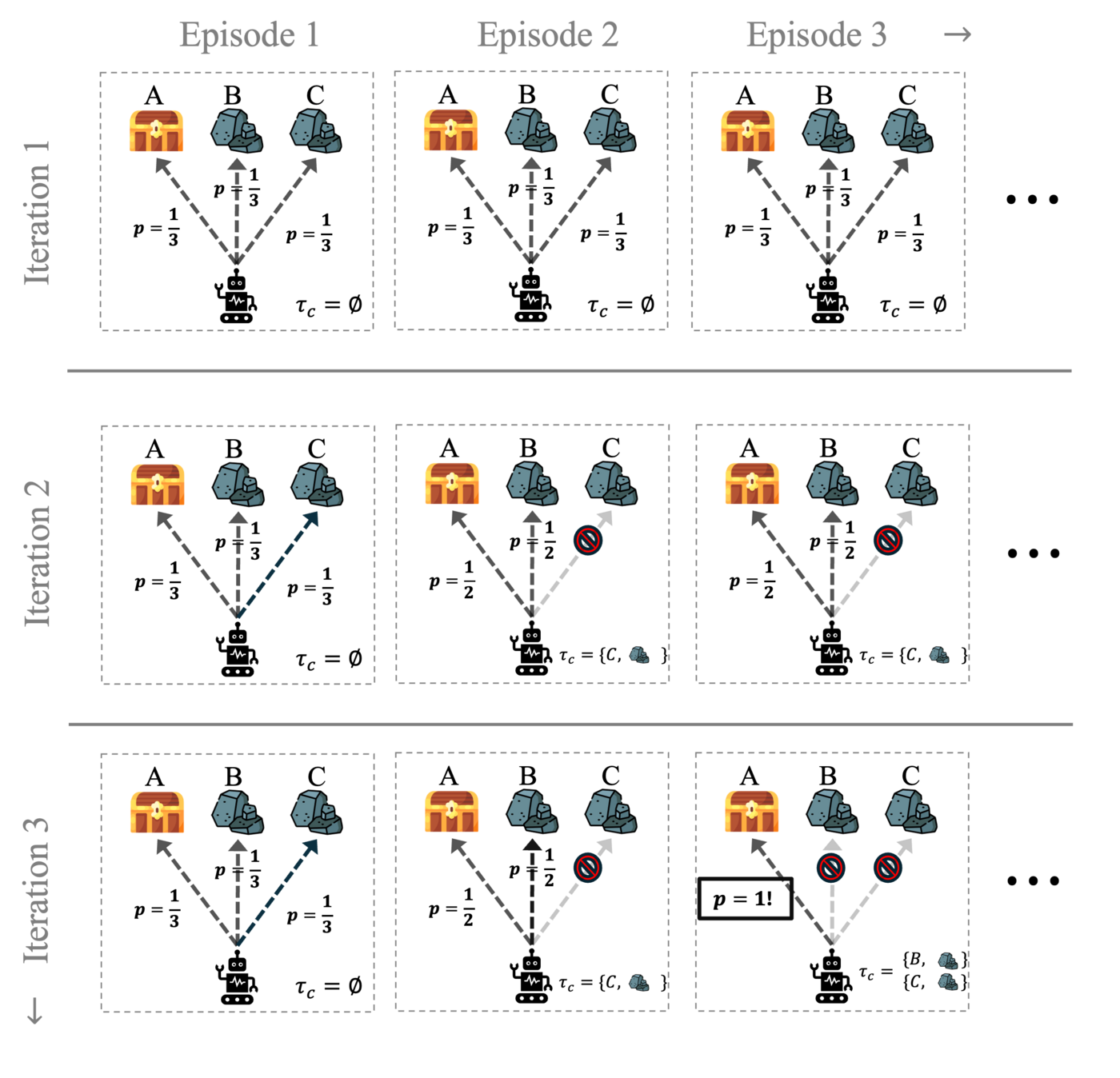}
    \caption{Toy example demonstrating that \MethodS{} learns adaptive in-context behavior, where it effectively explores based on history.}
    \label{fig:toy}
\end{figure}

Consider a POMDP with one hidden variable. For each task, a fixed goal $g \in \{A, B, C\}$ is sampled uniformly. The agent takes actions to move to one of the locations in $\{A, B, C\}$ and receives a reward of $1$ if it matches $g$ and $0$ otherwise.
A privileged teacher observes $g$ and always chooses the correct door, producing expert trajectories: $(A,1)$, $(B,1)$, or $(C,1)$ depending on the goal. If the student is trained only on these expert rollouts with no context (i.e., $\tau^c=\emptyset$), then at deployment it cannot infer $g$ and must choose uniformly, yielding expected reward $1/3$.

\textit{Dry run.}
At iteration $k=0$, we train on expert-only trajectories with $\tau^c=\emptyset$, so the student has no information about $g$ at decision time.
The resulting policy chooses $A, B, C$ uniformly, giving an expected reward of $1/3$.

Now deploy this student and collect on-policy contexts.
If the true goal is $g=C$, then across episodes the observed contexts are of the form $(A,0)$, $(B,0)$, or $(C,1)$, depending on which door was tried.
At iteration $k=1$, \MethodS{} pairs each such context with a privileged teacher trajectory generated under the same goal.
For example, conditioning on $\tau^c=(A,0)$ rules out $g=A$, so the teacher data for this context consists only of successful choices at the remaining doors, i.e., $(B,1)$ or $(C,1)$.
After distillation, the updated student still chooses uniformly when $\tau^c=\emptyset$, but it adapts after one failed attempt: given $\tau^c=(A,0)$ it chooses between $\{B, C\}$ uniformly, yielding expected reward $1/2$ on the next episode.

At iteration $k=2$, the student generates longer contexts such as $(A,0;B,0)$, $(A,0;C,1)$, etc.
A context like $(A,0; B,0)$ uniquely identifies $g=C$, so the teacher data for this context concentrates on choosing $C$.
After distillation, the student learns the corresponding rule: once it has observed two failures, it deterministically selects the remaining door and receives a reward of $1$.
Figure~\ref{fig:toy} depicts this progression: as $\tau^c$ grows with on-policy interaction, the student eliminates inconsistent goals based on context and its behavior approaches posterior sampling over $\{A, B,C\}$.

\paragraph{On-policy context enables more efficient learning than random exploration rollouts.}

\begin{figure}[!h]
    \centering
    \includegraphics[width=0.6\linewidth]{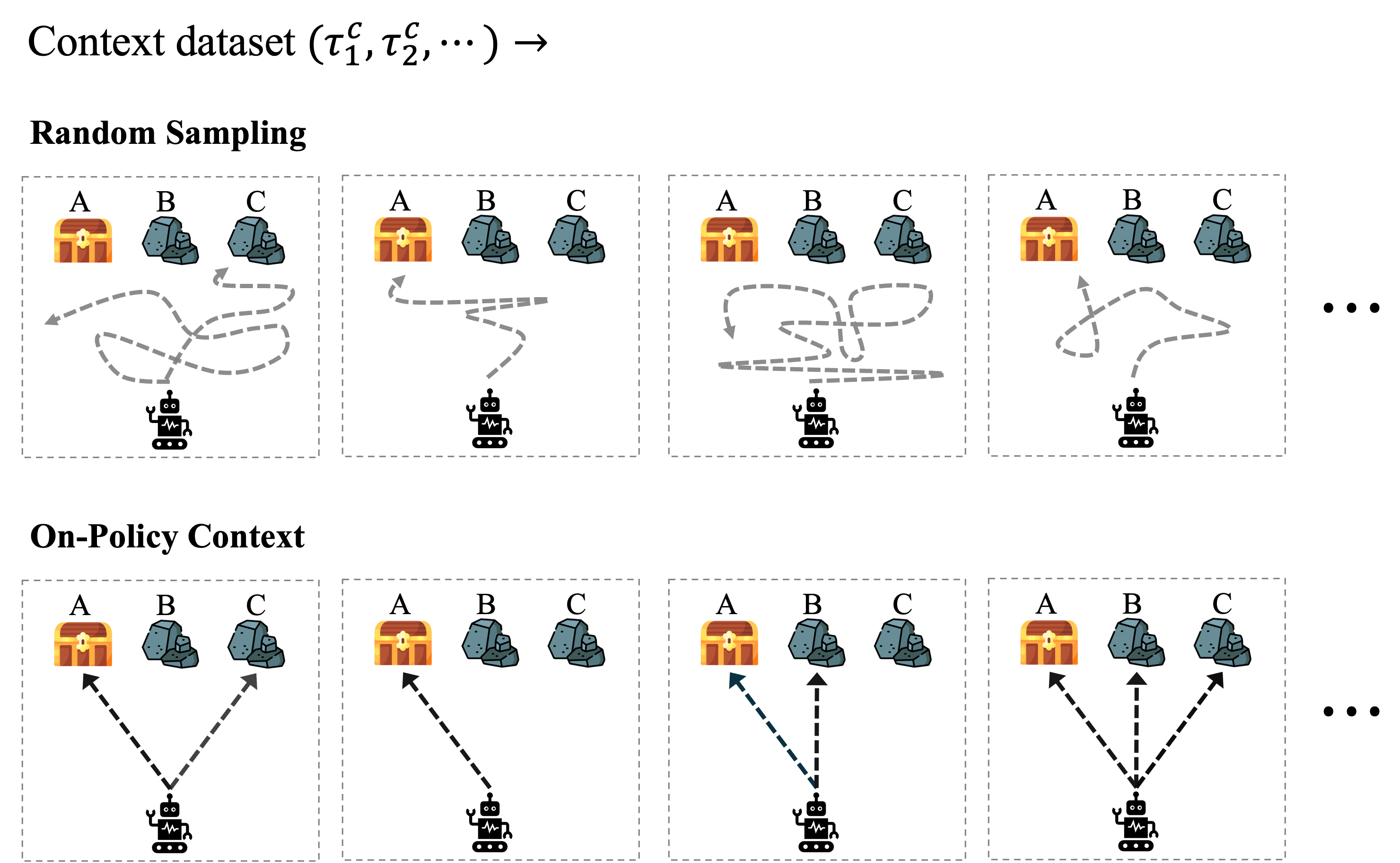}
    \caption{On-policy contexts ensure efficient coverage over relevant histories, while random exploration~\cite{lee2023supervised} can generate uninformative contextual information that harms training.}
    \label{fig:context}
\end{figure}

In this example, a context $\tau^c$ is useful only if it induces a non-trivial posterior over the goal, i.e., it eliminates some candidates in $\{A, B, C\}$ based on observed rewards.
Suppose instead that reaching each location requires multiple actions before any reward is revealed.
Then a large fraction of randomly generated rollouts will fail to reach any door within the context horizon (row 1 in Figure~\ref{fig:context}), producing trajectories whose rewards are uninformative and whose posterior over $g$ remains essentially unchanged.
By contrast, on-policy contexts produced by the student concentrate on trajectories that actually reach candidate doors, since the student is trained on teacher trajectories that are goal-directed and therefore yield plausible search behavior (row 2 in Figure~\ref{fig:context}). As a result, the contexts in $\mathcal{D}_k$ more often contain the key feedback events that reduce uncertainty about $g$, providing a higher-signal training distribution for learning a context-conditioned policy.

\section{Training and Implementation Details}\label{app:experiments}

\subsection{Training Details}

\paragraph{Model Architecture.} \MethodS{} uses a causal transformer similar to \citet{lee2023supervised, chen2021decisiontransformer}. Each token encodes the $\langle$observation, previous action, reward, terminal signal$\rangle$ tuple via an MLP encoder. For images, we encode observations with an ImageNet-pretrained ResNet18. For discrete actions, we predict logits over the action space, and for continuous actions, we predict the mean and standard deviation of a Gaussian.

\paragraph{Asymmetric Sampling of datasets.} For the robot manipulation tasks, to avoid overfitting, we skew the sampling ratio of the data collected across iterations to upweight the new data more. This significantly reduced overfitting to past data and enabled better policy learning.

\paragraph{Filtering for successful trajectories.} As \MethodS{} imitates privileged actions conditioned on context, we keep only the successful expert trajectories while collecting the datasets during each iteration. This ensures that the learned in-context policy samples from a distribution of expert actions that are optimal under the true parameters that could generate the given context.

\paragraph{Task Experts}

ASTEROID depends on a strong privileged expert but is agnostic to how it is generated: it may leverage simple planners, use RL, imitation learning, or hybrid methods~\citep{schulman2017ppo, yin2026emergent, zhang2026balanced}. Concretely: for all navigation and gridworld tasks, the experts are shortest-path planners. For manipulation, we train state-based policies via large-scale simulation RL~\citep{yin2026emergent} with hyperparameters from the original implementations. Privileged state makes these experts cheap to train. E.g., in peg insertion, a state-based teacher (object and goal 6-D pose) reaches 90+\% success in $\sim$8\,h, while a wrist-camera policy takes 4+ days to reach only $\sim$10\%. This state information is not available at deployment, and thus, we learn exploration via \MethodS. Most of our baselines also use privileged experts, so teacher cost is not exclusive to our method.

\paragraph{Compute.}
DPT, \MethodS, BC, and AAWR learn context-conditioned policies, and we use the same architecture and evaluation setup. Thus, these methods use the same tokens and computation during rollouts. During training, \MethodS{} additionally conditions on the expert observations, adding at most |T| tokens more than DPT, where |T| is the length of an episode. Unless noted otherwise, all models were trained on a single L40 GPU; training takes 1–6 h, and the main bottleneck is environment rollouts, not training. For Isaac Lab tasks, we used 2× RTX-4090s for our method and all baselines. We carefully match the compute in terms of environment interactions across all baselines (except when comparing efficiency). In all experiments, each baseline is trained on as much as or more data than \MethodS{}. ASTEROID runs $K$ iterations with $M$ environment steps each; all offline baselines (BC, DPT, SPOC, AAWR) train on a dataset of size $K \times M$. Figure~\ref{fig:pomdp_manipulation} plots training steps vs.\ performance for the manipulation baselines. RL baselines are trained to convergence. In Table~\ref{tab:data_parity}, we outline the parameters:

\begin{table}[h!]
\centering
\caption{Per-environment data and compute budget for ASTEROID and baselines.}
\label{tab:data_parity}
\resizebox{\linewidth}{!}{%
\begin{tabular}{lcccccc}
\toprule
Environment & $K$ (iters) & ASTEROID data (trajs) & Offline baselines data & RL baselines env steps & Expert labels (rollouts) & Teacher cost \\
\midrule
Gridworld (Darkroom, 2D-Nav) & 10 & $10 \times 1000 = 10$k (env horizon $H$) & 10k expert trajs & to convergence & 10k & negligible (planner) \\
Tactile-Pick & 3 & $3 \times 8192 \approx 24$k & $\approx$24k expert trajs & $\sim 10^8$ & $\approx$24k & $\sim$8\,h (state-based RL) \\
Peg-Insertion & 4 & $4 \times 4096 \approx 16$k & $\approx$16k expert trajs & $\sim 10^8$ & $\approx$16k & $\sim$8\,h (state-based RL) \\
Habitat & 5 & $5 \times 1024 = 5120$ & 5120 (Distilled Planner) & n/a & 5120 & negligible (planner) \\
\bottomrule
\end{tabular}%
}
\end{table}

\subsection{Tasks}

\paragraph{Wrist-Camera Peg-Insertion.}\cite{yin2026emergent}
A 7-DoF robot must locate a peg on a tabletop and insert it into the corresponding hole using only a wrist-mounted RGB camera and proprioception. This is a single-episode POMDP: the peg and arm are randomly reset at the start of each episode, and the robot must search under a limited camera field of view. We denote the privileged state as $s_t=(q_t,p_{\text{peg}},p_{\text{hole}})$, where $q_t$ is the robot proprioceptive state and $p_{\text{peg}},p_{\text{hole}}$ are the peg and hole positions. The student observes $o_t=(I_t^{\text{wrist}},q_t,a_{t-1})$, where $I_t^{\text{wrist}}$ is the wrist-camera RGB image and $a_{t-1}$ is the previous action. The reward indicates successful insertion, and expert trajectories are generated using large-scale RL as described in \citet{yin2026emergent}. An effective exploration strategy must scan the workspace, localize the peg and hole from partial views, and then execute the contact-rich insertion behavior.

\paragraph{Tactile-Pick.}\cite{yin2026emergent, higuera2023perceiving}
The robot must locate and pick up a cube using proprioceptive and tactile feedback, without visual observations. This is a single-episode POMDP: the cube pose is randomly reset and hidden, so the agent can only infer it through contact during the episode. We denote the privileged state as $s_t=(q_t,p_{\text{cube}})$, where $q_t$ is the robot proprioceptive state and $p_{\text{cube}}$ is the cube position. The student observes $o_t=(q_t,\tau_t, a_{t-1})$, where $\tau_t$ is the $2\times3$-dimensional force feedback from the two fingers of the parallel-jaw gripper. The reward indicates successful pickup, and expert trajectories are generated using large-scale RL as described in \citet{yin2026emergent}. An effective exploration strategy must probe the workspace, use contact events to infer the cube location, and then transition from search to grasping. We also evaluate a harder variant with observations $o_t=(\tau_t,a_{t-1})$, where the robot observes the world only through contact forces and past actions. In this setting, the robot starts from a random configuration and must use its interaction history to maintain a belief over both its own joint configuration and the cube position.

\paragraph{Darkroom.}\cite{MinigridMiniworld23,lee2023supervised}
The agent must find a hidden goal in a partially observable gridworld. We use three versions with different grid sizes: Mini ($5\times5$), Easy ($10\times10$), and Hard ($25\times25$). This is a BAMDP: the agent interacts with the environment over multiple test-time episodes, the goal remains fixed across episodes, and the agent is reset to $(0,0)$ at the start of each episode. We denote the privileged state as $s_t=(p_x,p_y,g_x,g_y)$, where $(p_x,p_y)$ is the agent position and $(g_x,g_y)$ is the hidden goal. The student observes $o_t=(p_x,p_y,r_t)$, where $r_t=\mathbbm{1}\{(p_x,p_y)=(g_x,g_y)\}$. The expert is obtained by planning the shortest path to the hidden goal using privileged goal information. An effective exploration strategy must use the interaction history to avoid revisiting locations that already produced zero reward, search new candidate goals, and exploit the goal once it is found. This is exactly the structure induced by \MethodS: given a student history, if a rewarding position has been observed, the expert labels return to it; otherwise, the expert labels target positions not yet ruled out by the history. We train and evaluate on an 80--20 random split of goal locations.

\paragraph{Key-Door.}\cite{MinigridMiniworld23,lee2023supervised}
The agent must first locate a hidden key and then navigate to a door; the key is picked up and the door is opened by moving to their respective positions. This is a BAMDP with non-Markovian observations: the agent interacts across multiple test-time episodes, and optimal behavior depends on remembering which locations have been searched and whether the key has already been acquired. We denote the privileged state as
$s_t=(p_x,p_y,k_x,k_y,d_x,d_y,\mathbbm{1}\{\text{has key}\})$,
where $(p_x,p_y)$ is the agent position, $(k_x,k_y)$ is the key position, $(d_x,d_y)$ is the door position, and the final bit indicates whether the agent has the key. The student observes $o_t=(p_x,p_y,r_t)$. In the Markovian version, the agent receives a reward of $+1$ at every timestep after acquiring the key, effectively revealing the task stage. This dense reward also leads the RL baselines to get stuck in a local minimum and perform poorly. In the non-Markovian version, the agent receives a reward of $+1$ only at the timestep when it picks up the key, so it must use its history to decide whether to continue searching for the key or switch to searching for the door. In both versions, the agent receives an additional reward of $+1$ upon reaching the door with the key, and the episode ends. An effective exploration strategy must search for the key, remember how to acquire it, and then search for or navigate to the door.

\paragraph{Ant and 2D Navigation.}\cite{Rakelly2019EfficientOM, zintgraf2021varibad}
The agent must discover and reach a hidden goal, either in a continuous 2D navigation domain or with a high-dimensional quadruped. These are BAMDPs: the goal remains fixed across multiple episodes within a trajectory, while the agent resets between episodes and must use in-context exploration to identify unvisited candidate goals. We denote the privileged state as $s_t=(x_t,g)$, where $x_t$ is the agent state and $g$ is the hidden goal. The student observes $o_t=(x_t,r_t)$, but not $g$. The reward is sparse, with $r_t=\mathbbm{1}\{\|x_t-g\|_2\leq\epsilon\}$ for a success threshold $\epsilon$. The expert uses privileged access to $g$ to produce goal-directed behavior. We train and evaluate on an 80--20 split of goal locations. An effective exploration strategy must use history to avoid revisiting unlikely goals, sample unexplored regions, and exploit the goal once discovered.

\paragraph{Habitat Object Search.}\citep{puig2024habitat}
The agent must find a target object in a realistic indoor environment, where the object is placed at one of $k$ possible locations. This is a partially observed task: the agent interacts for a single episode, while the hidden object location remains fixed throughout the episode and must be inferred from experience. We denote the privileged state as $s_t=(\mathcal{M},x_t,g)$, where $\mathcal{M}$ is the map of the house, $x_t$ is the robot pose, and $g\in\{1,\ldots,k\}$ is the hidden object location. The student observes $o_t=(x_t,I_t,b_t)$, where $I_t$ is the onboard camera observation and $b_t\in\{0,1\}$ is the output of an object detector applied to the current image. The reward is sparse and indicates whether the target object has been found. The expert is an oracle planner that uses the map and object location to navigate to the target. The student policy must use history to eliminate previously checked locations, search among the remaining candidates, and stop once the detector identifies the object. 

\paragraph{Procgen Maze Navigation.}
The agent must navigate a procedurally generated maze to find a hidden goal. This is a single-episode POMDP: the maze layout, initial position, and goal location are unknown, and the agent observes only a local $3\times3$ window centered at its current position. We denote the privileged state as $s_t=(\mathcal{M},p_t,g)$, where $\mathcal{M}$ is the maze layout, $p_t$ is the agent position, and $g$ is the goal location. The student observes $o_t=\mathcal{M}_{3\times3}(p_t)$, the local egocentric crop centered on the agent, but not the global map or the goal. The reward is sparse and indicates whether the agent reaches the goal. The expert plans the shortest path using Q-learning and privileged access to the full maze layout and goal location. An effective exploration strategy must use temporal context to build an implicit belief over the maze, avoid repeatedly visiting uninformative regions, and search until the goal is found. In Fixed Maze Navigation, we evaluate a simplified variant with a fixed maze layout but random initial and goal positions, isolating the challenge of goal-directed exploration under local $3\times3$ observations from the additional generalization required by procedurally generated layouts.

\textbf{Real Robot Manipulation.} We deploy on a UR5e arm with a parallel-jaw gripper.
The cube is placed at a random position within a 10\,cm $\times$ 2\,cm region on the table,
and the arm starts from a fixed initial configuration. The student observes only proprioception,
i.e., joint angles and gripper jaw position, with no cameras or touch sensing, so the cube
location can only be inferred from whether a grasp attempt closes on the object. Here, if the robot attempts a grasp at a false location, the gripper closes all the way, eliminating that location. Prior work~\cite{bronars2026blindpreciseassemblyproprioceptive} also explores learning such proprioception-only behavior via RL for insertion tasks. We train the privileged expert in Isaac Lab with large-scale RL and domain randomization over physics parameters following~\cite{mittal2023orbit,yin2026emergent}. We train both ASTEROID and the BC baseline in simulation with the same partial observations, then transfer them to hardware. We evaluate each method on 10 trials with the same set of cube positions, and count a trial as successful if the cube is lifted within 20 seconds.

\section{Baseline Implementation Details}

\paragraph{\textbf{DPT}~\citep{lee2023supervised}.}
We use the official implementation of DPT with the default hyperparameters from \citet{lee2023supervised}. For a fair comparison, we set the dataset size to match the total data collected over all iterations of \MethodS.

\paragraph{\textbf{SPOC}~\citep{spoc2023ehsani}.}
SPOC augments naive distillation with sampled positional embeddings. We implement SPOC within our codebase and train it on the same total amount of data used across all iterations of \MethodS.

\paragraph{\textbf{RL baselines}.}
For continuous-control RL baselines, we use high-quality implementations from RSL-RL~\citep{mittal2023orbit} and OmniReset~\citep{yin2026emergent}, which support large-scale GPU-parallel training. We implement ADVISOR in RSL-RL and tune its $\alpha$ and $\beta$ parameters over the ranges reported in the original paper. For the meta-RL baselines, we use RL$^2$/PPO baselines~\citep{duan2016rl} and adapted versions that incorporate privileged expert supervision, thereby reducing the computational burden of learning exploration and task execution from sparse rewards alone. For gridworld tasks, we use Stable-Baselines3 PPO implementations with task-specific modifications. On the wrist-camera task, we run the vision-based RL baseline for two seeds due to computational cost: $10^8$ environment steps take roughly four days on an RTX-enabled GPU. Despite this scale, the learned exploration policy achieves only $10\%$ success.

\paragraph{\textbf{AAWR}~\citep{hu2025realworld}.}
We use an extended version of the original AAWR implementation. To construct the initial offline dataset, we add Gaussian noise to generate broad coverage over possible histories. As shown in Figure~\ref{fig:aawr_ablation}, AAWR is sensitive to this initial dataset coverage, whereas \MethodS{} avoids this dependence by collecting context on-policy. AAWR and SPOC were not originally designed for BAMDPs, but we include them in both settings for completeness.

\paragraph{\textbf{Distilled Planner}.}
For Habitat object search, we construct an upper-bound baseline by distilling a posterior-sampling planner under a uniform prior. The dataset consists of expert trajectories that visit the possible object locations in a fixed order and terminate upon finding the target. This produces goal-directed exploration behavior, which we distill into a transformer policy. Unlike \MethodS, however, this baseline assumes access to the full set of possible object locations and the underlying task structure during training, which is unrealistic.

\subsection{Pseudocode}
\label{app:algorithm}
\begin{algorithm}[H]
\caption{\MethodS{} for general POMDPs}
\label{alg:method2}
\begin{algorithmic}[1]
\setlength{\itemsep}{0pt}
\setlength{\parskip}{0pt}

\REQUIRE iterations $K$, samples per iteration $N$, horizon $H$, privileged teacher $\pi^*$, curriculum $\{[\alpha_k, \beta_k]\}_{k=1}^K$
\STATE Initialize $\mathcal{D}_0 \leftarrow \text{expert trajectories}$, student $\pi_{\theta_0} \leftarrow \textit{i.i.d.}\;\text{policy}$

\FOR{$k=1,\dots,K$}
    \STATE $\mathcal{D}_{k} \leftarrow \mathcal{D}_{k-1}$
    \STATE \texttt{\textcolor{NavyBlue}{// Collect on-policy contexts and privileged labels}}
    \FOR{$i=1,\dots,N$}
        \STATE Sample $T$ uniformly in $[\alpha_k H, \beta_k H]$ \texttt{\textcolor{NavyBlue}{// Curriculum (Table~\ref{tab:hyperparameters_isaac})}}
        \STATE Roll out $\pi_{\theta_{k-1}}$ in the POMDP for $T$ steps: $h_k^{(i)} = (o_0, a_0, r_0 \cdots o_{T-1}, a_{T-1}, r_{T-1})$
        \STATE Roll out $\pi^*$ in the same episode from $o_T$:
        $
            \tau^{*,(i)} = (o^*_{T},a^*_{T},r^*_{T},\dots,o^*_{H-1},a^*_{H-1},r^*_{H-1})
        $
        \STATE Add labels
        $
        \mathcal{D}_{k}
        \leftarrow
        \mathcal{D}_{k}
        \cup
        \{(h_k^{(i)}, o^*_{T:t}, r^*_{T:t-1}, a^*_t)\}_{t=T}^{H-1}
        $
    \ENDFOR

    \STATE \texttt{\textcolor{NavyBlue}{// Update student by conditional imitation}}
    \STATE Train $\pi_{\theta_{k}}$ on $\mathcal{D}_{k}$ using Eq.~\ref{eq:loss}
\ENDFOR

\STATE \textbf{return} $\pi_{\theta_K}$

\end{algorithmic}
\end{algorithm}

\subsection{Hyperparameters}

\begin{table}[H]
\centering
\caption{\footnotesize Hyperparameters for \MethodS{} on Gridworld tasks.}
\label{tab:hyperparameters}
\begin{tabular}{lc}
\toprule
\textbf{Hyperparameter} & \textbf{Value} \\
\midrule
\multicolumn{2}{c}{\textit{Training}} \\
\midrule
Dataset size (per iteration) & 1000 \\
\MethodS{} iterations & 10 \\
Batch size & 64 \\
Learning rate & $1 \times 10^{-4}$ \\
Weight decay & $1 \times 10^{-4}$ \\
Number of training epochs (per iteration) & 100 \\
Warmup ratio & 0.03 \\
Gradient clipping & 1.0 \\
Optimizer & AdamW \\
LR scheduler & Warmup + Cosine Decay \\
\midrule
\multicolumn{2}{c}{\textit{Model Architecture (Decision Transformer~\citep{chen2021decisiontransformer})}} \\
\midrule
Architecture & GPT-2 \\
Tokenizer & Project $(o_t,a_{t-1},r_{t-1})$ linearly \\
Parameters & $\sim$3M \\
Number of layers & 4 \\
Number of attention heads & 4 \\
Embedding dimension & 256 \\
Dropout & 0.1 \\
Max context length & $H \times K$ \\
\midrule
\multicolumn{2}{c}{\textit{Evaluation}} \\
\midrule
Test-time episodes & 40 \\
Episode length & 100 \\
\bottomrule
\end{tabular}
\vspace{0.5em}
\end{table}

\begin{table}[H]
\centering
\caption{\footnotesize Hyperparameters for \MethodS{} on Isaac Lab tasks. The exploration horizon denotes the $[\min, \max]$ fraction of the episode during which the exploration policy acts before handing off to the expert.}
\label{tab:hyperparameters_isaac}
\begin{tabular}{lc}
\toprule
\textbf{Hyperparameter} & \textbf{Value} \\
\midrule
\multicolumn{2}{c}{\textit{Training}} \\
\midrule
Demonstrations per iteration & 8192 for Tactile-Pick / 4096 for Peg-Insertion \\
Training steps per iteration & 50,000 gradient steps \\
Learning rate & $1 \times 10^{-4}$ \\
Optimizer & AdamW \\
\midrule
\multicolumn{2}{c}{\textit{Curriculum Schedule} (Ratio of episode length)} \\
\midrule
Exploration horizon (iteration 1) & $[0.2, 0.5]$ \\
Exploration horizon (iteration 2) & $[0.3, 0.8]$ \\
Exploration horizon (iteration 3) & $[0.4, 0.9]$ \\
\midrule
\multicolumn{2}{c}{\textit{Evaluation}} \\
\midrule
Evaluation episodes & 64 \\
Episode length & 240 \\
\bottomrule
\end{tabular}
\vspace{0.5em}
\end{table}

\newpage
\section{Additional Theory}
\label{app:theory}

We present a theoretical analysis of properties of learning in-context exploration using \MethodS{}. We discuss why on-policy contexts help: since \MethodS{} updates the policy conditioned on the student's own histories, it behaves like DAgger~\citep{ross2011reduction} in history space, while an offline in-context learner behaves like behaviour cloning. We analyse the BAMDP setting of Algorithm~\ref{alg:method}. Reusing notation from Section~\ref{sec:prelim}, $h$ denotes the student's history, $d_\pi$ denotes the history distribution induced by a policy $\pi$, $d_{\mathcal{D}}$ denotes the empirical history distribution in a dataset, $\pi^*(a\mid o, z)$ denotes the privileged teacher, and $\pi_\theta(a\mid h)$ denotes the history-conditioned student. We write $\pi_k \equiv \pi_{\theta_k}$ for the student iterate at \MethodS{}  iteration $k$.

\subsection{Preliminaries}
\label{app:theory:bamdp}

Following PSRL~\citep{osband2013psrl}, we analyse ASTEROID at the episode level: at the
start of each episode, the policy selects a hypothesis $\tilde{z}$ and executes the teacher
$\pi^*(\cdot \mid \cdot, \tilde{z})$ for the full episode. Here, the student induces an
episode-level policy $\pi(\tilde{z} \mid h)$. We define the posterior-sampling policy
\[
    \pi^{\mathrm{PS}}(\tilde{z} \mid h) := p(\tilde{z} \mid h),
\]
which samples a hypothesis from the posterior and commits to it for the episode. This produces a policy $\pi^\text{PS}(a|o,h) = \int \pi^*(a|o,z) \;\mathrm{p}(z|h)\; dz$. Here, we evaluate the policy over $K$ online
episodes of fixed episode horizon $H$. Per-episode returns are normalised to $[0,1]$.
 
Let $d^{\pi}_{k}$ denote the distribution of context at the start of episode $k$ when rolling out $\pi$: the history $h$ contains the first $k-1$ episodes. We measure the student's error in episode $k$ by its deviation from posterior sampling on its own contexts,
\begin{equation}
    e_k(\pi)
    :=
    \mathbb{E}_{h\sim d^{\pi}_{k}}
    \left[
        \mathrm{TV}\!\left(
            \pi(\cdot\mid h),\,
            \pi^{\mathrm{PS}}(\cdot\mid h)
        \right)
    \right].
    \label{eq:ek}
\end{equation}
Minimising \eqref{eq:loss} controls this via Pinsker's inequality. 

\begin{assumption}[Compliance]
\label{ass:compliance}
The policy used to generate context is unprivileged: at time $t$, its action
depends only on the observed history, not directly on the latent variable
$z$.
\end{assumption}
\noindent This mirrors Assumption~1 of DPT \citep{lee2023supervised}.

\begin{assumption}[On-policy support coverage]
\label{ass:support}
At \MethodS{}  iteration $k$, the aggregated dataset covers the support of the
histories visited by the current rollout policy up to context of $k$ episodes.
\end{assumption}

\begin{assumption}[Realizability and exact optimization]
\label{ass:realizability}
The policy class $\Pi$ contains $\pi^{\mathrm{PS}}$, and at every iteration
the conditional maximum-likelihood objective in ~\eqref{eq:loss} is
minimised exactly in the population limit.
\end{assumption}

\begin{assumption}[Bounded episode cost]
\label{assump:recover}
A wrong hypothesis in episode $k$ affects the return only through that episode, i.e.,
the episode-$k$ return gap between $\pi$ and $\pi^{\mathrm{PS}}$ is at most $e_k(\pi)$. Summing over episodes gives $J(\pi^{\mathrm{PS}}) - J(\pi) \le \textstyle\sum_{k=1}^K e_k(\pi).$
\end{assumption}

\begin{remark}
\label{rem:student_compliant}
The student $\pi_\theta$ acts only on its observed history and never on $z$,
so every context generated by \MethodS{}  rollouts satisfies
Assumption~\ref{ass:compliance} by construction.
\end{remark}

\subsection{Privileged Labels Provide Posterior-Sampling Supervision}
\label{app:theory:ps-labels}

The key observation is that privileged labels are posterior-sampling labels
when conditioned on compliant histories. A history is compliant if it is
generated by a policy that does not directly observe the latent variable
$z$, as in prior work on learning a posterior from in-context
supervision~\citep{lee2023supervised}. Intuitively, compliance ensures that
the history reveals information about $z$ only through observations,
actions, and rewards available to the student, not through hidden privileged
information.

\begin{lemma}[Privileged labels equal posterior-sampling labels]
\label{lem:teacher_ps}
Let $h$ be a compliant history and $o$ a query observation. If the privileged
teacher label is sampled as $a^* \sim \pi^*(\cdot\mid o, z)$, then
\[
    \mathbb{P}(a^*=a\mid h, o)
    =
    \pi^{\mathrm{PS}}(a\mid h, o).
\]
\end{lemma}

This lemma says that we do not need to explicitly infer $z$: if we train on privileged labels conditioned on compliant histories, maximum likelihood directly trains the posterior-predictive policy. The proof follows \citet{lee2023supervised}. 

\begin{theorem}[\MethodS{}  approximates posterior sampling on covered histories]
\label{thm:asteroid_ps}
Under Assumptions~\ref{ass:compliance}--\ref{ass:realizability}, at every
iteration $k$, for every context $h$ in the support of the aggregated
on-policy dataset $\mathcal{D}_{k}$,
\[
    \pi_{k}(\cdot\mid h)
    =
    \pi^{\mathrm{PS}}(\cdot\mid h).
\]
Consequently, as the context horizon reaches the full online interaction
horizon, \MethodS{}  matches posterior sampling on the covered trajectory
support.
\end{theorem}

\begin{proof}[Proof sketch]
The proof is by induction over iterations.

\paragraph{Base case.}
At iteration $0$, \MethodS{}  trains on null-context expert data. For the empty context $h=\emptyset$ (and along the expert query prefixes, per Remark~\ref{rem:student_compliant}), Lemma~\ref{lem:teacher_ps} implies that the conditional label distribution is exactly $\pi^{\mathrm{PS}}(\cdot\mid h)$. Minimizing cross-entropy over $\mathcal{D}_0$ under Assumption~\ref{ass:realizability} recovers this conditional distribution, so $\pi_0(\cdot\mid h)=\pi^{\mathrm{PS}}(\cdot\mid h)$ on the support of the dataset.

\paragraph{Inductive step.}
Assume that at iteration $k$, $\pi_{k-1}(\cdot\mid h) = \pi^{\mathrm{PS}}(\cdot\mid h)$ for every context $h$ of depth at most $k-1$ episodes in the covered support. \MethodS{}  now rolls out $\pi_{k-1}$ for $k$ episodes to collect contexts. By Remark ~\ref{rem:student_compliant} these contexts are compliant, and since the two policies agree up to that depth, the collected histories follow the same distribution as those posterior sampling would generate up to that horizon. At each such history $h$, \MethodS{}  queries the privileged teacher.
By Lemma~\ref{lem:teacher_ps}, the conditional label distribution given $h$ is $\pi^{\mathrm{PS}}(\cdot\mid h)$, so exact minimization of ~\eqref{eq:loss} yields $\pi_{k}(\cdot\mid h) = \pi^{\mathrm{PS}}(\cdot\mid h)$. Because the context schedule expands from $k-1$ to $k$ episodes, the agreement extends to the larger support, proving the inductive claim.
\end{proof}

Theorem~\ref{thm:asteroid_ps} is an on-support statement: \MethodS{}  does not need offline coverage of all possible histories. It constructs the relevant support by rolling out the student and querying the teacher on the histories the student actually visits.

\subsection{Offline versus On-Policy Training}
\label{app:theory:compounding}
\label{app:theory:offline-error}
\label{app:theory:online-error}

DPT~\citep{lee2023supervised} recovers posterior sampling in the population limit, and Theorem~\ref{thm:asteroid_ps} shows that this holds under on-policy support. We now relax Assumption~\ref{ass:realizability}. The argument follows the classic comparison of behavior cloning and DAgger~\citep{ross2011reduction}, with states replaced by histories and actions by $\tilde{z}$. 

\paragraph{Offline error compounds across episodes.}
Suppose an offline learner is trained on contexts from a fixed distribution
$d_k^{\mathrm{off}}$, and has small error there, $e_k^{\mathrm{off}}(\pi) \le \epsilon$.
At deployment, however, it is evaluated on its own contexts $d_k^\pi$. Since
$\mathrm{TV}$ is bounded,
\[
    e_k(\pi) \le \epsilon + \mathrm{TV}(d_k^\pi, d_k^{\mathrm{off}}).
\]

\begin{lemma}[Offline history mismatch compounds]
\label{lem:offline_compounding}
If the offline contexts are generated by posterior sampling, $d_k^{\mathrm{off}} = d_k^{\mathrm{PS}}$,
then $\mathrm{TV}(d_k^\pi, d_k^{\mathrm{off}}) \le (k-1)\epsilon$, and hence $e_k(\pi) \le k\epsilon$.
\end{lemma}
\begin{proof}[Proof sketch]
Couple a rollout from the offline context distribution with a rollout from the learner.
In each earlier episode, the learner makes an error with probability at most $\epsilon$.
Once such an error occurs, the subsequent history may differ from the offline history,
since future observations, rewards, and actions are conditioned on the previous trajectory.
A union bound over the first $k-1$ episodes bounds the probability that the two histories
have diverged by $(k-1)\epsilon$, which upper-bounds the total variation distance.
Substituting into the previous inequality gives $e_k(\pi) \le k\epsilon$.
\end{proof}
For other offline distributions, such as random contexts, the mismatch is only larger.
Summing over $K$ online episodes,
\[
    \sum_{k=1}^{K} e_k(\pi) \le \epsilon \sum_{k=1}^{K} k = O(\epsilon K^2).
\]
Thus, even if the offline learner has small error on its dataset, its error can compound
quadratically because the histories it sees at deployment drift away from the histories
used for training.

\paragraph{On-policy error is linear in the number of episodes.}
\MethodS{}  avoids this mismatch by training on the student's own histories. Here, we assume the number of training iterations equals the number of online episodes, so we index both by $K$. Let $\pi_1,\ldots,\pi_K$ be the policies produced during iterative training, and define the on-policy  loss
\[
    \ell_i(\pi)
    :=
    \frac{1}{K}\sum_{k=1}^{K} e_k^{(i)}(\pi),
\]
where $e_k^{(i)}(\pi)$ denotes the per-episode deviation of ~\eqref{eq:ek} evaluated on the histories induced by the iterate $\pi_i$.
Let $\epsilon_K := \min_{\pi\in\Pi}\frac{1}{K}\sum_{i=1}^K \ell_i(\pi)$ be the best average error achievable within the policy class, and suppose the supervised learner is no-regret with online regret
$\gamma_K$:
\[
    \frac{1}{K}\sum_{i=1}^K \ell_i(\pi_i)
    \leq
    \epsilon_K+\gamma_K .
\]

\begin{lemma}[On-policy error]
\label{lem:onpolicy_error}
There exists an iterate $\hat{\pi}=\pi_{\hat\imath}\in\{\pi_1,\ldots,\pi_K\}$ such that
\[
    \ell_{\hat\imath}(\hat{\pi})
    \leq
    \epsilon_K+\gamma_K ,
\]
i.e., its average error on its own induced history distribution is at most $\epsilon_K+\gamma_K$.
\end{lemma}

\begin{proof}
The average on-policy loss of the iterates is at most $\epsilon_K+\gamma_K$. Therefore, at least one iterate has loss no larger than this average.
\end{proof}

\begin{theorem}[On-policy error is linear in the number of episodes]
\label{thm:onpolicy_no_compounding}
Let $\hat{\pi}$ be the iterate from Lemma~\ref{lem:onpolicy_error}, and $\epsilon := \epsilon_K+\gamma_K$. Then the cumulative error over $K$ online episodes satisfies
\[
    \sum_{k=1}^K e_k(\hat{\pi})
    =
    K\,\ell_{\hat\imath}(\hat{\pi})
    \leq
    K(\epsilon_K+\gamma_K)
    =
    O(\epsilon K).
\]
\end{theorem}

\begin{proof}[Proof sketch]
Unlike the offline learner, \MethodS{}  does not train on histories from a fixed behavior distribution: at each iteration, the histories in the dataset are generated by the current student policy itself. On-policy histories are collected and relabeled at the next iteration, so the student learns to continue from its own histories. When the learned policy is evaluated on its own histories, there is therefore no additional mismatch term of the form $\mathrm{TV}(d_k^{\hat{\pi}}, d_k^{\mathrm{off}})$. The only remaining error is the supervised error on the student's own induced history distribution, which Lemma~\ref{lem:onpolicy_error} bounds by $\epsilon_K+\gamma_K$ on average over context depths. Multiplying by $K$ gives the claim. Thus, on-policy context aggregation removes the history-mismatch term that makes the offline error grow like $O(k\epsilon)$ per episode: the cumulative error is linear rather than quadratic in the number of online episodes.
\end{proof}

\begin{corollary}[Total regret]
\label{thm:total_regret}
$\mathrm{Reg}_{\mathrm{PS}}(K)$ is the Bayesian regret of posterior sampling over
$K$ episodes. By Assumption~\ref{assump:recover} and Theorem~\ref{thm:onpolicy_no_compounding},
$J(\pi^{\mathrm{PS}}) - J(\hat{\pi}) \le \epsilon K$, so the Bayesian regret of \MethodS{} is
\[
    \mathrm{Reg}_{\mathrm{PS}}(K) + O(\epsilon K),
\]
posterior sampling's regret plus a linear approximation term.
\end{corollary}

\begin{corollary}[Finite samples]
\label{cor:finite_samples}
By Lemma~\ref{lem:teacher_ps}, each iteration is a standard conditional maximum-likelihood problem. Applying the finite-class MLE guarantee of FLAMBE~\citep[Theorem~21]{agarwal2020flambestructuralcomplexityrepresentation} with $N$ on-policy contexts per iteration (and a union bound over the $K$ iterations), with probability at least
$1-\delta$,
\[
    \sum_{k\le K} e_k(\hat{\pi})
    =
    \widetilde{O}\!\Big(K\big(\epsilon + \sqrt{\log(K|\Pi|/\delta)/N}\big)\Big),
\]
so finite data adds a $\widetilde{O}(K/\sqrt{N})$ term without changing the linear dependence on
the number of episodes.
\end{corollary}

\end{document}